\documentclass{article}

\usepackage[preprint]{corl_2026} % Use this for the initial submission.
\usepackage{enumitem}
\usepackage[numbers]{natbib}
\usepackage{amsthm}
\usepackage{amsmath}
\usepackage{amssymb}
\usepackage{mathtools}
\usepackage{xcolor}
\usepackage{physics}
\usepackage{booktabs}
\usepackage[ruled]{algorithm2e}
\usepackage{float}
\usepackage{fnpct}
\usepackage{wrapfig}
\usepackage{comment}
\usepackage{multirow}

\usepackage[table]{xcolor}
\SetKwInput{KwInput}{Input}

\newcommand{\argmin}{\operatorname*{argmin}}
\newcommand{\diag}{\operatorname{diag}}        % diagonal matrix
\newcommand{\bdiag}{\operatorname{bdiag}}      % block diagonal matrix
\newcommand{\defeq}{\vcentcolon=}

\newcommand{\softmin}{\operatorname{softmin}}

\newcommand{\calI}{\mathcal{I}}
\newcommand{\calL}{\mathcal{L}}

\newcommand{\R}{\mathbb{R}}

\newcommand{\bx}{\boldsymbol{x}}
\newcommand{\bu}{\boldsymbol{u}}
\newcommand{\by}{\boldsymbol{y}}
\newcommand{\bz}{\boldsymbol{z}}

\newcommand{\bs}{\boldsymbol{s}}

\newcommand{\tx}{\tilde{\mathbf{x}}}
\newcommand{\tu}{\tilde{\mathbf{u}}}
\newcommand{\ty}{\tilde{\mathbf{y}}}
\newcommand{\tz}{\tilde{\mathbf{z}}}

\newcommand{\ts}{\tilde{\mathbf{s}}}

\newcommand{\txu}{\widetilde{\mathbf{xu}}}

\newcommand{\blambda}{\boldsymbol{\lambda}}
\newcommand{\bxi}{\boldsymbol{\xi}}

\newcommand{\bzero}{\boldsymbol{0}}
\newcommand{\bp}{\boldsymbol{p}}
\newcommand{\bc}{\boldsymbol{c}}
\newcommand{\bv}{\boldsymbol{v}}
\newcommand{\bTheta}{\boldsymbol{\boldsymbol{\theta}}}
\newcommand{\bOmega}{\boldsymbol{\Omega}}
\newcommand{\bchi}{\boldsymbol{\chi}}
\newcommand{\bzeta}{\boldsymbol{\zeta}}

\newcommand{\bQ}{\boldsymbol{Q}}
\newcommand{\bR}{\boldsymbol{R}}
\newcommand{\bI}{\boldsymbol{I}}

\newcommand{\J}[1]{J_{#1}}
\newcommand{\f}[1]{f_{#1}}
\newcommand{\g}[1]{g_{#1}}
\newcommand{\h}[1]{h_{#1}}
\newcommand{\paramell}[1]{\ell_{#1}}

\newcommand{\paramrho}[1]{\boldsymbol{\rho}_{#1}}
\newcommand{\parammu}[1]{\boldsymbol{\mu}_{#1}}

\renewcommand{\norm}[1]{\left\lVert{#1}\right\rVert}

\newtheorem{theorem}{Theorem}

\newtheorem{proposition}[theorem]{Proposition}

\newtheorem{corollary}[theorem]{Corollary}

\newcommand{\deepcoordinator}{Deep Coordinator }
\newcommand{\deepcoordinatorcomma}{Deep Coordinator, }
\newcommand{\deepcoordinatorperiod}{Deep Coordinator. }

\newcommand{\hquad}{\hspace{0.5em}} 

\definecolor{ForestGreen}{RGB}{25, 102, 25}
\definecolor{MplBlue}{RGB}{16, 105, 189}
\definecolor{MplOrange}{RGB}{224, 111, 11}

\title{Deep-Unfolded Coordination}
\author{
Hunter Kuperman \qquad 
Minchan Jung\thanks{These authors contributed equally (equal second authorship)} \qquad 
Rahul V. Ghosh\footnotemark[1]
\\[0.5em]
\textbf{Alex Oshin}\qquad
\textbf{Evangelos A. Theodorou} 
\\[0.75em]
Autonomous Control and Decision Systems Laboratory \\
Georgia Institute of Technology \\
United States
}

\begin{document}
\maketitle

\begin{abstract}
    Distributed optimization is a highly scalable and structurally transparent technique to solve multi-agent robotics problems; however, such methods often suffer from the need for highly-specialized, problem-specific hyperparameter tunings. In this work, we propose \textbf{Deep Coordinator}, a deep-unfolding framework that learns to dynamically adjust the hyperparameters of ADMM-DDP, a popular distributed solver for robotics tasks, at solve-time in response to optimizer performance. Our architecture consists of unrolling a fixed number of ADMM-DDP iterations into a neural network with learnable functions between layers mapping the optimizer state to the next hyperparameters. To the best of our knowledge, \deepcoordinator is the first deep-unfolding framework to adapt the penalty parameters of a non-convex optimizer at solve-time; we show that the mainstream supervised approach can yield degenerate solutions when training such models, and propose an unsupervised learning scheme. On simulations with fleets of cars and quadrotors, Deep Coordinator produces trajectories of comparable quality 6.18-9.44x faster than conventional solvers. Furthermore, Deep Coordinator retains its performance benefits when deployed to systems up to 8x larger than trained on.
\end{abstract}

% \keywords{Multi-Agent Coordination, Model-Based Learning, Deep-Unfolding} 

\section{Introduction}
\begin{figure}[h]
\centering
\begin{tabular}{ccc}
    \hspace{-12pt}
    \includegraphics[width=0.33\linewidth]{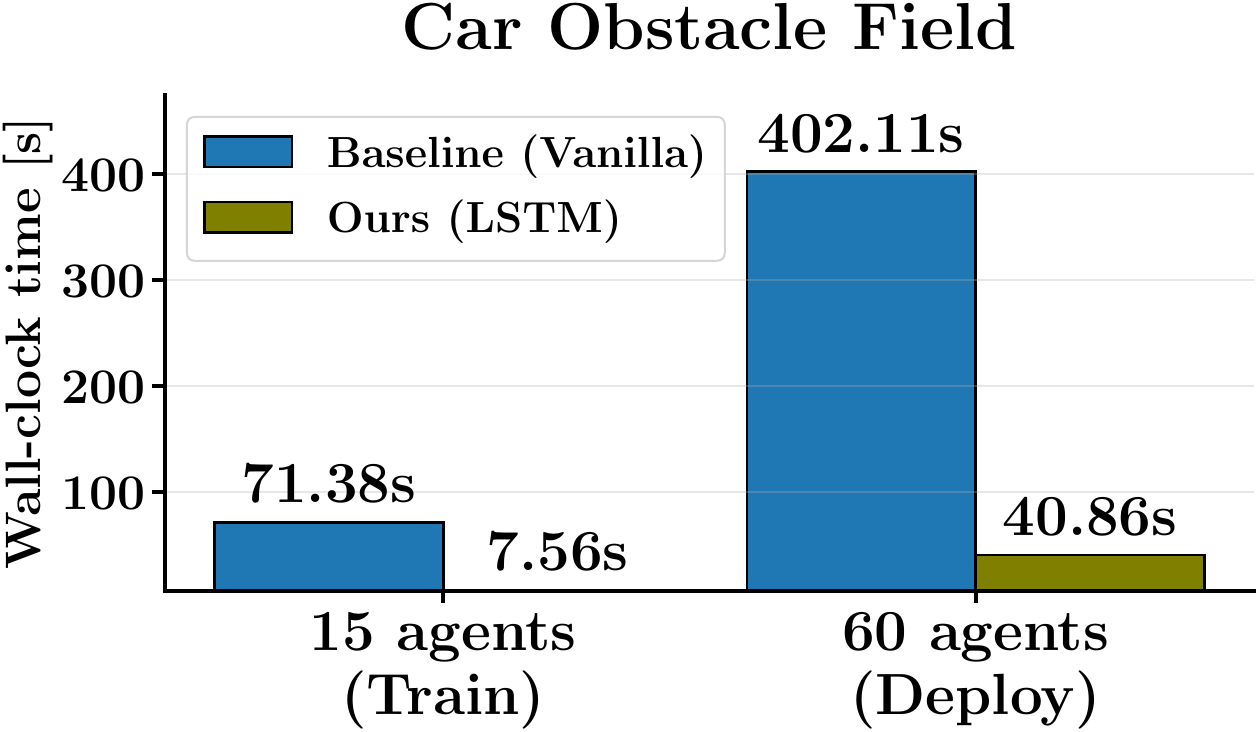} &
    \hspace{-12pt}
    \includegraphics[width=0.33\linewidth]{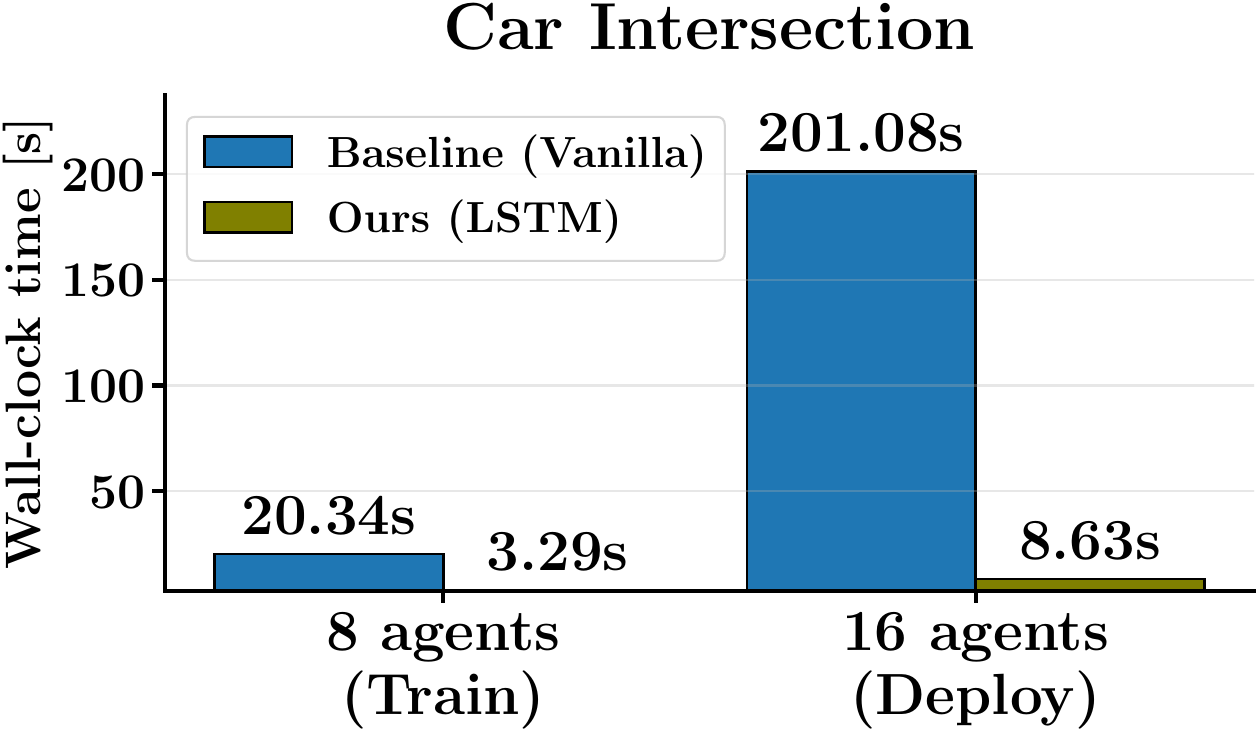} &
    \hspace{-12pt}
    \includegraphics[width=0.33\linewidth]{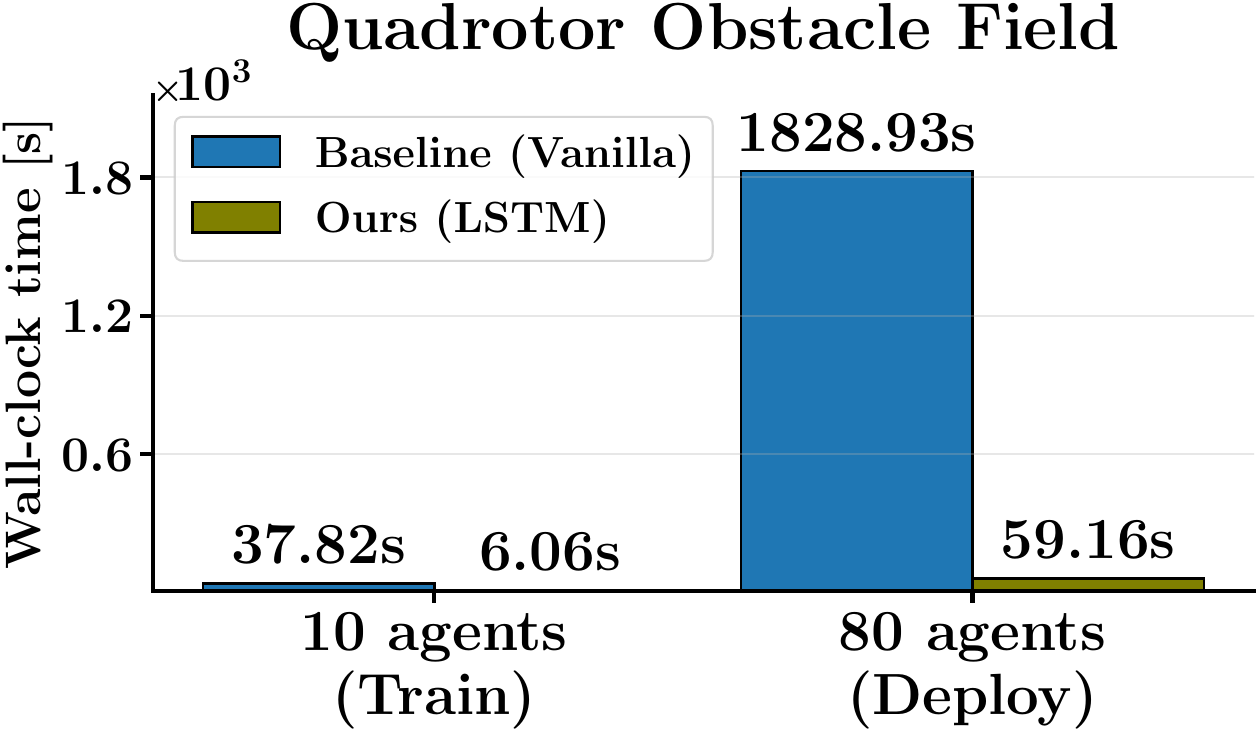}
    \hspace{-12pt}
\end{tabular} 
    \caption{Wall-clock time of \deepcoordinator trained on smaller-scale systems but deployed to larger-scale ones, where \deepcoordinator yields significant speedups across task types and scales. Plot titles indicate multi-agent task type. Results are averaged over 20 problems for each task.}
\end{figure}
Multi-agent robotic systems, characterized by their large dimensionality, inter-agent constraints, and the need to coordinate effectively, span a variety of applications. Drone swarms performing search and rescue, for example, must coordinate their coverage patterns while avoiding collision~\citep{ferraro2023multi}; teams of robots collaborating to lift heavy objects must manage weight distribution while maintaining balance~\citep{tuci2018cooperative}; and fleets of autonomous vehicles must jointly plan maneuvers such as lane changes and merges without compromising safety~\citep{hua2025connected}. Many of these applications benefit from scaling to systems with hundreds or thousands of agents. However, in robotics tasks, controllers typically only have milliseconds to solve the next problem, and existing methods often cannot deliver solutions fast enough to be practical at scale.

Two main approaches have arisen to solve such multi-agent problems, each with complementary strengths. Multi-Agent Reinforcement Learning (MARL)~\citep{zhang2019marl} leverages data to learn decentralized policies for each agent, resulting in fast inference time and the ability to specialize to complex problem distributions. However, MARL policies are frequently parameterized by model-free neural networks, meaning the resulting models are difficult to interpret and---although recent work~\citep{li2024cooperative, sun2024llmbased} has made progress in this area---often generalize poorly to out-of-distribution tasks.

Conversely, distributed optimization decomposes large tasks into smaller optimization problems that can be solved quicker and in parallel. A popular method under this paradigm is ADMM-DDP \citep{saravanos2022distributed}, which combines the Alternating Direction Method of Multipliers (ADMM) \citep{boyd2011distributed} with Differential Dynamic Programming (DDP) \citep{mayne1966secondorder}. Their method introduces safe copy-variables of the control and state which always satisfy constraints, enabling constraint satisfaction to be handled separately from solving the underlying control task. Although ADMM-DDP efficiently decomposes the problem structure and is highly interpretable, its convergence speed is heavily dependent on the tuning of coupled hyperparameters, the number of which increases significantly with the size of the problem. Despite this, ADMM-DDP is often hand-tuned, leading to slow convergence. Furthermore, performant tunings can vary significantly between problems, limiting the generalization capabilities of any one tuning.

Rather than choosing between the speed and specializability of data-driven learning or the safety and interpretability of structured optimization, we view these characteristics as composable. In recent years, \emph{deep-unfolding} \citep{monga2021algorithm} has emerged as a promising technique to accelerate optimization algorithms by parameterizing optimizer steps with trainable weights that can be learned from data. Deep-unfolding reinterprets these optimizer iterations as differentiable layers in a neural network, which can then be trained end-to-end to learn the optimal weights. The resulting learned optimizers inherit the interpretability and structure of their traditional counterparts, while often converging significantly faster and demonstrating improved generalization~\citep{saravanos2024deep, noah2025distributed, lupu2023deep}.

We claim that multi-agent robotics tasks are particularly well-suited for deep-unfolding, for the following reasons: 
\begin{enumerate}
    \item Multi-agent distributed optimizers often have tunable parameters for each agent, yielding hundreds or thousands of total hyperparameters. Deep-unfolding can efficiently learn complex relationships between these hyperparameters that are difficult to identify manually.
    \item Control algorithms are often deployed in a model-predictive control (MPC) fashion, where the optimizer re-solves similar problems, differing only in their initial conditions and constraints, at each timestep. Deep-unfolded networks can be trained exclusively on this problem distribution, yielding specialized hyperparameter policies that improve convergence.
    \item In robotics, slow convergence can lead to suboptimal or unsafe behavior. Given the time constraints, the performance benefits of deep-unfolded networks are particularly valuable. 
    % \item Given the time-constraints in robotics contexts, slow convergence can lead to suboptimal behavior. The performance benefits of deep-unfolded networks are particularly valuable.
\end{enumerate}

We propose \deepcoordinatorcomma a deep-unfolding framework for learning to dynamically adapt ADMM-DDP hyperparameters at solve-time in response to optimizer performance. Our architecture consists of unrolling a fixed number of ADMM-DDP iterations into a neural network with learnable functions between layers mapping the optimizer state to the next hyperparameters, and can be deployed to larger or different teams of agents than trained on. For training, we utilize the Implicit Function Theorem (IFT)~\citep{krantz2003implicit} to derive a novel gradient computation scheme which exploits the distributed structure of ADMM-DDP to efficiently differentiate through Deep Coordinator end-to-end. To the best of our knowledge, \deepcoordinator is the first deep-unfolding framework to adapt the penalty parameters of a non-convex optimizer at solve-time. We show that the mainstream supervised learning approach can yield degenerate solutions when training such non-convex deep-unfolded models, and propose an unsupervised training framework. 

To showcase the performance of \deepcoordinatorcomma we benchmark our model on three multi-agent tasks involving fleets of cars and quadrotors. Across problems, \deepcoordinator finds solutions of the same cost and comparable constraint satisfaction 6.18-9.44x faster than its traditional optimizer counterparts. Furthermore, \deepcoordinator retains its performance benefits when deployed to systems with up to 8x larger team sizes than trained on. Lastly, we validate the improved performance of the unsupervised framework when compared to the mainstream supervised approach.

\section{Related Work}
\textbf{Deep-Unfolding as Control.} 
The authors of \citep{saravanos2024deep} note that selecting the optimal hyperparameters in a deep-unfolded network can itself be interpreted as a robotic control problem. From this perspective, the optimizer steps at each iteration serve as a dynamics function which propagates the intermediate solutions through time, while the hyperparameters act as control inputs which drive the system towards an optimum. Under this interpretation, the mainstream approach of learning fixed, per-iteration parameters in deep-unfolding can be viewed as a form of open-loop control, since the hyperparameters are selected without feedback. In this sense, schemes that adapt hyperparameters in response to the optimizer state ``close the loop."

\textbf{Learning to Optimize ADMM. } Deep-unfolding belongs to a broader class of methods called Learning to Optimize~(L2O), which use machine learning to improve the performance of optimization algorithms. Several recent L2O works have targeted ADMM-based distributed optimization. One work \citep{noah2025distributed} unfolds a generic distributed ADMM algorithm and learns an open-loop hyperparameter sequence. Another work \citep{doerks2025learning} reformulates ADMM as a Graph Neural Network~(GNN) and learns feedback policies on step size and communication weights. Reinterpreting ADMM as a GNN enables the learned model to explicitly reason over the underlying graph topology induced by the problem splitting; however, their framework does not consider learning penalty parameters and is limited to unconstrained, convex problems. In addition, \citep{saravanos2024deep} introduce and unfold a distributed quadratic programming algorithm based on OSQP~\cite{stellato2020osqp}. They adapt hyperparameters using closed-loop feedback policies, which are parameterized as Multi-Layer Perceptrons (MLPs) taking local and global residuals as input. Each agent uses the same MLP weights, enabling policies to be deployed to problems with more agents than trained on. Although it demonstrates strong performance, their framework is limited to convex quadratic programs, and, while their model adjusts the penalty parameters per-agent, it does not adapt the entire penalty vector.

\textbf{Learning to Optimize in Robotics. }
Several other recent works have applied L2O methods to robotics tasks. The authors of \citep{lupu2023deep} unfold centralized projected gradient descent, replacing several terms in the update equations with learned matrices, and apply the resulting learned optimizer to a linear MPC task. Another work \citep{wang2025learning} introduces a framework to meta-learn ADMM-DDP hyperparameters from example problems and derives an IFT-based gradient computation framework to differentiate through the ADMM-DDP optimizer. This approach demonstrates promising results, but is limited to producing a fixed hyperparameter sequence without utilizing the closed-loop form of adaptivity; additionally, the presented gradient computation framework does not handle cases where learnable functions are introduced between iterations.

\section{ADMM-DDP}
\subsection{Problem Formulation}
Before we present our framework, we first review the mathematical foundation of ADMM-DDP. Consider a team of $N$ agents collaborating on a task. Let $\bx_{i,t} \in \mathbb{R}^{n_x}$ and $\bu_{i,t} \in \mathbb{R}^{n_u}$ denote the state and control of agent $i$ at time $t$, respectively. Each agent $i$ is subject to the dynamics
\begin{equation*}
    \bx_{i,t+1} = \f{i}(\bx_{i,t}, \bu_{i,t})
\end{equation*}
and seeks to minimize its cost \begin{equation*}
    \J{i}(\bx_i, \bu_i) = \paramell{T}(\bx_{i,T}) + \sum_{t=0}^{T-1} \paramell{t}(\bx_{i,t}, \bu_{i,t})
\end{equation*}
where $T$ denotes the planning horizon, $\paramell{t}$ denotes the stage cost, and $\paramell{T}$ denotes the terminal cost.
% \begin{align*}
%     \bx_{i,t+1} = \f{i}(\bx_{i,t}, \bu_{i,t})
% \end{align*}
% Each agent seeks to minimize its cost
% \begin{align*}
%     \J{i}(\bx_i, \bu_i) = \paramell{T}(\bx_{i,T}) + \sum_{t=0}^{T-1} \paramell{t}(\bx_{i,t}, \bu_{i,t})
% \end{align*}
% where $\paramell{t}$ denotes the step cost and $\paramell{T}$ denotes the terminal cost.
We consider the following optimization problem
\begin{align*}
\label{eq:optimization_problem} \tag{A}
    \left(\{\bx_i, \bu_i\}_{i=1}^N\right)^* = \argmin_{\{\bx_i, \bu_i\}} \quad & \sum_{i=1}^N \J{i}(\bx_i, \bu_i), \notag \\
    \text{subject to} \quad & \bx_{i,t+1} = \f{i}(\bx_{i,t}, \bu_{i,t}), \quad \forall i, t \notag \\
    & \g{i, t}(\bx_{i,t}, \bu_{i,t}) \leq 0, \quad \forall i, t \notag \\
    & \h{ij,t}(\bx_{i,t}, \bu_{i,t}, \bx_{j,t}, \bu_{j,t}) \leq 0, \quad \forall i, j, t
\end{align*}
where $\g{i, t}$ defines the local constraints for agent $i$ and $\h{ij,t}$ defines the inter-agent constraints between agents $i$ and $j$.

\subsection{The ADMM-DDP Algorithm}
Problem \eqref{eq:optimization_problem} is a multi-agent nonlinear program (NLP) with a large number of variables and constraints, and cannot be solved scalably in a centralized fashion. To solve this problem efficiently, we consider the ADMM-DDP formulation employed in \citep{wang2025learning}, introducing safe copy-variables $\tx$ and $\tu$ and decomposing the original task into the distributed problem
\begin{equation*} \label{eq:distributed_optimization}\tag{B}
\begin{aligned}
    \min_{\bx, \bu, \tx, \tu} \quad & \sum_{i=1}^{N} \Biggl[ 
    \J{i}(\bx_i, \bu_i) 
    + \calI_{\f{i}}(\bx_i, \bu_i) + \calI_{\g{i}}(\tx_i, \tu_i) + \sum_{j \neq i} \Bigl(\calI_{\h{ij}}(\tx_i, \tx_j, \tu_i, \tu_j) \Bigr) \Biggl], \\
    \text{subject to} \quad &  \hquad \bx = \tx, \hquad \bu = \tu,
    % \blambda_i^T(\bx_i - \tx_i) \notag \\
    % &\qquad + \bxi_i^T(\bu_i - \tu_i) 
    % + \frac{\paramrho{i}}{2}\|\bx_i - \tx_i\|_2^2 
    % + \frac{\parammu{i}}{2}\|\bu_i - \tu_i\|_2^2 \Biggr],
\end{aligned}
\end{equation*}
where $\calI_{\f{i}}$, $\calI_{\g{i}}$, $\calI_{\h{ij}}$ are indicator functions for the dynamics, local, and inter-agent constraints, respectively.
The augmented Lagrangian for this problem is
\begin{align*}
    \calL = &\sum_{i=1}^{N} \Biggl[ 
    \J{i}(\bx_i, \bu_i) 
    + \calI_{\f{i}}(\bx_i, \bu_i) 
    + \calI_{\g{i}}(\tx_i, \tu_i) + \sum_{j \neq i} \Bigl(\calI_{\h{ij}}(\tx_i, \tx_j, \tu_i, \tu_j) \Bigr)
     \notag \\
    &\qquad \qquad  + \blambda_i^T(\bx_i - \tx_i) + \bxi_i^T(\bu_i - \tu_i) 
    + \frac{\paramrho{i}}{2}\|\bx_i - \tx_i\|_2^2 
    + \frac{\parammu{i}}{2}\|\bu_i - \tu_i\|_2^2 \Biggr]
\end{align*}
% Note that this is different from the augmented Lagrangian in the Nested-Distributed and Merged-Distributed structures proposed in \citep{saravanos2022distributed}, as it does not contain local copy variables (these are denoted $\bx_i^a$ in that work, although we re-use the superscript $a$ to mean ADMM iterations in our work). 
% \vspace{1em}
% \noindent
where $\boldsymbol{\rho}, \blambda \in \R^{N \times (T + 1) \times n_x}$ and $\boldsymbol{\mu}, \bxi \in \R^{N \times T \times n_u}$ are the penalty parameter and dual variable for the consensus constraints on $\bx$ and $\bu$, respectively. Minimizing the augmented Lagrangian with respect to each variable yields the following subproblems to solve at ADMM iteration $a + 1$, whose mathematical formulations are given in Appendix~\ref{app:sec:subproblems_of_admm_ddp}:

\textbf{Subproblem 1: } First, consider solving for the nominal trajectory updates. This yields the agent-wise unconstrained control problems~\eqref{eq:subproblem_1}. These subproblems can be solved in parallel using DDP to find dynamically-feasible trajectories that minimize each agent's individual costs.

\textbf{Subproblem 2: } Second, consider solving for the safe copy variable updates. Minimizing the augmented Lagrangian yields optimization problems~\eqref{eq:subproblem_2} decoupled across time.
The $\tx, \tu$ variables find a safe trajectory that respects the agent-specific and inter-agent constraints, and can be updated in parallel across the time horizon using an NLP solver.

\textbf{Subproblem 3: } Third, consider updating the dual variables via dual ascent. This yields the parallelizable updates~\eqref{eq:subproblem_3} that are decoupled for each agent-time pair.

\pagebreak
\section{Deep Coordinator}
In this section, we present the \deepcoordinator framework. For notational concision, we stack the ADMM iterates into a single vector $\bv^a = [\bx^a, \bu^a, \tx^a, \tu^a, \blambda^a, \bxi^a].$

\subsection{The Deep Coordinator Framework}
Suppose $K$ is a predetermined maximum number of optimizer iterations we intend to unfold. We unroll $K$ ADMM-DDP iterations into a neural network, re-interpreting each ADMM iteration as its own layer. We then construct the \deepcoordinator architecture by introducing penalty parameter adaptation policies between ADMM-DDP layers, of the form
\begin{equation*}
\begin{split}
    \boldsymbol{\theta}^a = \pi_{w^a}(\boldsymbol{\theta}^{a-1}, \bv^{a-1}, \boldsymbol{\chi}),
\end{split}
\end{equation*}
where $\boldsymbol{\theta}^a = [\boldsymbol{\rho}^a, \boldsymbol{\mu}^a]$, $\pi_{w^a}$ is a learnable feedback policy, $w^a$ are the network weights, and $\boldsymbol{\chi}$ represents problem-specific data.
There are several reasonable choices of $\pi_{w^a}$, which we discuss in Appendix \ref{app:choosing_a_feedback_scheme}. We denote the resulting neural network with $\Pi_w$.
Running \deepcoordinator amounts to propagating the initial iterates and problem data through $\Pi_w$, i.e.,  alternating between computing the next hyperparameters and executing a step of ADMM-DDP with these hyperparameters. We summarize this process in Algorithm~\ref{alg:deep_coordinator}, and the architecture is visualized in Fig. \ref{fig:deep_coordinator_architecture}.

To learn the optimal policy weights, we train \deepcoordinator end-to-end on a dataset of $M$ sample tasks. Our training objective is the following bi-level optimization problem to find optimal weights
\begin{equation*}
\begin{split}
    w^* = \argmin_{w} \; & \sum_{m = 1}^M \; L_m(\Pi_w(\bv^0, \boldsymbol{\chi}_m)),
\end{split}
\end{equation*}
where $\boldsymbol{\chi}_1, ..., \boldsymbol{\chi}_M \in \mathcal{X}$ are sample tasks and $L_m$ is an upper-level loss function which evaluates the trajectories generated by Deep Coordinator. We discuss choices of $L_m$ in Subsection \ref{subsection:trainingloop}.

\begin{figure*}[t]
    \centering
    \includegraphics[width=0.8\textwidth]{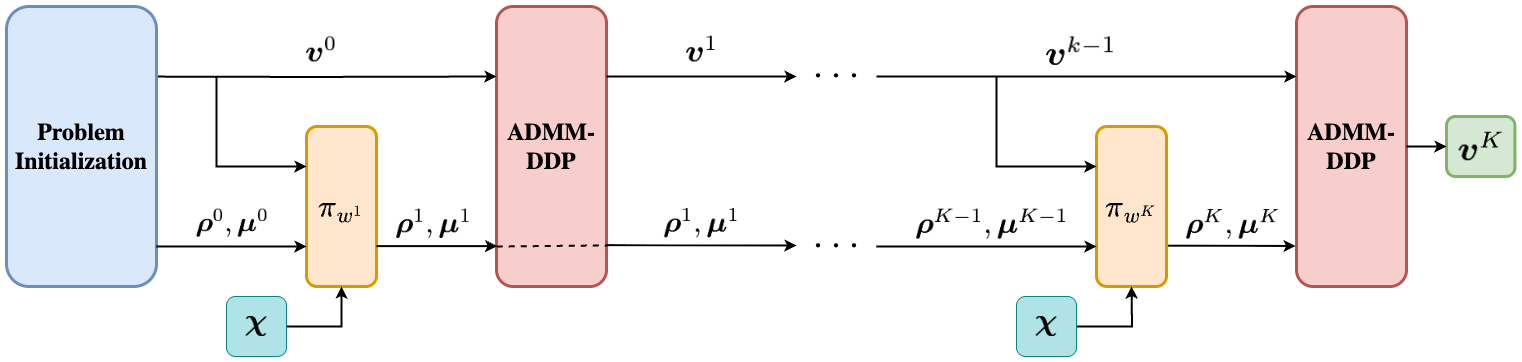}
    \vspace{-3pt}
    \caption{The Deep Coordinator architecture. Our framework unfolds ADMM-DDP iterations into a deep learning architecture. We train policies for determining the optimal hyperparameters in an end-to-end fashion to maximize task performance.}
    \label{fig:deep_coordinator_architecture}
\end{figure*}

\subsection{Designing the Training Loop} \label{subsection:trainingloop}

The mainstream approach~\citep{lupu2023deep, doerks2025learning, sambharya2024learning, saravanos2024deep} to train deep-unfolded networks has been to minimize a supervised loss function which measures the distance to a ground-truth solution. This strategy has proven effective when unfolding convex solvers, which admit a unique global minimizer to utilize as the ground-truth. However, the multi-agent robotics tasks \deepcoordinator solves are heavily non-convex and often have many local optima. Whether or not a supervised approach is well-suited for this case is an open question. 

To build intuition for why the supervised framework can be poorly suited for training non-convex deep-unfolded models, we present a 1D example employing a simpler optimizer in Appendix~\ref{app:illustrative_example}.

Motivated by this, we adopt the following unsupervised loss function
\begin{equation*}
\begin{split}
    L(\bv^K) = \gamma_{\textrm{cost}} \sum_{i} J_i(\bv_i^K) + \gamma_{\textrm{const}} \left( \sum_{i, t} \max{\{0, g_{i,t}(\bv_{i, t}^K)\}}^2 + \sum_{i, j, t} \max{\{0, h_{ij, t}(\bv_{i, t}^K, \bv_{j, t}^K)\}}^2\right),
\end{split}
\end{equation*}
where $\gamma_{\textrm{cost}}$ and $\gamma_{\textrm{const}}$ are tunable parameters. This loss penalizes a linear combination of the total cost and constraint violation and proved empirically reliable in our experiments.

A natural alternative scheme is to employ a loss function which utilizes the ADMM-DDP residuals (e.g., a linear combination of the primal and dual residuals). However, when the initial trajectory is constraint-satisfying, both ADMM-DDP residuals can be made zero by degenerately scaling the penalty parameters arbitrarily high. In this case, the quadratic coupling term scaled by the penalty parameters acts as a proximal term, halting the optimization process. With arbitrarily high penalty parameters, the dual residual becomes zero since the iterates do not change, while the primal residual remains zero since the initial solution was feasible. As a result, residual minimization can reward stalled optimization rather than progress towards lower cost solutions.

\subsection{Computing Gradients}
Training \deepcoordinator requires backpropagating through $K$ iterations of ADMM-DDP, where each iteration involves solving Subproblems \eqref{eq:subproblem_1} and \eqref{eq:subproblem_2} using iterative optimizers.
Unrolling these inner optimization loops using automatic differentiation \citep{baydin2017automatic} is computationally intractable, requiring $O(K \cdot T \cdot N_\text{inner})$ time and memory, where $N_\text{inner}$ is the number of inner solver iterations.
Moreover, this unrolling approach does not support the use of black box optimizers that do not provide gradients. To address these shortcomings, we utilize the IFT~\citep{krantz2003implicit} to derive a gradient computation scheme that avoids differentiating through the inner solver iterations and only requires $O(K \cdot T)$ time and memory. A detailed discussion of this framework is in Appendices \ref{app:gradient_computation} and \ref{app:gradient_computation_detailed_derivation}. 

\vspace{0.5em}
\begin{algorithm}[H]
\caption{Deep Coordinator}\label{alg:deep_coordinator}
\KwInput{Problem data $\bchi$, weights $w$, number of iterations $K$, initialization $\bv^0$}
% \KwOutput{Final iterate $\bv^K$}
\For{$a = 1, \ldots, K$}{
    Predict hyperparameters: $\boldsymbol{\theta}^a = \pi_{w^a}(\boldsymbol{\theta}^{a-1}, \bv^{a-1}, \boldsymbol{\chi})$; \\
    Solve Subproblem 1 with DDP: $\bx^a, \bu^a \gets $ Eq. \eqref{eq:subproblem_1}; \\
    Solve Subproblem 2 with NLP: $\tx^a, \tu^a \gets $ Eq. \eqref{eq:subproblem_2}; \\
    Update dual variables: $\blambda^a, \bxi^a \gets$ Eq. \eqref{eq:subproblem_3};
}
\Return{$\bv^K$}
\end{algorithm}

\section{Experimental Results}\label{sec:experimental_results}
To evaluate the performance of \deepcoordinatorcomma we benchmark our framework on several tasks involving fleets of cars and quadrotors. We compare \deepcoordinator models with three different parameterizations of the policy $\pi_{w^a}$, which we discuss below, to analyze the effect of feedback.

\textbf{Learned Scalar. } We first consider an open-loop policy parameterized by a single scalar for $\boldsymbol{\rho}$ and a single scalar for $\boldsymbol{\mu}$ which are shared across all agents, timesteps, and algorithm iterations.

\textbf{Learned Scalar Per Iteration. } We second consider a policy parameterized by $K$ open-loop scalars for $\boldsymbol{\rho}$ and $K$ open-loop scalars for $\boldsymbol{\mu}$ which are shared across all agents and timesteps, but can differ across iterations.  

\textbf{LSTM. } We last consider a shared feedback policy parametrized by a Long Short-Term Memory~(LSTM) network \citep{hochreiter1997long}. The LSTM encodes an architectural prior that memory of past iterations is important in selecting the next hyperparameters. 

For each task, we unfold \deepcoordinator for $K = 30$ iterations and train with 80 train and 20 test problems. We compare \deepcoordinator against two un-learned benchmarks: Vanilla ADMM-DDP, which employs a fixed, manually-tuned scalar, and Adaptive ADMM-DDP, which employs the penalty parameter adaptation rule proposed in \citep{saravanos2022distributed}. For details on the training setup, LSTM architecture, and baseline optimizers, consult Appendix~\ref{app:details_on_experimental_setup}. Additional experiments ablating the effect of feedback and comparing the supervised and unsupervised loss can be found in Appendix~\ref{app:additional_experiments}.

\subsection{Main Experiments}
We study three types of control problems, which we discuss below. For additional details on problem generation and agent dynamics, consult Appendix~\ref{app:details_on_problem_types}.

\textbf{Car Obstacle Field. } We first consider a multi-agent avoidance scenario with a randomized obstacle field. In this task, teams of 15 Dubins vehicles must reach a target formation across a field of three circular obstacles. The centers and radii of the obstacles are randomized for each task instance. The results and a representative \deepcoordinator trajectory are shown in Fig. \ref{fig:car_plot_type_1}.

\begin{figure}[h]
\centering
\begin{tabular}{ccc}
    \hspace{-12pt}
    \includegraphics[width=0.33\linewidth]{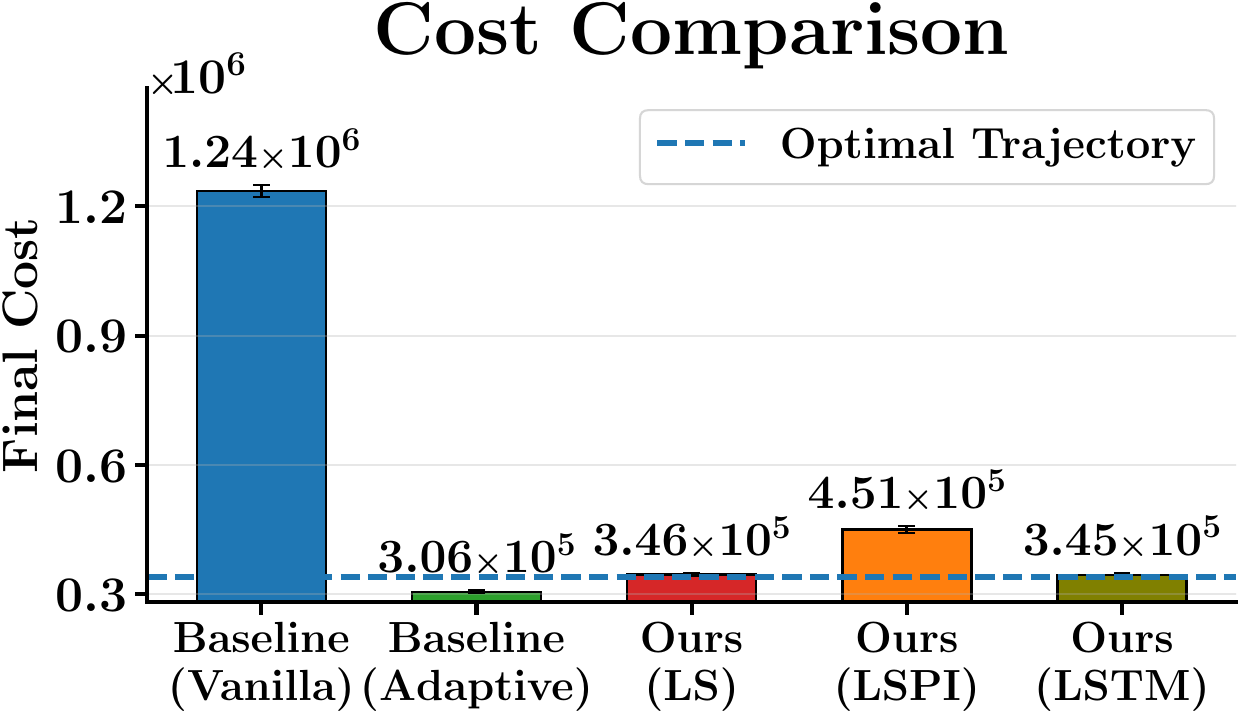} &
    \hspace{-12pt}
    \includegraphics[width=0.33\linewidth]{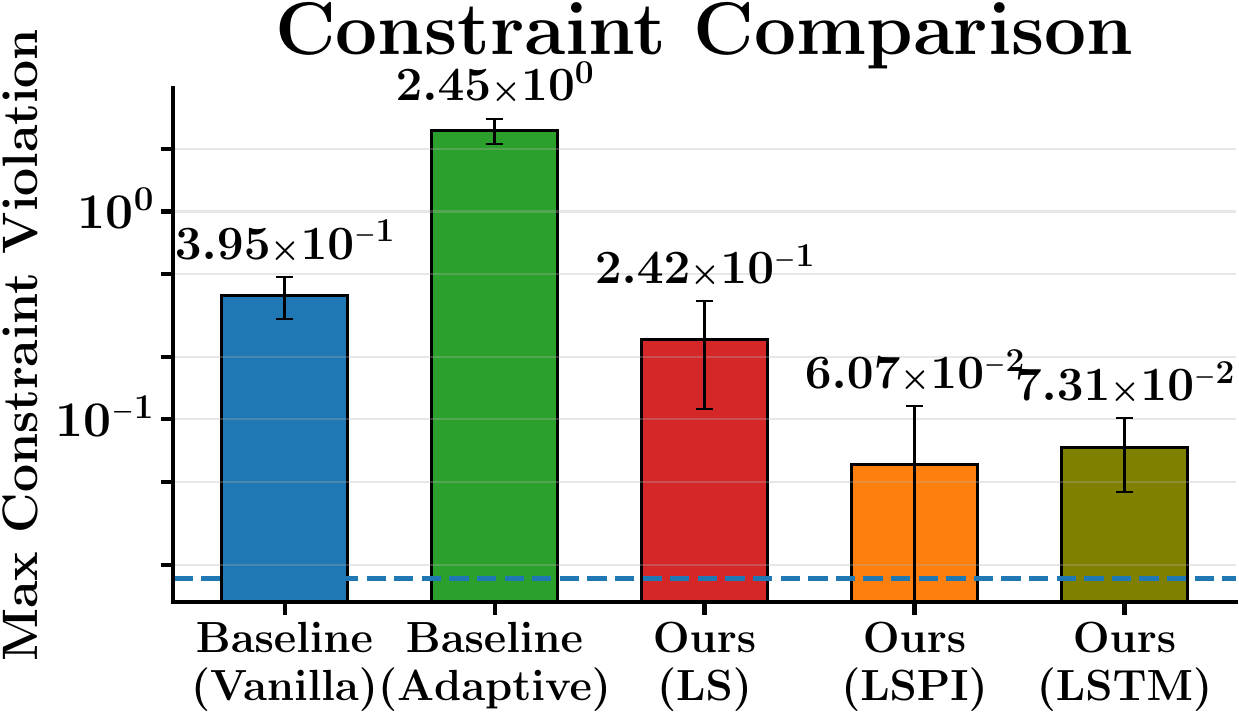} &
    \hspace{-12pt}
    \includegraphics[width=0.33\linewidth]{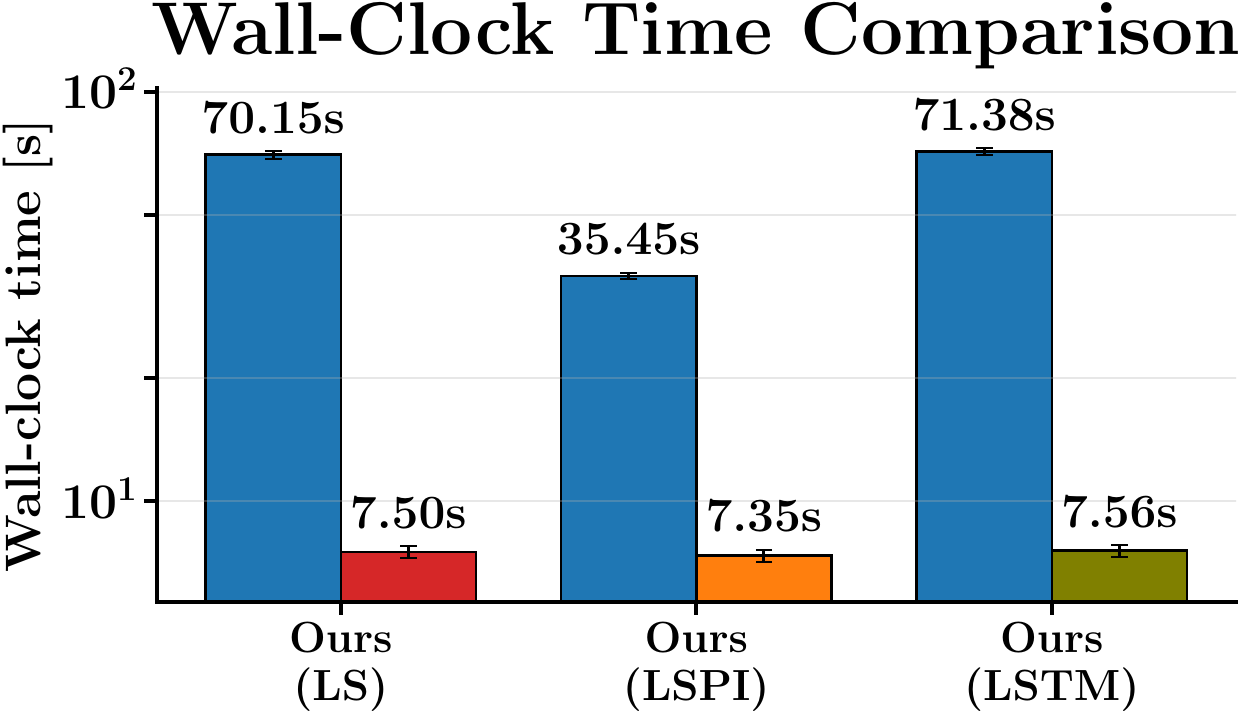}
    \hspace{-12pt}
\end{tabular} \\
\includegraphics[width=0.24\linewidth]{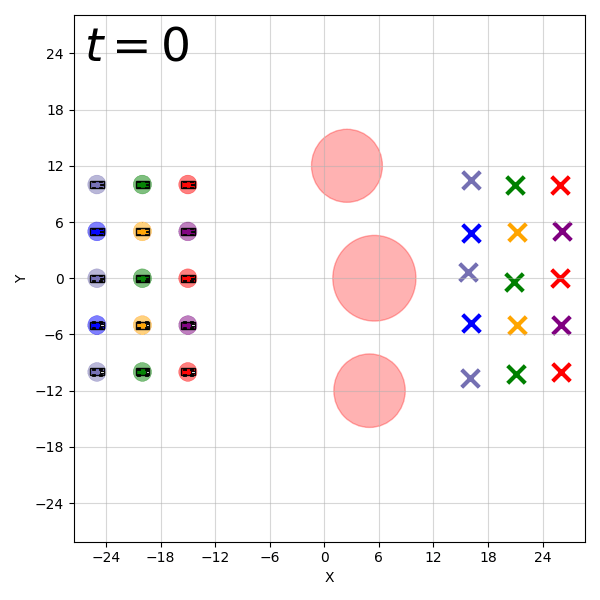}
\includegraphics[width=0.24\linewidth]{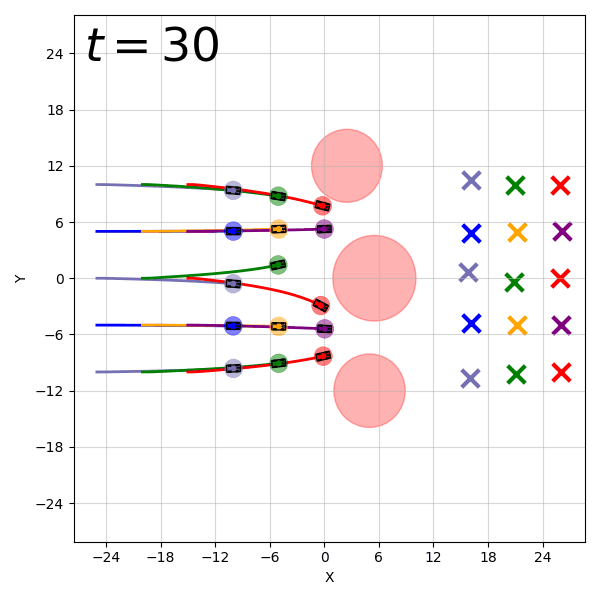}
\includegraphics[width=0.24\linewidth]{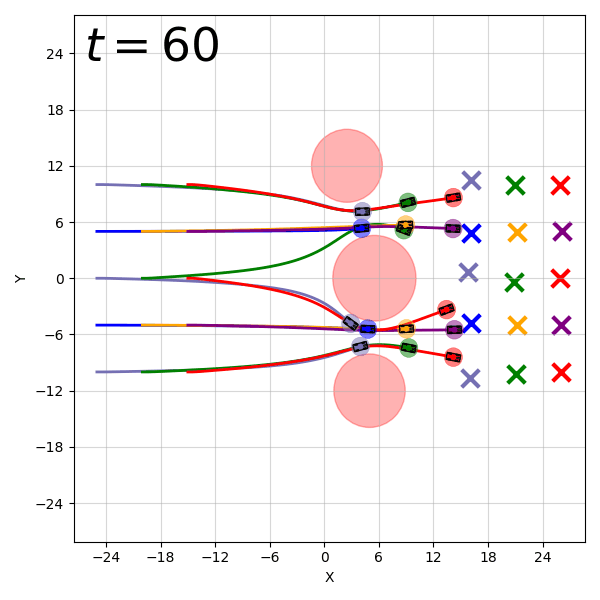}
\includegraphics[width=0.24\linewidth]{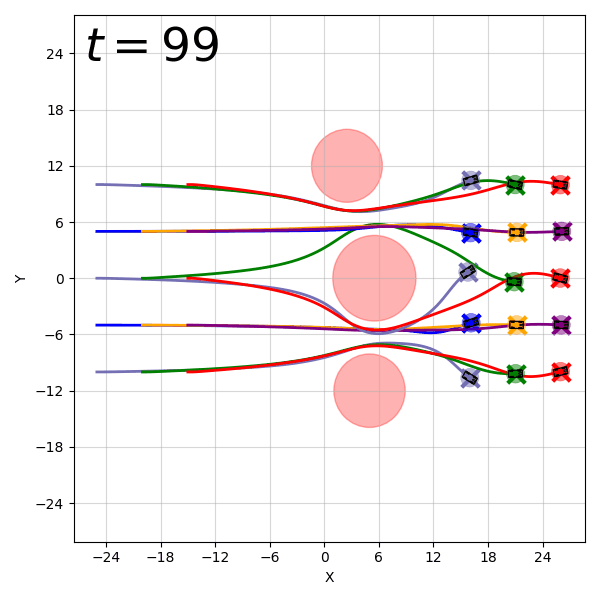}
\vspace{-6pt}
\caption{\textbf{Top: }Results of the car obstacle field task on 20 unseen test instances. Error bars indicate standard deviation. \textbf{Bottom: }Trajectories of 15 vehicles performing the car obstacle field task generated by \deepcoordinatorcomma with snapshots from different timesteps.}
\label{fig:car_plot_type_1}
\vspace{-6pt}
\end{figure}

\textbf{Car Intersection. } We second study a more complex intersection task where 8 Dubins vehicles must navigate a four-way junction. The initial and target states of the agents are randomly sampled, requiring significantly different trajectories for each problem instance. The results and a representative \deepcoordinator trajectory are shown in Fig. \ref{fig:inter_plot_type_1}.

\begin{figure}[h]
\centering
\begin{tabular}{ccc}
    \hspace{-12pt}
    \includegraphics[width=0.33\linewidth]{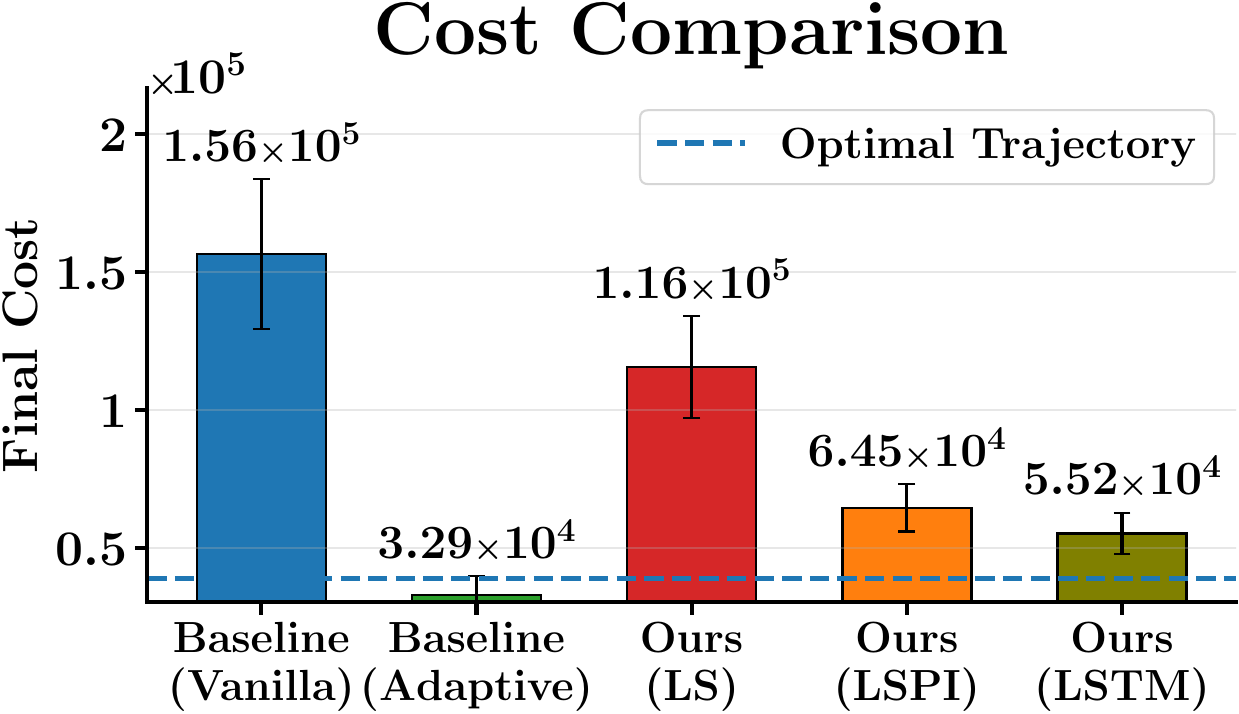} &
    \hspace{-12pt}
    \includegraphics[width=0.33\linewidth]{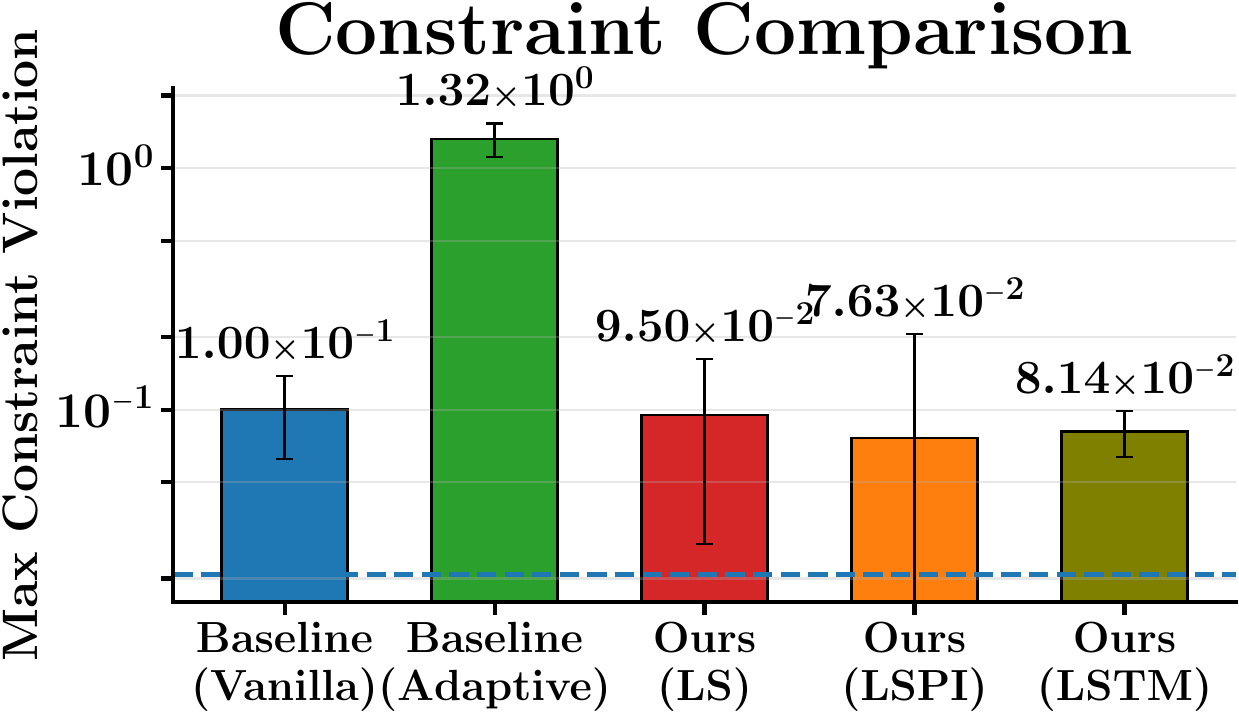} &
    \hspace{-12pt}
    \includegraphics[width=0.33\linewidth]{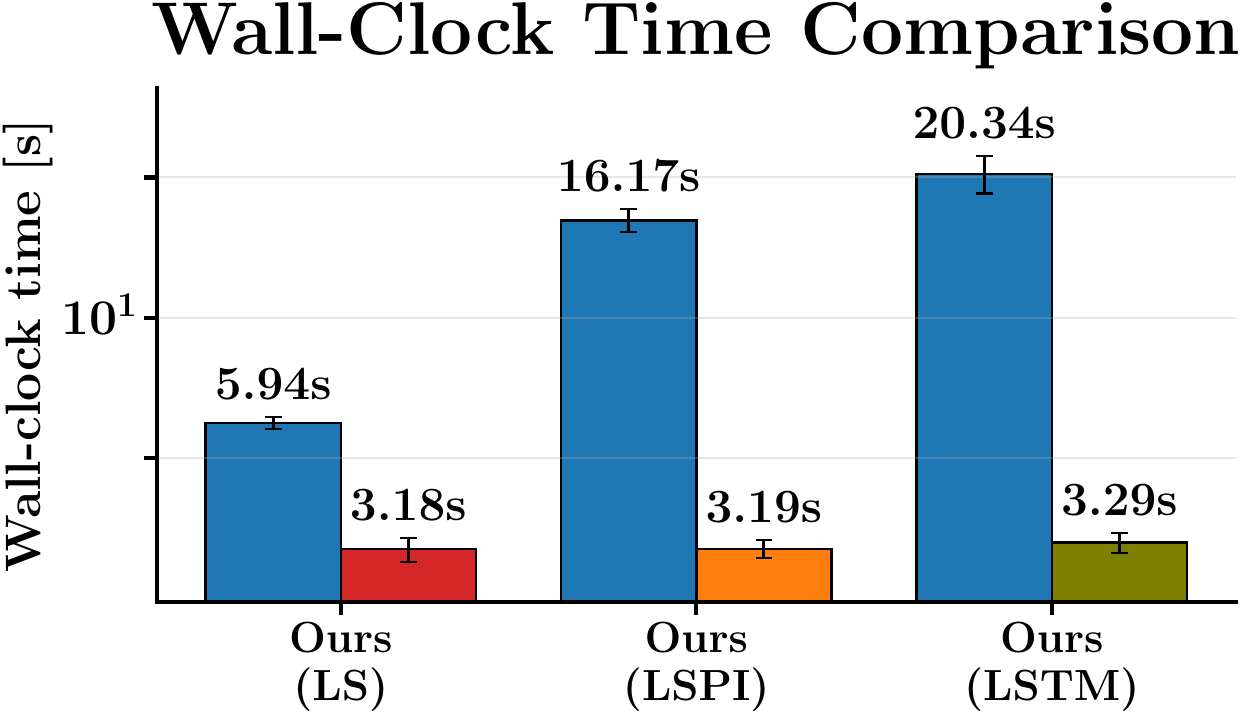}
    \hspace{-12pt}
\end{tabular} \\
\includegraphics[width=0.24\linewidth]{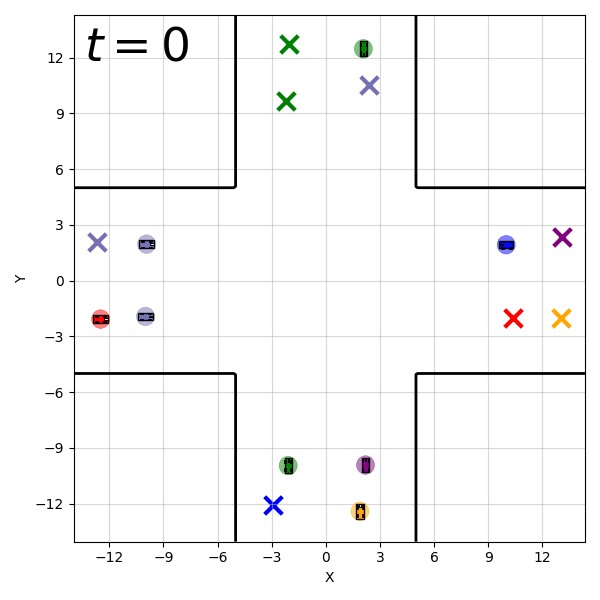}
\includegraphics[width=0.24\linewidth]{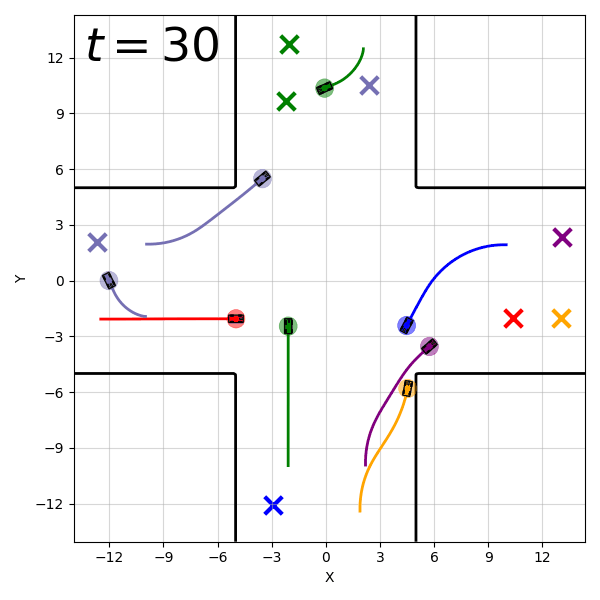}
\includegraphics[width=0.24\linewidth]{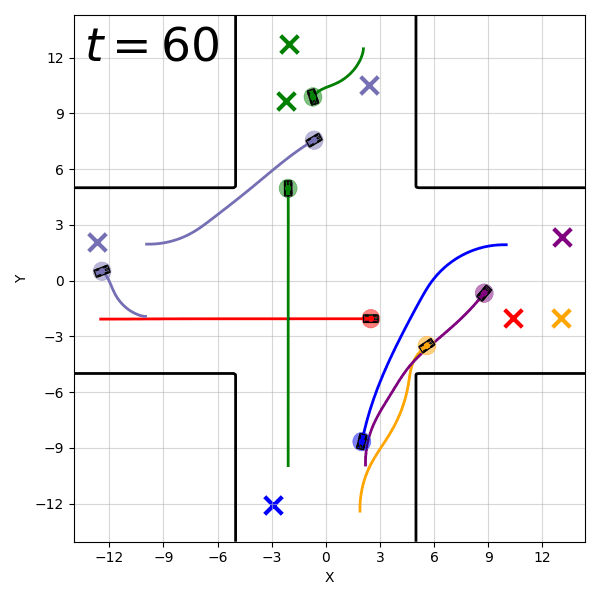}
\includegraphics[width=0.24\linewidth]{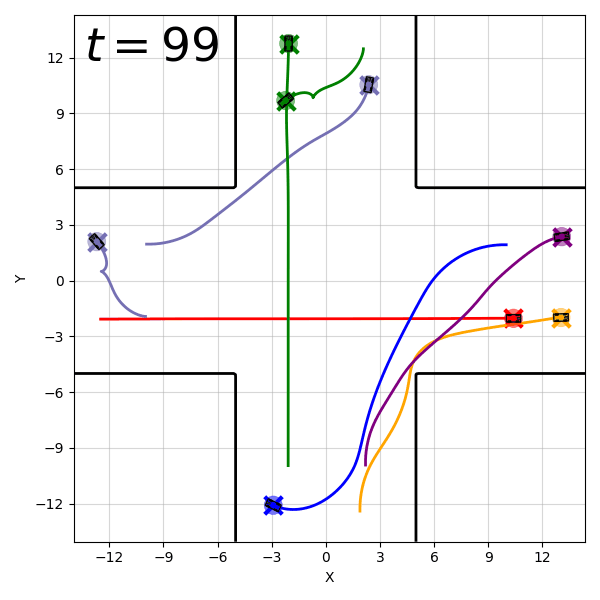}
\vspace{-6pt}
\caption{\textbf{Top: }Results of car intersection tasks on 20 unseen test instances. Error bars indicate standard deviation. \textbf{Bottom: }Trajectories of 8 vehicles performing the car intersection task generated by \deepcoordinatorcomma with snapshots from different timesteps.}
\label{fig:inter_plot_type_1}
\vspace{-6pt}
\end{figure}

\textbf{Quadrotor Obstacle Field. } We third examine a quadrotor maneuvering problem where teams of 10 quadrotors navigate a field of 7 cylindrical obstacles. To maximize interactions between agents, the two columns of agents must swap their relative positions while crossing the field. Both displacements to the quadrotors' initial altitudes and the obstacles' centers and radii are randomized across task instances. The results and a representative \deepcoordinator trajectory are shown in Fig. \ref{fig:quad_plot_type_1}.

\begin{figure}[h]
\centering
\begin{tabular}{ccc}
    \hspace{-12pt}
    \includegraphics[width=0.33\linewidth]{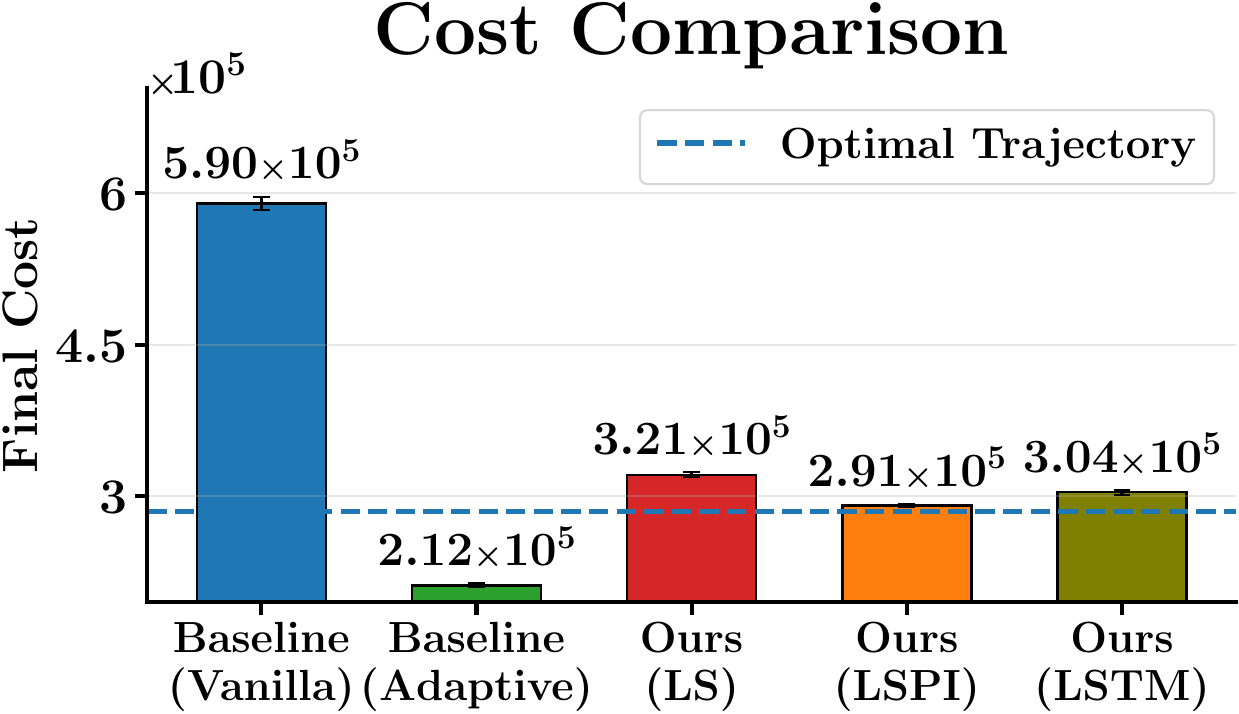} &
    \hspace{-12pt}
    \includegraphics[width=0.33\linewidth]{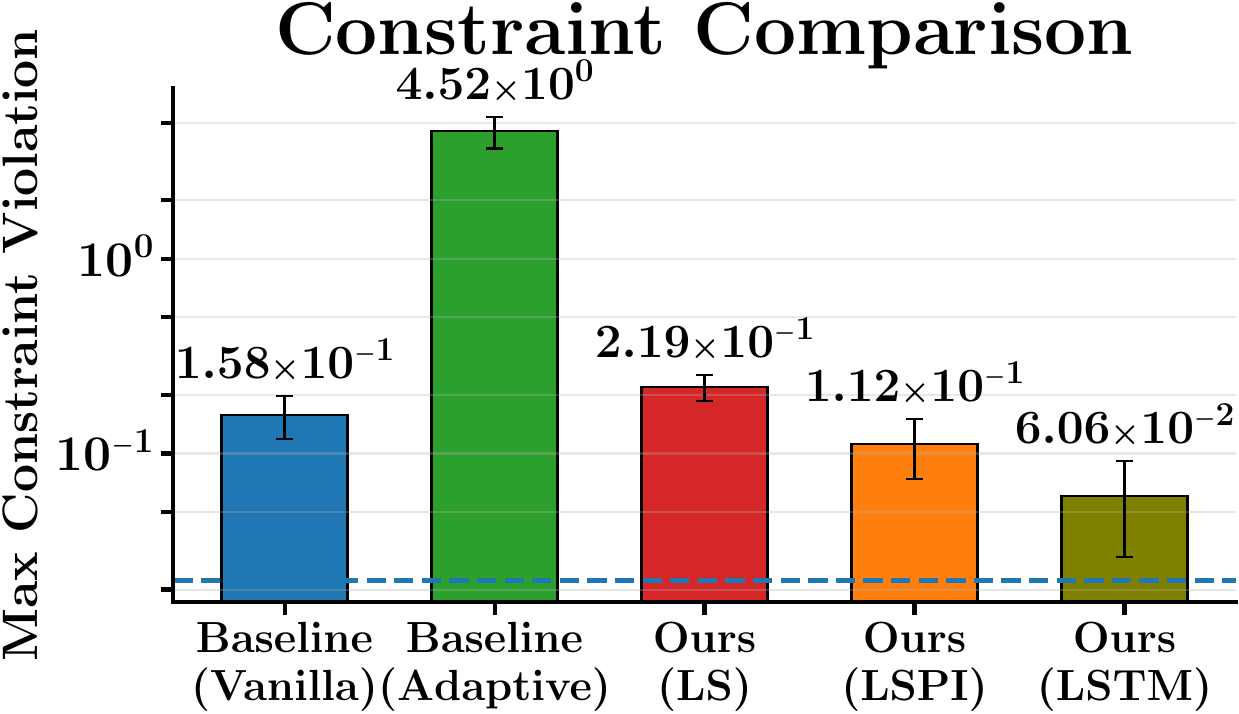} &
    \hspace{-12pt}
    \includegraphics[width=0.33\linewidth]{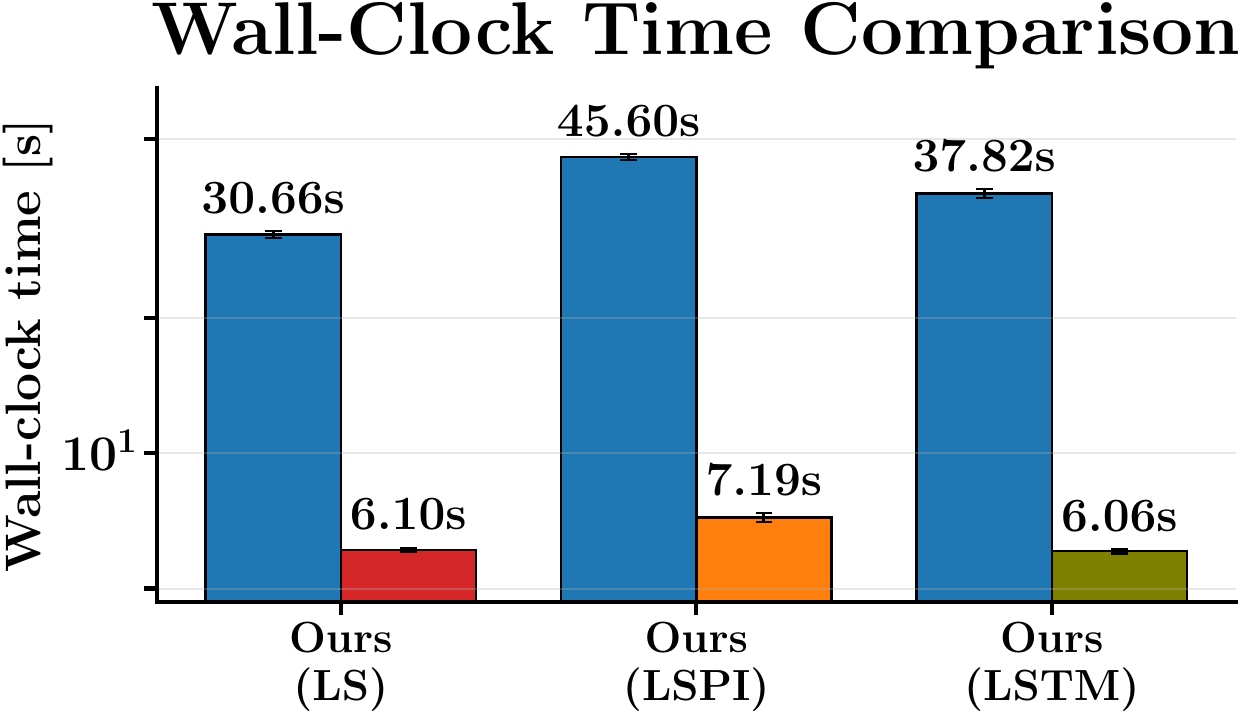}
    \hspace{-12pt}
\end{tabular} \\
\includegraphics[width=0.24\linewidth]{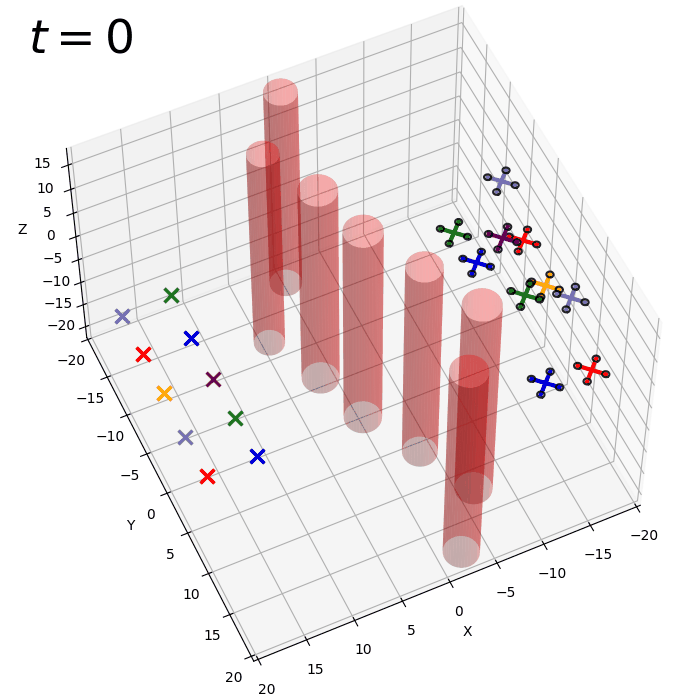}
\includegraphics[width=0.24\linewidth]{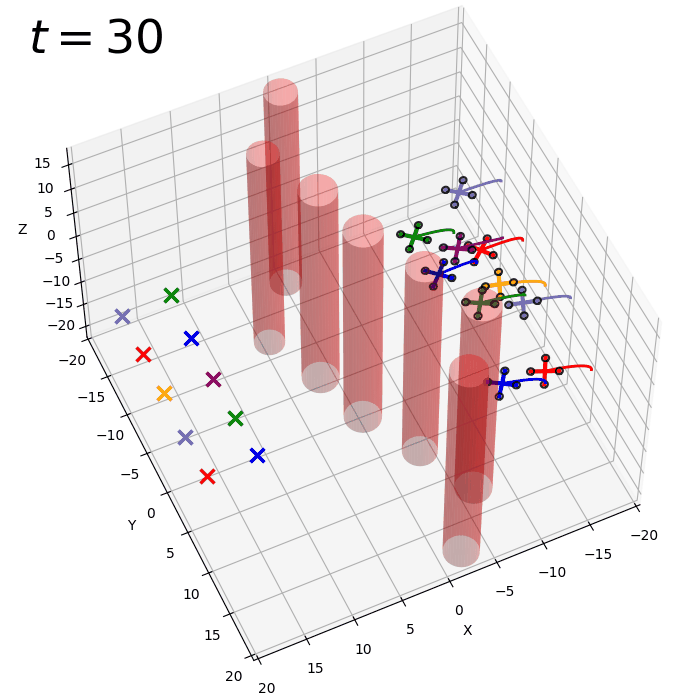}
\includegraphics[width=0.24\linewidth]{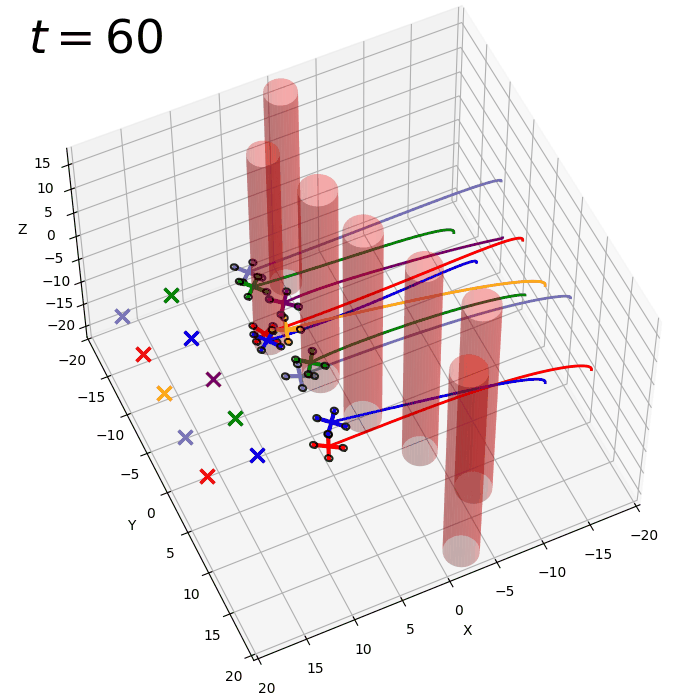}
\includegraphics[width=0.24\linewidth]{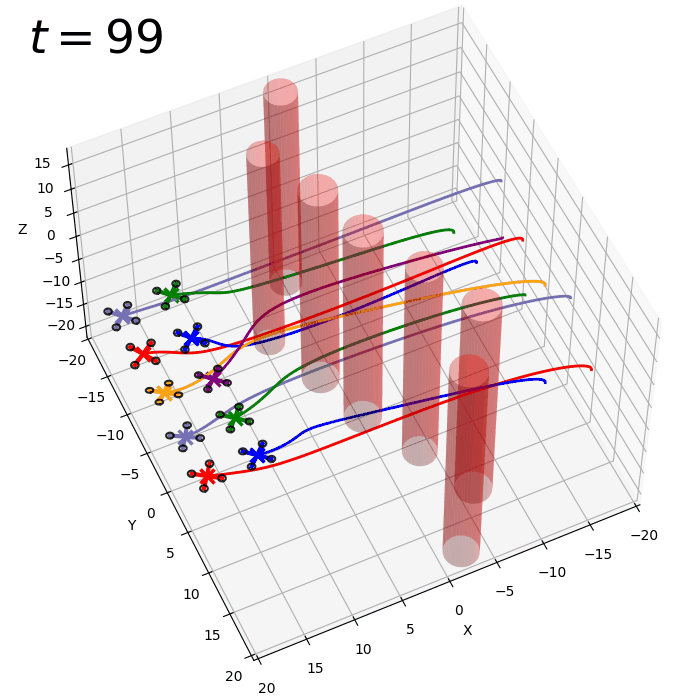}
\vspace{-6pt}
\caption{\textbf{Top: }Results of quadrotor obstacle field tasks on 20 unseen test instances. Error bars indicate standard deviation. \textbf{Bottom: }Trajectories of 10 quadrotors performing the obstacle field task generated by \deepcoordinatorcomma with snapshots from different timesteps.}
\label{fig:quad_plot_type_1}
\vspace{-6pt}
\end{figure}

For each task, we evaluate the cost, maximum constraint violation, and wall-clock time of all models on 20 unseen test problems, and average the results. The cost and constraint plots report values at the end of the $K = 30$ iteration budget. To compare wall-clock times, \deepcoordinator is run for 30 iterations. Then, in order to show the relative improvement of our method, Vanilla ADMM-DDP is run until its cost is within 5\% of Deep Coordinator's. We omit timings for Adaptive ADMM-DDP as, though it quickly finds low-cost solutions, its trajectories significantly violate safety constraints. The optimal trajectory is computed by running Vanilla ADMM-DDP until convergence. 

Our experiments indicate that \deepcoordinator significantly outperforms its traditional optimizer counterparts. We partition our analysis into three parts: within the 30-iteration budget, when allowing Vanilla ADMM-DDP to run until it reaches the same cost as \deepcoordinatorcomma and compared to the optimal solution.

Within the 30-iteration budget, Deep Coordinator finds much higher-quality solutions than its traditional counterparts. The LSTM policy, in particular, Pareto dominates Vanilla ADMM-DDP. On the car and quadrotor obstacle field tasks, LSTM \deepcoordinator is able to find solutions with significantly lower cost (1.94-3.59x) and constraint violation (2.61-5.40x) than Vanilla ADMM-DDP. In the intersection task, where the agents must manage more complex right-of-way interactions, the performance benefit of \deepcoordinator is lower, but still yields 2.82x lower cost and 1.22x lower constraint violation trajectories. Adaptive ADMM-DDP quickly converges to slightly lower cost solutions than \deepcoordinatorcomma but does so with much higher constraint violation---up to an order of magnitude greater than the other solvers---and is therefore unsafe to use. 

The most significant benefit of the \deepcoordinator framework is the reduction in wall-clock time required to reach high-quality solutions. When allowing Vanilla ADMM-DDP to run for longer, \deepcoordinator finds trajectories of comparable cost \textbf{6.18-9.44x faster} with comparable constraint violation. For additional context, figures depicting the cost and constraint violation over an extended wall-clock time are in Appendix~\ref{app:extended_wall_clock_time_plots}. We emphasize that the cost-thresholded Vanilla solution is not the optimal trajectory; rather, it is an intermediate iterate that has not yet converged and, in general, exhibits higher constraint violation. Compared to the optimal trajectories, \deepcoordinator finds solutions of comparable cost, but higher constraint violation. This is expected, as \deepcoordinator is significantly constrained by its iteration budget, taking 13.8-18.6x less time than it takes to converge to the optimal solution. 

\begin{figure}[h]
    \centering
    \begin{tabular}{ccc}
    \hspace{-10pt}
    \includegraphics[width=0.33\linewidth]{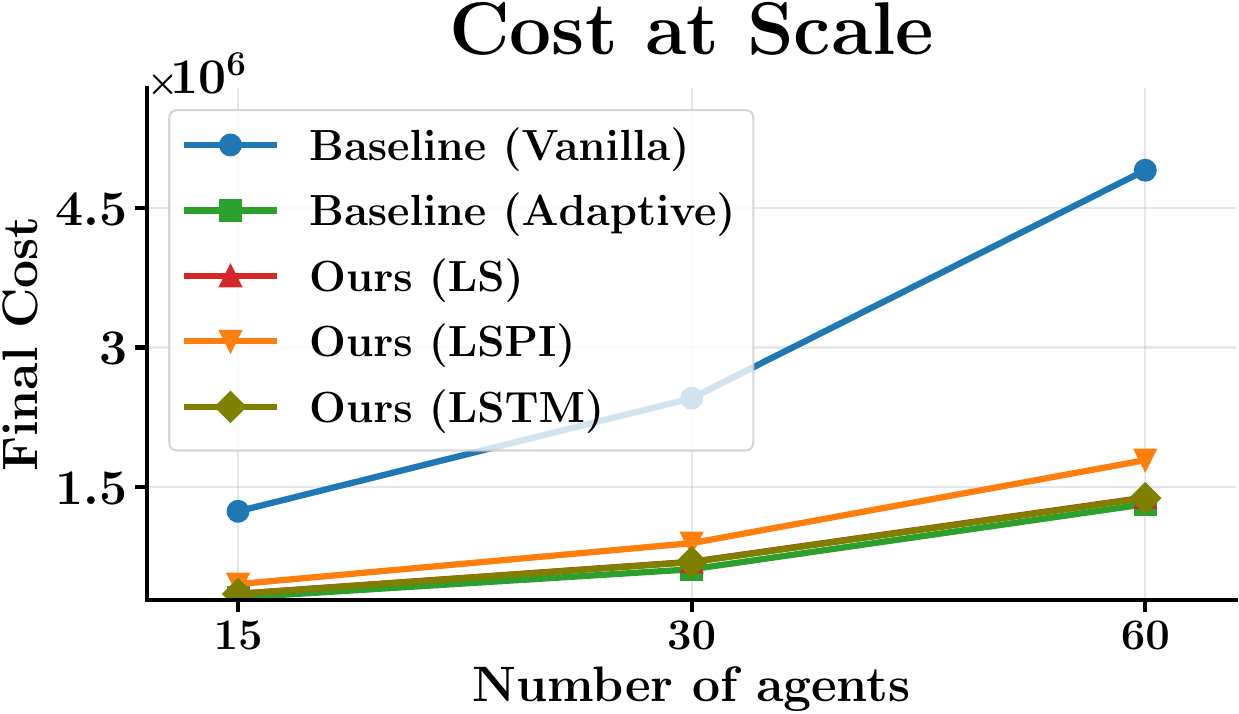} &
    \hspace{-14pt}
    \includegraphics[width=0.33\linewidth]{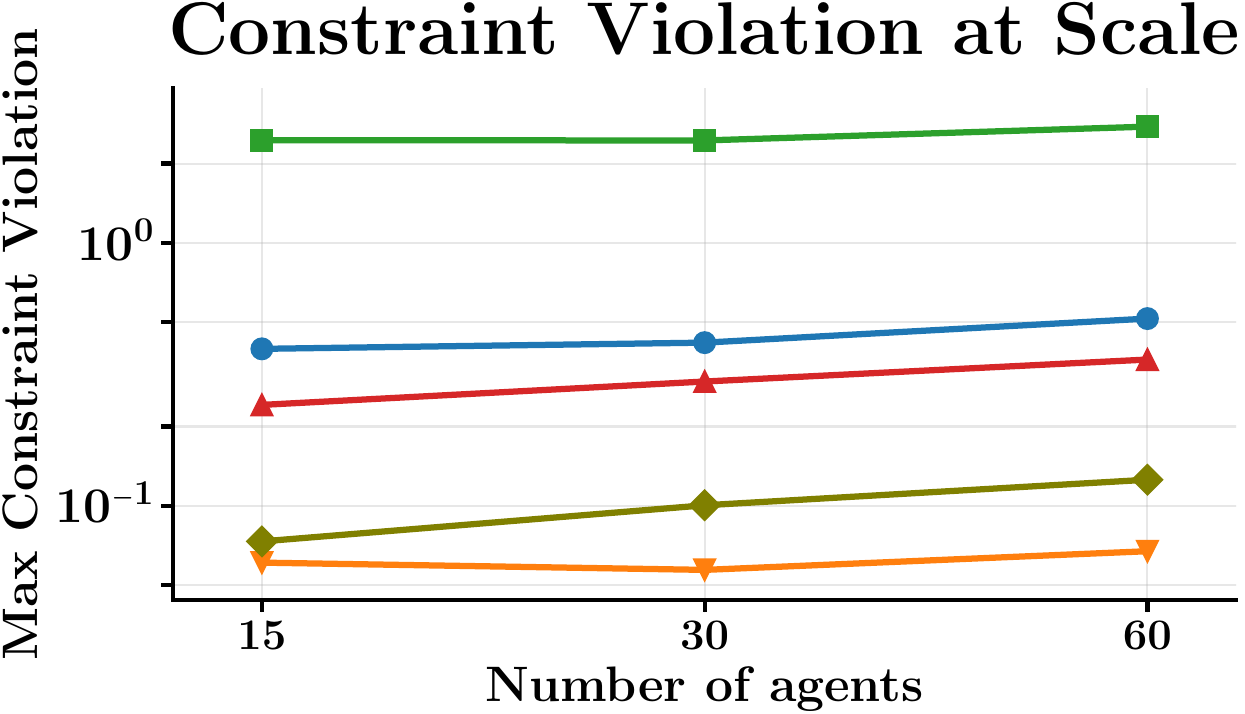} &
    \hspace{-14pt}
    \includegraphics[width=0.33\linewidth]{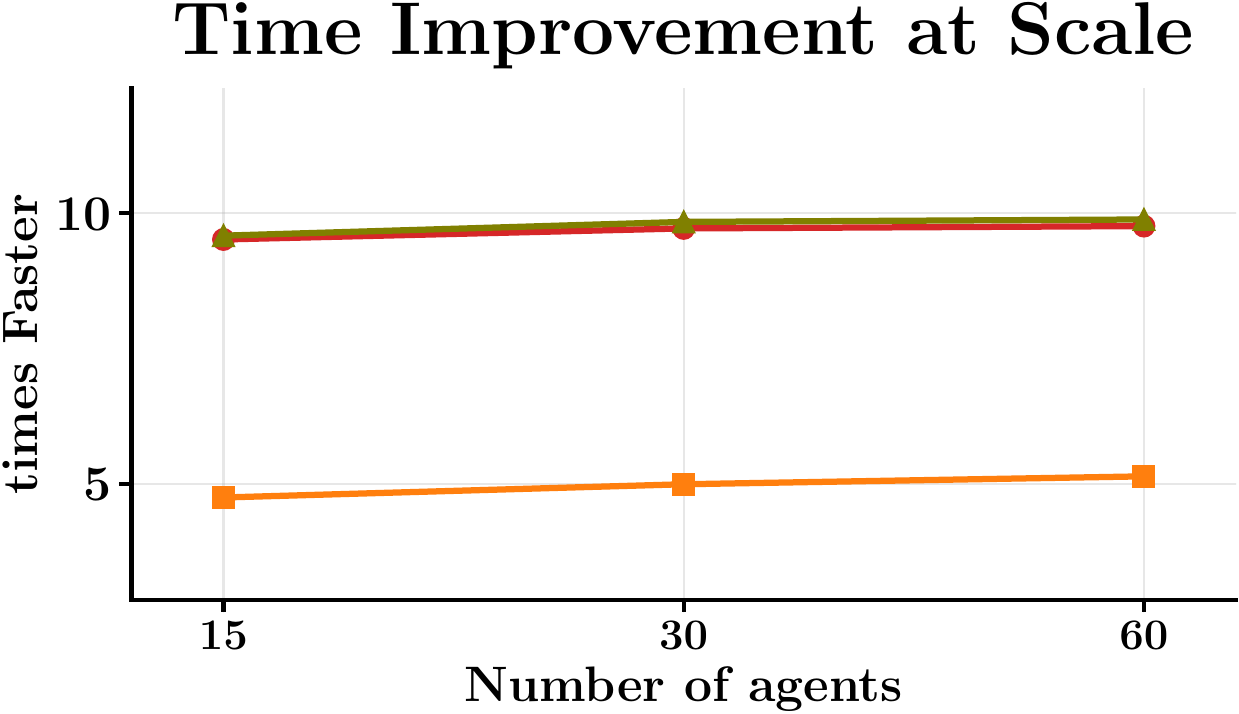} 
    \hspace{-10pt} \\
    \hspace{-10pt}
    \includegraphics[width=0.33\linewidth]{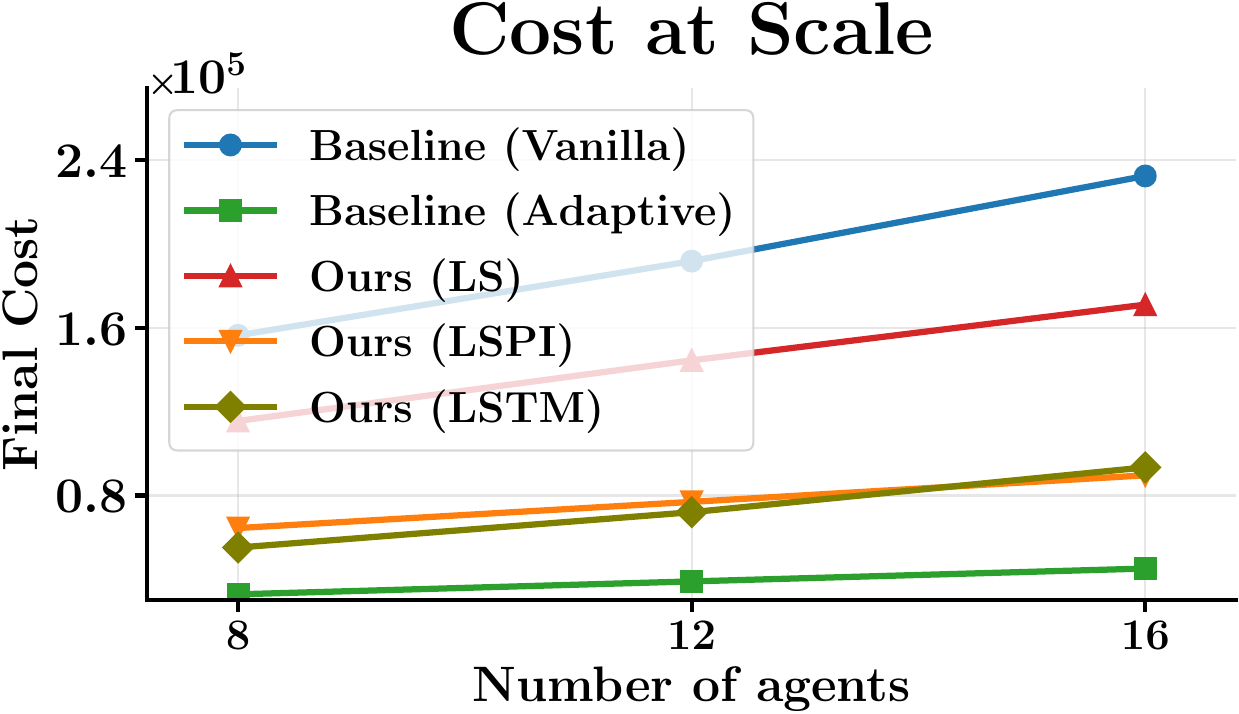} &
    \hspace{-14pt}
    \includegraphics[width=0.33\linewidth]{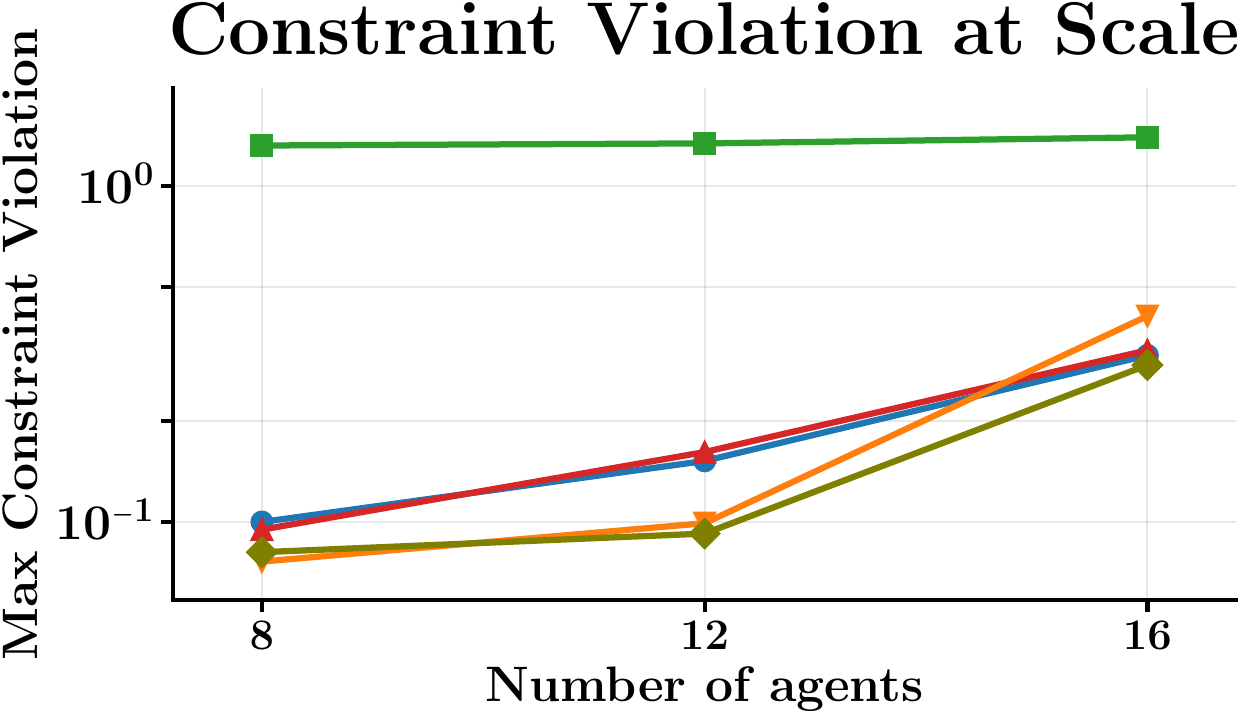} &
    \hspace{-14pt}
    \includegraphics[width=0.33\linewidth]{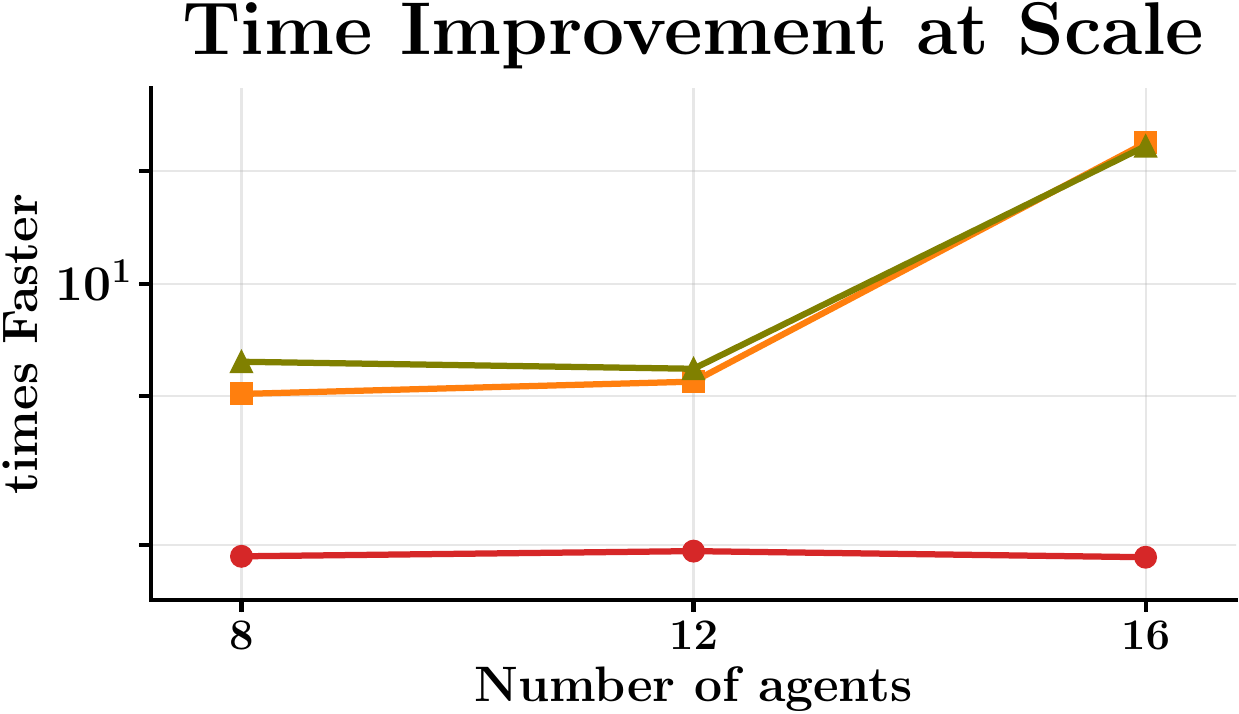} 
    \hspace{-10pt} \\
    \hspace{-10pt}
    \includegraphics[width=0.33\linewidth]{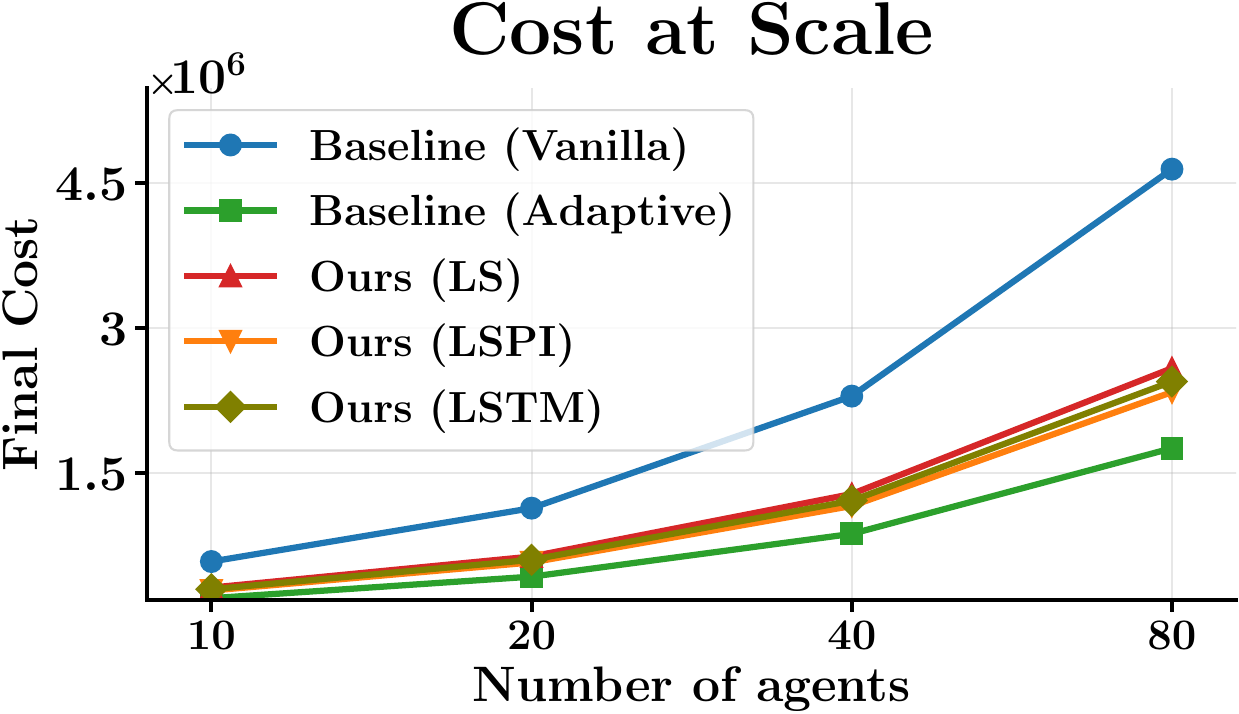} &
    \hspace{-14pt}
    \includegraphics[width=0.33\linewidth]{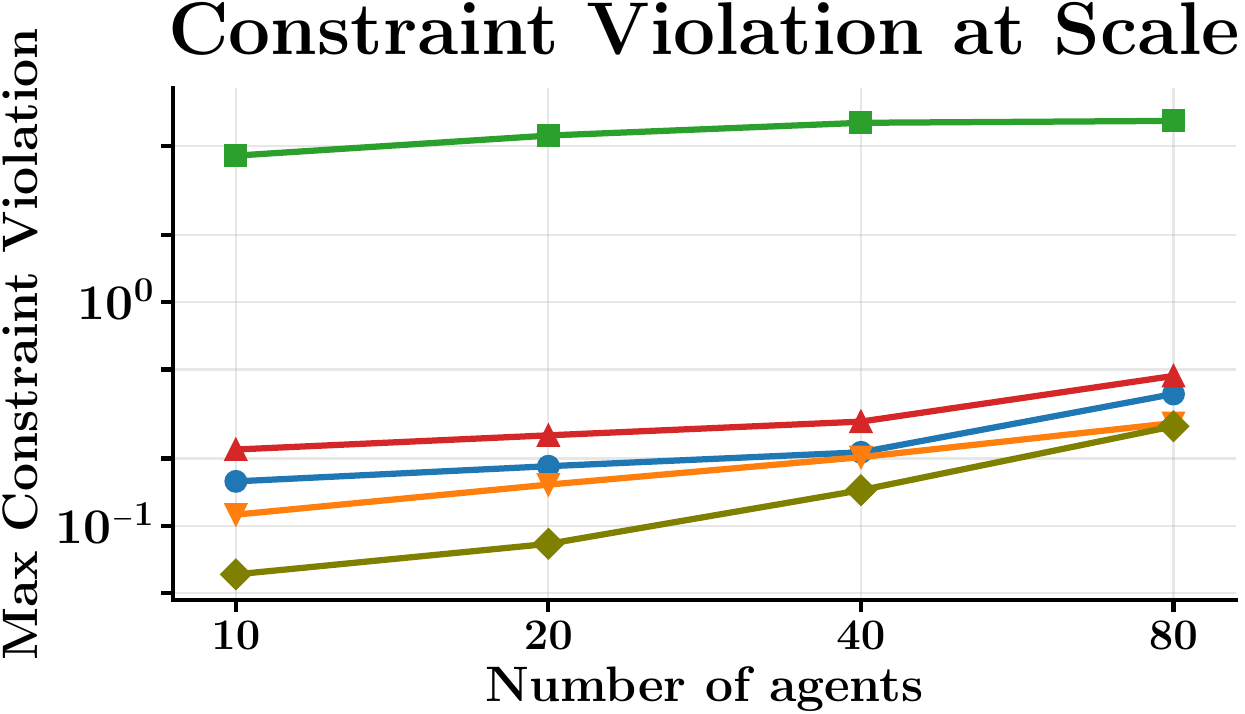} &
    \hspace{-14pt}
    \includegraphics[width=0.33\linewidth]{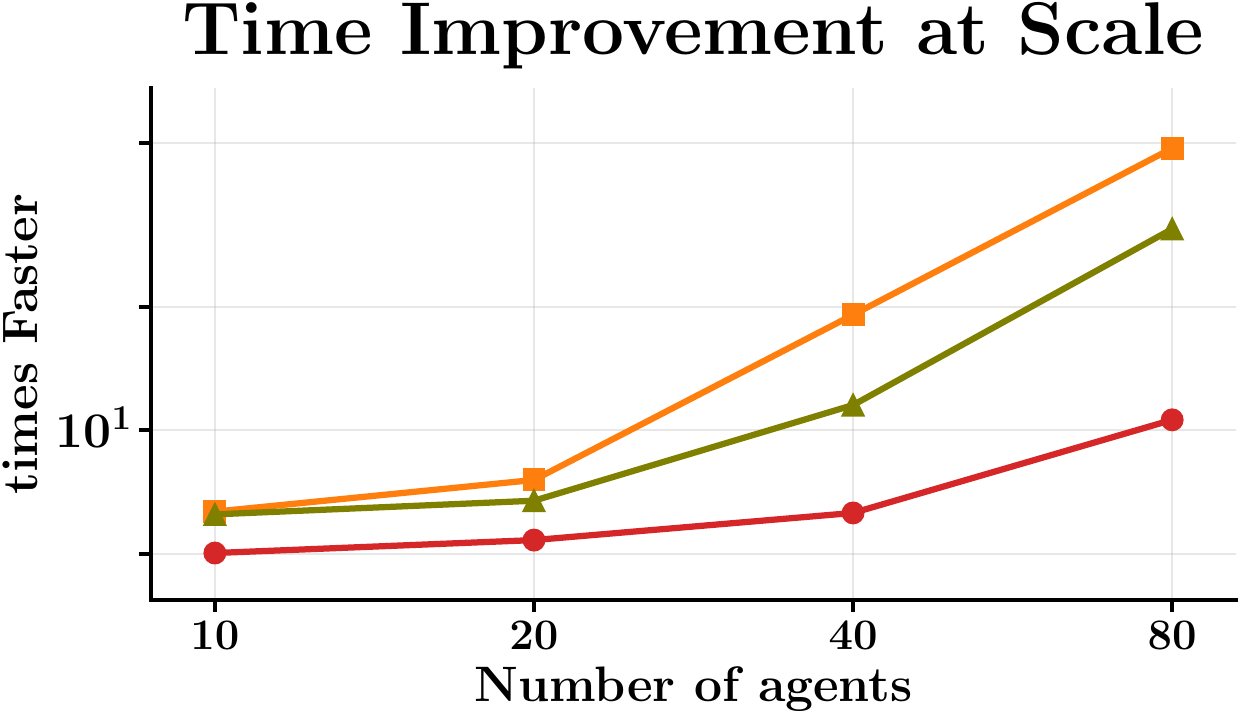} 
    \hspace{-10pt}
    \end{tabular}
\vspace{-6pt}
\caption{Deployment results as the number of agents scales. For each task type and scaling factor, we test on 20 problems and average the results. From top to bottom, the rows correspond to the car obstacle field, car intersection, and quadrotor obstacle field task, respectively.
% Detailed numerical results and plotting details are provided in Table~\ref{tab:scaling_details}.
}
\vspace{-12pt}
\label{fig:scaling}
\end{figure}

\subsection{Scaling Experiments}
Next, we study how Deep Coordinator's learned update policies generalize to larger and different teams of agents than those seen during training. We deploy the trained \deepcoordinator models to scaled tasks. For each task type and scaling factor, we test on twenty problems and average the results. Fig. \ref{fig:scaling} summarizes the results, while a table of numerical values is in Table~\ref{tab:scaling_details} in Appendix~\ref{app:detailed_experimental_results}. For details on scaled problem generation, consult Appendix~\ref{app:details_on_scaled_problem_types}.

In the car and quadrotor obstacle field tasks, scaling largely preserves the local structure observed by each agent. All three \deepcoordinator parameterizations generally preserve their advantage over Vanilla ADMM-DDP as the number of agents increases. 

The car intersection task presents a more significant challenge, as increasing the number of agents fundamentally alters the underlying coordination problem and the optimal right-of-way behavior. Nevertheless, \deepcoordinator continues to find significantly lower-cost and higher-constraint satisfaction solutions than Vanilla ADMM-DDP, while retaining its 5x speedup in the LSPI and LSTM cases.

Standard black-box neural networks trained on such a narrow and small dataset of tasks typically perform poorly in such out-of-distribution regimes due to overfitting to the agent topologies and spatial trajectories found in the training dataset. Deep Coordinator avoids this by delegating the specific inter-agent coordination behavior to the ADMM-DDP optimizer and instead learning to control the optimization process itself. The feedback policies it learns are local, taking ADMM quantities specific to a single (agent, timestep) pair which reflect whether that agent's trajectory is consistent with its neighbors', regardless of how many neighbors are present or how they are arranged. These signals remain structurally similar across tasks, meaning the policy stays in-distribution with respect to the optimizer, even when deployed on problems well outside of its training distribution. In this sense, the architectural prior of the ADMM-DDP structure, rather than the expressive capacity of the neural network alone, enables Deep Coordinator's generalization.

We perform additional experiments and analysis ablating the impact of feedback and empirically validating the performance improvements of the unsupervised loss in Appendix~\ref{app:additional_experiments}. 

\newpage
\section{Limitations} 
In principle, non-convex deep-unfolded optimizers can be trained to find globally-optimal solutions. However, our work is focused on accelerating the convergence of ADMM-DDP to high-quality local optima. Training to find global solutions would require incorporating a mechanism to explore the optimization landscape (e.g., training with the exploration mechanism of RL or unfolding a global optimizer). Future work will explore training with RL and unfolding global optimizers to train models which can quickly converge to the globally-optimal trajectories. 

Furthermore, while \deepcoordinator performs well within its iteration budget, it has no canonical mechanism to extend the optimization process beyond this, if needed. In principle, the Learned Scalar and LSTM policies can be deployed for more than $K$ ADMM-DDP iterations; however, there is no guarantee that the policy will generalize well beyond this point. Other deep-unfolding work
\citep{sambharya2024learning} proposes a progressive training scheme, which is a promising mechanism to address this.

\bibliography{references}  % .bib

\newpage
\appendix

\newpage
\section{Detailed Description of the ADMM-DDP Subproblems}
\label{app:sec:subproblems_of_admm_ddp}

In this appendix, we provide the mathematical breakdown of each subproblem of ADMM-DDP.

\textbf{Subproblem 1. } 
For each agent $i = 1, ..., N$, we solve
\begin{equation} \label{eq:subproblem_1}
\begin{aligned}[c]
    \bx_i^{a+1}, \bu_i^{a+1} = \argmin_{\{\bx_i, \bu_i\}} \hquad & \sum_{t=0}^{T-1} \hat{\paramell{}}_{i, t}(\bx_{i, t}, \bu_{i, t}) + \hat{\paramell{}}_{i, T}(\bx_{i, T}) \\
    \text{subject to} \hquad & \bx_{i,t+1} = \f{i}(\bx_{i,t}, \bu_{i,t})
\end{aligned}
\end{equation}
where the augmented stage costs are defined as
\begin{equation*}
\begin{aligned}
    \hat{\paramell{}}_{i, t}(\bx_{i, t}, \bu_{i, t}) &= \paramell{i,t}(\bx_{i, t}, \bu_{i, t}) + \frac{\paramrho{i, t}}{2} \left\| \bx_{i,t} - \tx_{i,t}^a + \paramrho{i, t}^{-1} \blambda_{i,t}^a \right\|_2^2 + \frac{\parammu{i, t}}{2} \left\| \bu_{i,t} - \tu_{i,t}^a + \parammu{i, t}^{-1} \bxi_{i,t}^a \right\|_2^2,
\end{aligned}
\end{equation*}

and the terminal cost is defined as
\begin{equation*}
\begin{aligned}
    \hat{\paramell{}}_{i,T}(\bx_{i, T}) &= \paramell{i, T}(\bx_{i, T}) + \frac{\paramrho{i, T}}{2} \left\| \bx_{i,T} - \tx_{i,T}^a + \paramrho{i, T}^{-1} \blambda_{i,T}^a \right\|_2^2.
\end{aligned}
\end{equation*}

\textbf{Subproblem 2. } 
For each time $t = 0, \ldots, T - 1$, we solve
\begin{equation} \label{eq:subproblem_2}
\begin{aligned}[c] 
    \hspace{-1em}\tx_t^{a+1}, \tu_t^{a+1} = \argmin_{\tx_t, \tu_t} & \sum_{i=1}^N \Biggl[ \frac{\paramrho{i, t}}{2} \left\| \bx_{i,t}^{a+1} - \tx_{i,t} + \paramrho{i, t}^{-1} \blambda_{i,t}^a \right\|_2^2 + \frac{\parammu{i, t}}{2} \left\| \bu_{i,t}^{a+1} - \tu_{i,t} + \parammu{i, t}^{-1} \bxi_{i,t}^a \right\|_2^2 \Biggr] \\
    \text{subject to} \hquad & \g{i, t}(\tx_{i,t}, \tu_{i,t}) \leq 0, \\
    & \h{ij,t}(\tx_{i,t}, \tx_{j,t}, \tu_{i,t}, \tu_{j,t}) \leq 0.
\end{aligned}
\end{equation}

\noindent
For time $t = T$, the problem is of the form
\begin{equation}
\begin{aligned}
    \tx_T^{a+1} = \argmin_{\tx_T} \hquad & \sum_{i=1}^N \frac{\paramrho{i, T}}{2} \left\| \bx_{i,T}^{a+1} - \tx_{i,T} + \paramrho{i, T}^{-1} \blambda_{i,T}^a \right\|_2^2, \notag \\
    \text{subject to} \hquad & \g{i, T}(\tx_{i,T}) \leq 0, \notag \\
    & \h{ij,T}(\tx_{i,T}, \tx_{j,T}) \leq 0.
\end{aligned}
\end{equation}

\textbf{Subproblem 3. } 
Across all agents and timesteps, we update the dual variables with
\begin{equation} \label{eq:subproblem_3}
\begin{aligned}[c]
    \blambda_{i,t}^{a+1} &= \blambda_{i,t}^a + \paramrho{i, t} (\bx_{i,t}^{a+1} - \tx_{i,t}^{a+1}), \\
    \bxi_{i,t}^{a+1} &= \bxi_{i,t}^a + \parammu{i, t} (\bu_{i,t}^{a+1} - \tu_{i,t}^{a+1}).
\end{aligned}
\end{equation}

\newpage
\section{Choosing a Feedback Scheme}\label{app:choosing_a_feedback_scheme}
There are many reasonable choices of feedback policies $\pi_{w^a}$. For example, a straightforward approach is to learn $\sigma$ and $\chi_{\textrm{\{incr},  \,\textrm{decr\}}}$ in the ADMM-DDP penalty parameter adaptation rule proposed in~\citep{saravanos2022distributed}. However, given the distributed nature of ADMM-DDP, we particularly desire feedback schemes that have the capacity to:
\begin{enumerate}
    \item Independently adapt each element of the penalty vectors $\boldsymbol{\rho}^a \in \R^{N \times (T + 1) \times n_x}$ and $\boldsymbol{\mu}^a \in \R^{N \times T \times n_u}$, and
    \item Deploy to problems with longer time horizons and larger or different graph topologies underlying the agents than trained on
\end{enumerate}
The first property allows for policies that capture complex relationships between hyperparameters that would be impossible to represent with simpler schemes (e.g., re-scaling a fixed vector). The second property not only increases the flexibility of the trained optimizer, but also allows models to be trained on smaller problems then deployed to much larger ones. This significantly reduces the computational burden of training. 

Since the penalty vectors scale with the number of agents and timesteps, any policy that has a fixed output size---such as a single MLP---cannot have property two. To address this, we consider \emph{shared per-agent, per-timestep policies} which utilize shared weights. Under this scheme, a shared policy $\varphi_{w^a}$ maps agent- and timestep-specific inputs to the associated components of the hyperparameter vector. While the policy weights are shared, the varying inputs enable the network to provide context-aware local updates. 

To ensure scalability, these inputs are derived from $\boldsymbol{\theta}^{a - 1}$, $\bv^{a-1}$ and $\boldsymbol{\chi}$ in a manner which is invariant to the problem dimension; examples of such inputs include the relevant components of the ADMM residual vector, dual variables, or trajectories. The overall feedback scheme can be formulated as
\begin{equation*}
\begin{split}
    (\pi_{w^a})_{i, t} = \varphi_{w^a} \circ \psi_{i, t},
\end{split}
\end{equation*}
where $\psi_{i, t}$ extracts the local inputs for agent $i$ and timestep $t$. 

\newpage
\section{Illustrative Example For The Supervised vs Unsupervised Loss}\label{app:illustrative_example}
To build intuition for why the supervised framework can be poorly suited for training non-convex deep-unfolded models, we present the following illustrative 1D example employing a simpler optimizer. 
Consider the task of unfolding $K = 10$ iterations of a gradient descent algorithm, characterized by the update equation
\begin{equation}
\begin{split}
    x^{a + 1} = x^{a} - \alpha \nabla_x f(x^a) 
\end{split}
\end{equation}
where $f$ is the objective function of the current problem. In this case, we can learn a scalar per iteration for the step-size hyperparameter $\alpha$. Consider the following training problem
\begin{equation}
\begin{split}
    \min_{\bx \in \mathbb{R}} \ \left(x-0.3\right)^{8}-\left(x-0.25\right)^{2}+0.55475
\end{split}
\end{equation}
where the constant term is added to set the global minimum to zero. This function exhibits a local minimum at $x \approx -0.4850$ and a global minimum at $x \approx 1.1017$. 

Suppose we train the resulting unfolded model using the supervised scheme introduced in Section~\ref{subsection:trainingloop} with the global minimizer as the ground-truth solution. Lastly, suppose the initial iterate $x^0$ is zero. We've plotted the objective function, along with the resulting loss function and the initial iterate in Fig.~\ref{fig:functions_plot}.

\begin{figure}[htbp]
    \centering
    \includegraphics[width=0.6\textwidth]{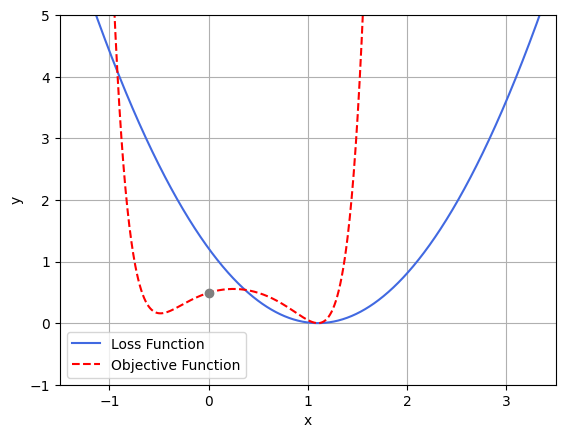}
    \caption{A comparison of the loss function (solid blue) and objective function (dashed red) in the illustrative example. The initial iterate is shown as a grey dot.}
    \label{fig:functions_plot}
\end{figure}

\noindent
With sufficiently small initial alpha values, the optimizer will converge to the local minimum at $x \approx -0.4850$. This is suboptimal with respect to the loss, which penalizes distance away from the global minimizer. 

To find the loss-minimizing solution, the upper-level training scheme is forced to  discover that large step sizes cause the optimizer to jump to the other basin. If it is unable to do so, the locally loss-minimizing choice of alpha---among those which do not cause the optimizer to ascend on the objective function---is to set all alpha values to zero. This prevents the optimizer from moving further away from the initial iterate. In this manner, when the initial iterates of an unfolded optimizer begin in a basin which does not contain the ground-truth solution, the training scheme may push the parameters to degenerate values.

\newpage
\section{Additional Analysis and Experiments}\label{app:additional_experiments}

\subsection{Ablating the Impact of Feedback}
We now compare the performance of the three feedback policy parameterizations, based on the results in Section \ref{sec:experimental_results}. The Learned Scalar policy represents the optimal hand-tuning of the penalty parameters. We, therefore, expect the more sophisticated policies to outperform LS. Indeed, they do. The LSTM policy, in particular, achieves lower cost and constraint violation than LS in all cases. 

We also observe that the LSTM exhibits favorable behavior compared to LSPI. In the quadrotor obstacle field task, the LSTM achieves moderately (1.85x) lower constraint violation while incurring slightly higher (1.05x) cost. In the car obstacle field and car intersection tasks, the LSTM achieves similar (0.83x-0.93x lower) constraint violation while incurring lower costs (1.17-1.31x) than LSPI. This improvement is meaningful, but lower than in the quadrotor case. The difference in the prioritization of cost and constraint satisfaction is determined by the choice of scaling parameters in the unsupervised loss. In principle, this difference could be alleviated by performing a finer-grained sweep of these parameters.

\begin{figure}[h]
\hspace{-16pt}
\centering
    \begin{tabular}{cc}
    \vspace{-2pt}
    \includegraphics[width=0.49\linewidth]{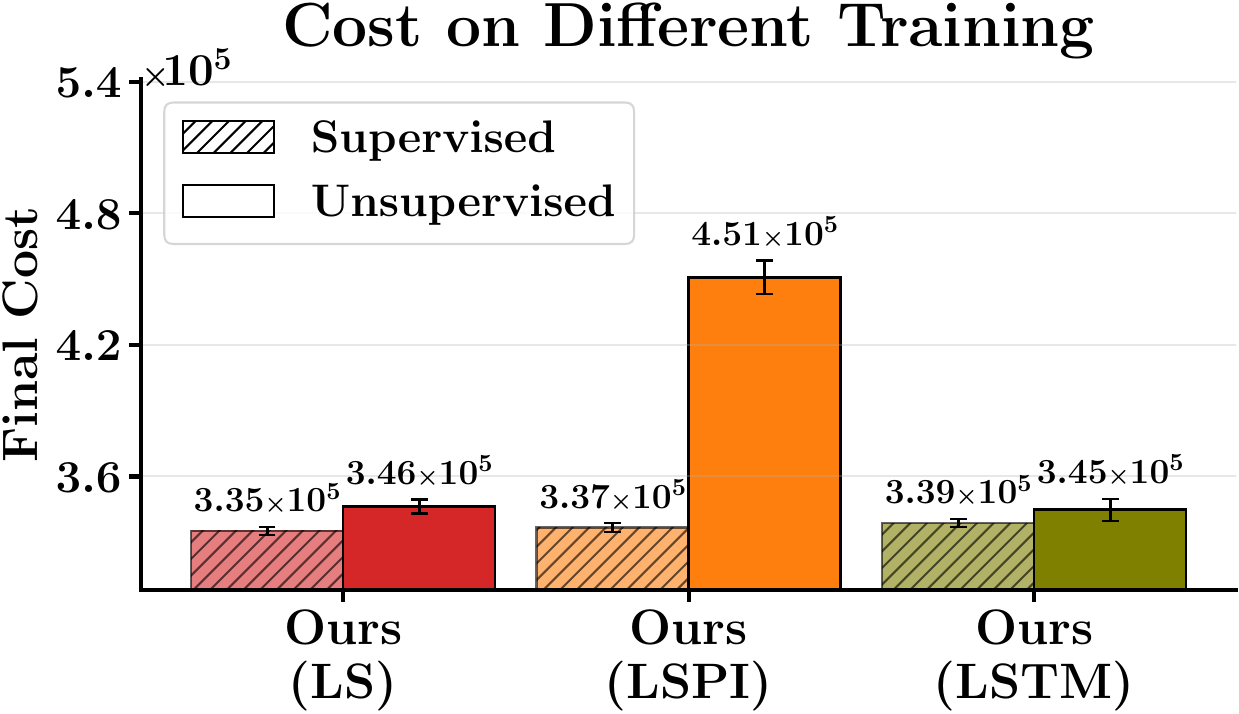} &
\hspace{-12pt}
    \includegraphics[width=0.49\linewidth]{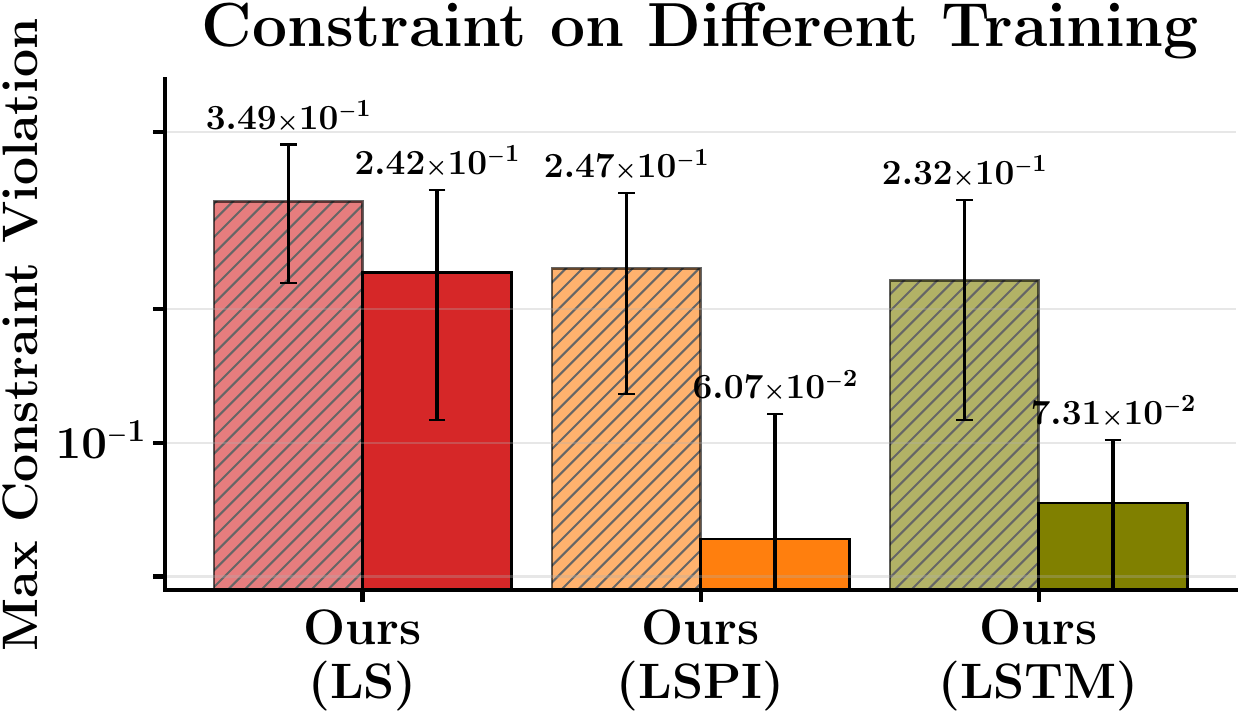} \\
    \vspace{-2pt}
    \includegraphics[width=0.49\linewidth]{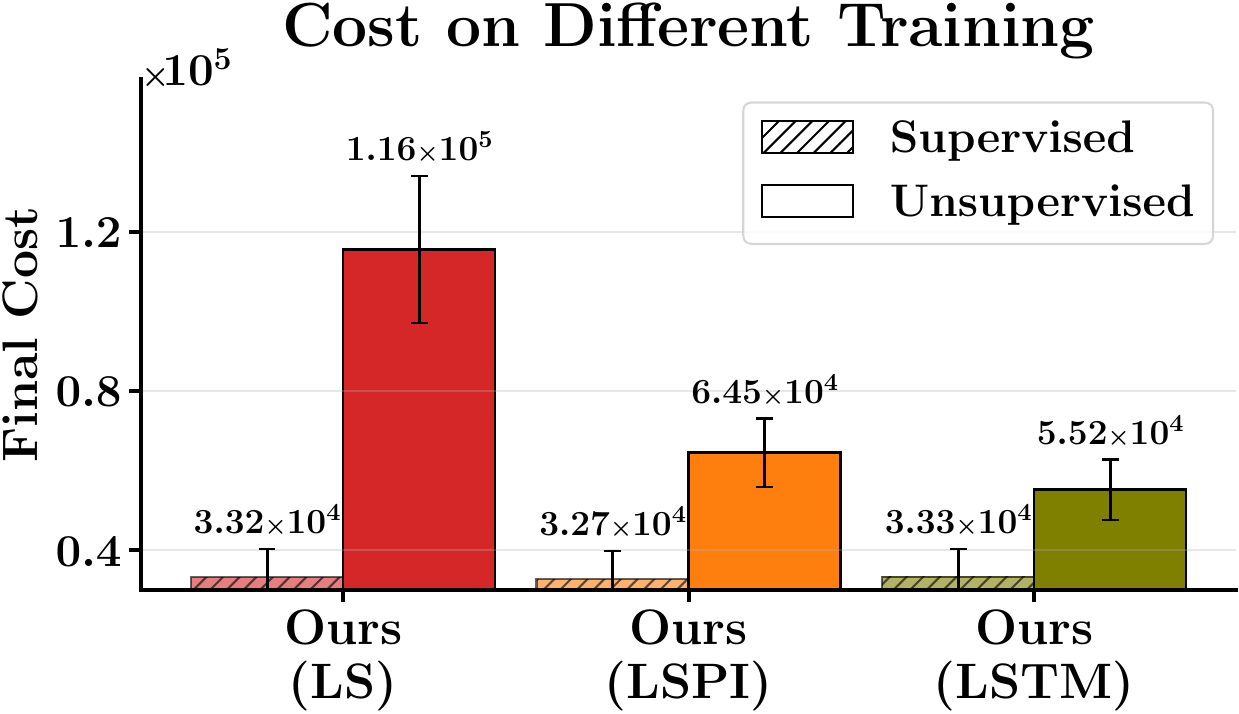} &
\hspace{-12pt}
    \includegraphics[width=0.49\linewidth]{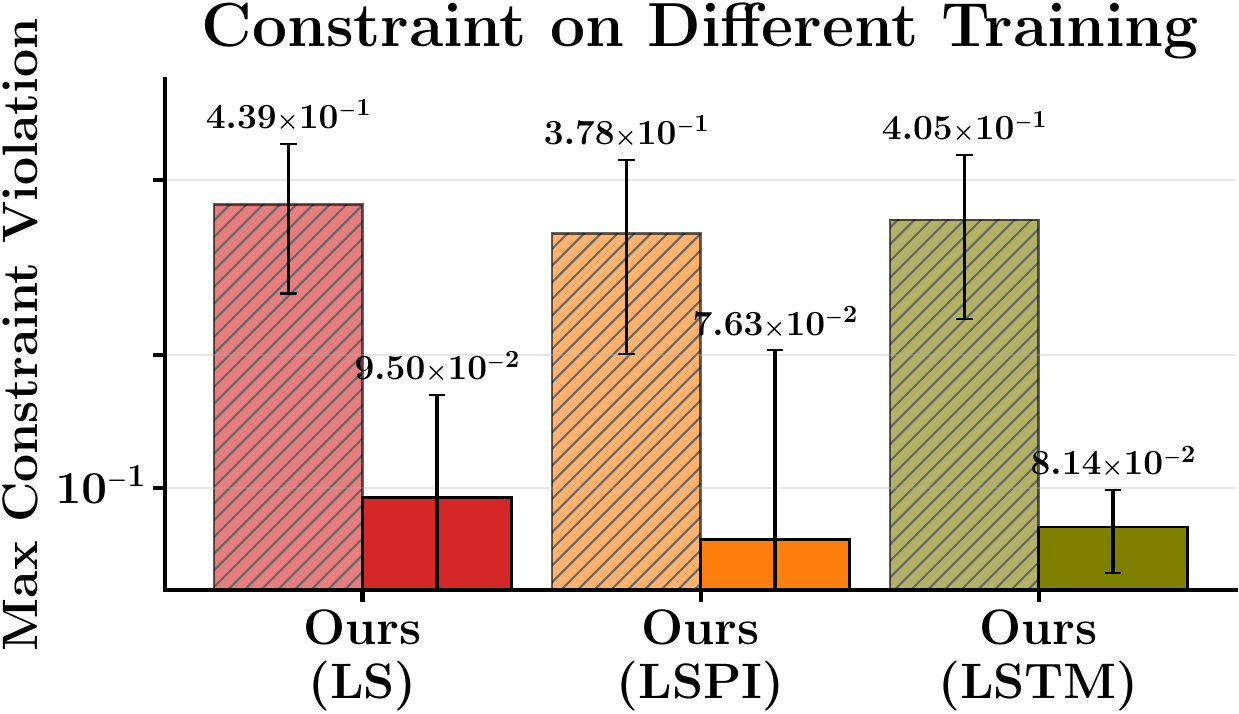} \\
    \vspace{-2pt}
    \includegraphics[width=0.49\linewidth]{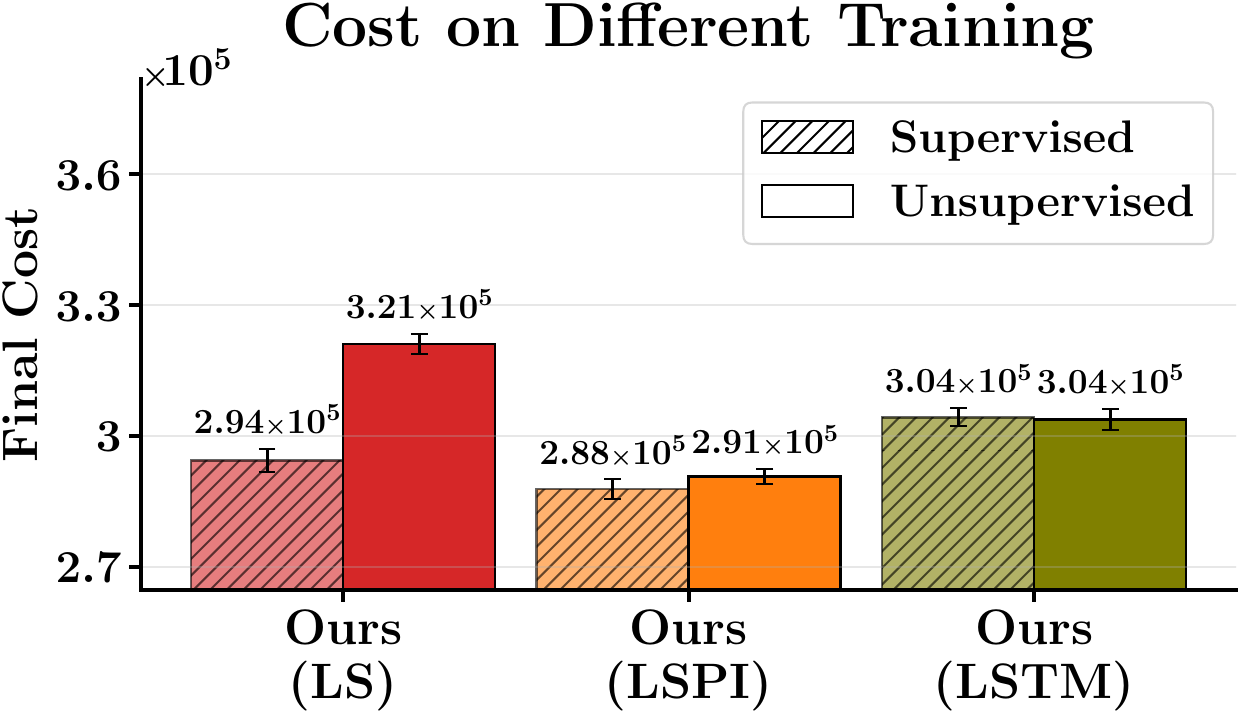} &
\hspace{-12pt}
    \includegraphics[width=0.49\linewidth]{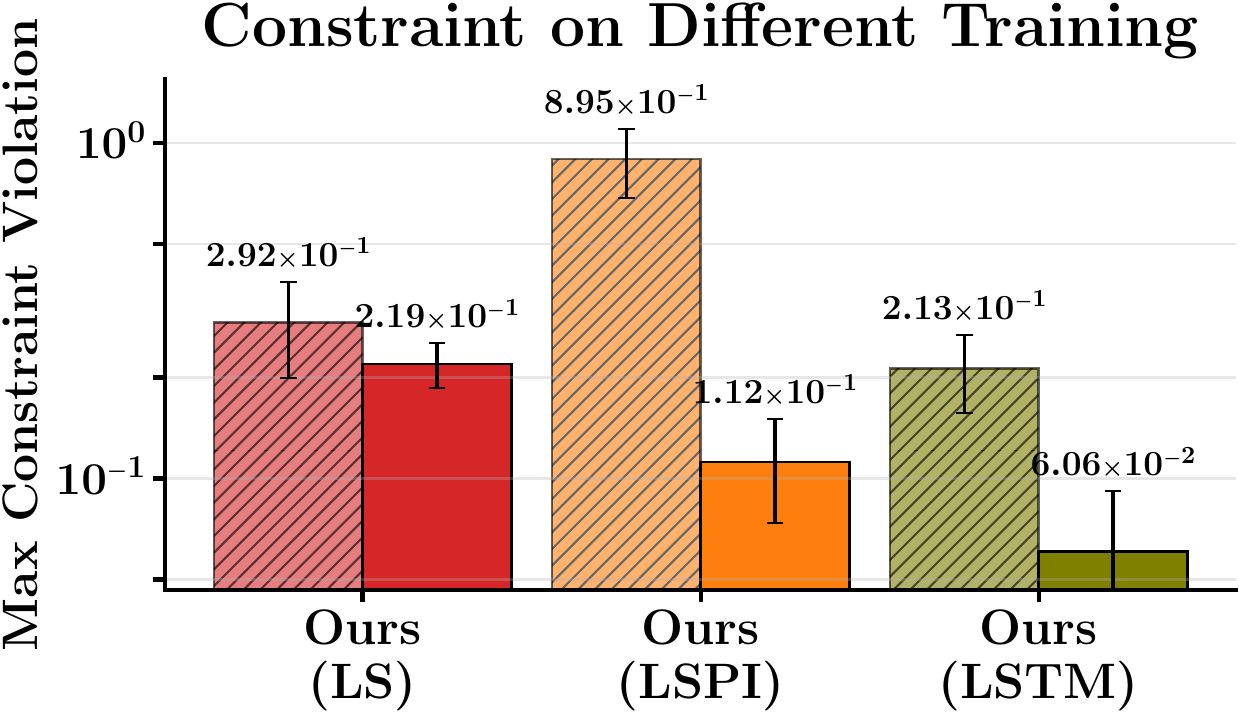}
    \end{tabular}
\hspace{-16pt}
\caption{Comparison of cost and constraint values generated by policies trained under different training schemes. Hatched bars correspond to policies trained using the supervised scheme, while plain bars correspond to policies trained using the unsupervised scheme. The rows correspond to car obstacle field (top), car intersection (middle), and quadrotor obstacle field (bottom). Error bars indicate standard deviation.}
\label{fig:sup_vs_unsup}
\end{figure}

In addition to its favorable performance over the other policies, the LSTM adds minimal overhead in wall-clock time. This is because the computational cost of solving the inner-loop ADMM-DDP subproblems is significantly higher than the cost of forward-propagating the LSTM. Furthermore, the LSTM achieves \emph{lower} wall-clock time than LSPI in the quadrotor case. This is due to the penalty parameters predicted by LSPI yielding inner-loop ADMM-DDP subproblems which take longer to solve.

A candidate explanation for the superior performance of the LSTM is that the increased number of parameters in the LSTM-based policies allow the model to find a better (potentially degenerate, constant) solution rather than actually incorporating feedback. Analysis of the policy outputs demonstrates that this is not the case. For example, the LSTM trained on the car obstacle field task produces high variance outputs with values below $10^1$ and above $10^2.$ We have included figures showing a representative sample of the outputs of each policy in Appendix~\ref{app:sample_policy_output}.

\subsection{Comparing the Unsupervised and Supervised Losses}
Finally, we compare the impact of our supervised and unsupervised loss functions on Deep Coordinator's performance by evaluating the cost and maximum constraint violation achieved by models trained on the two loss functions. For a detailed description of the supervised and unsupervised training schemes, consult Appendix~\ref{app:training_setup_details}. The results are shown in Fig. \ref{fig:sup_vs_unsup}. 

The supervised loss yields \deepcoordinator models that often find lower cost trajectories than the unsupervised loss. In particular, the cost achieved by supervised \deepcoordinator models on the car intersection task is 1.66-3.49x lower than that achieved by unsupervised \deepcoordinator models. However, the resulting trajectories are often degenerate, and heavily violate constraints. Additionally, 3 of the 9 supervised policies produced numerically unstable inner-loop problems when attempting to scale, failing to complete their iteration budget. This further indicates the degenerate nature of the trajectories produced by the supervised method. 

Across our three problems, the unsupervised \deepcoordinator models typically yield significantly lower constraint violation than their supervised counterparts. On car problems, LSPI and LSTM \deepcoordinator models achieve between 3.17-4.98x lower constraint violation when trained using the unsupervised loss in lieu of the supervised loss. This effect also occurs in the quadrotor obstacle field case; however the cost values are much closer between frameworks.

Overall, the unsupervised \deepcoordinator models tend to sacrifice some trajectory cost for considerably higher trajectory feasibility, indicating that it is better-aligned with the deployment objectives of \deepcoordinatorperiod 

\newpage
\section{Sample Policy Output}\label{app:sample_policy_output}
In this appendix, we illustrate how the outputs of the feedback policy differ between policy parameterizations.
For each feedback scheme, we display the output $\rho$ and $\mu$ values. We plot a line for each (agent, timestep, state/control dimension) tuple, representing the policy output averaged over the problem instances for this component of the penalty vector. For each line, the area $\pm$ one standard deviation around the mean is drawn as a translucent shaded region around the mean. A colormap over the flattened (agent, timestep, state/control dimension) dimensions is used to distinguish between lines. 
The plots below are from the unsupervised models for the car obstacle field task. 

The LS policy employs a single constant value for $\rho$ and $\mu$, applied uniformly to every problem and iteration. Its output is therefore flat across iterations with zero variance.
\begin{figure}[H]
    \centering
    \begin{tabular}{cc}
        \includegraphics[width=0.45\linewidth]{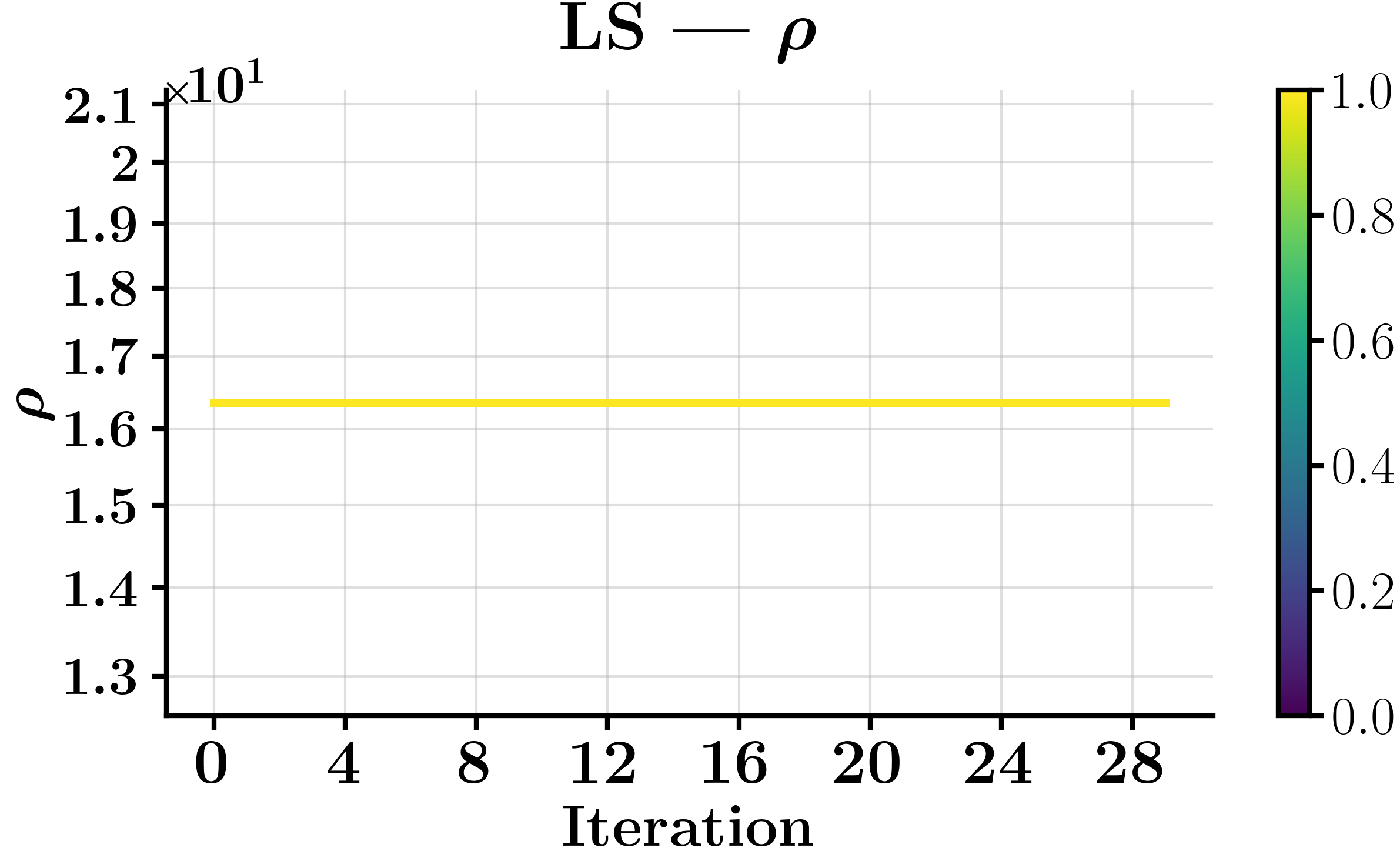} & \includegraphics[width=0.45\linewidth]{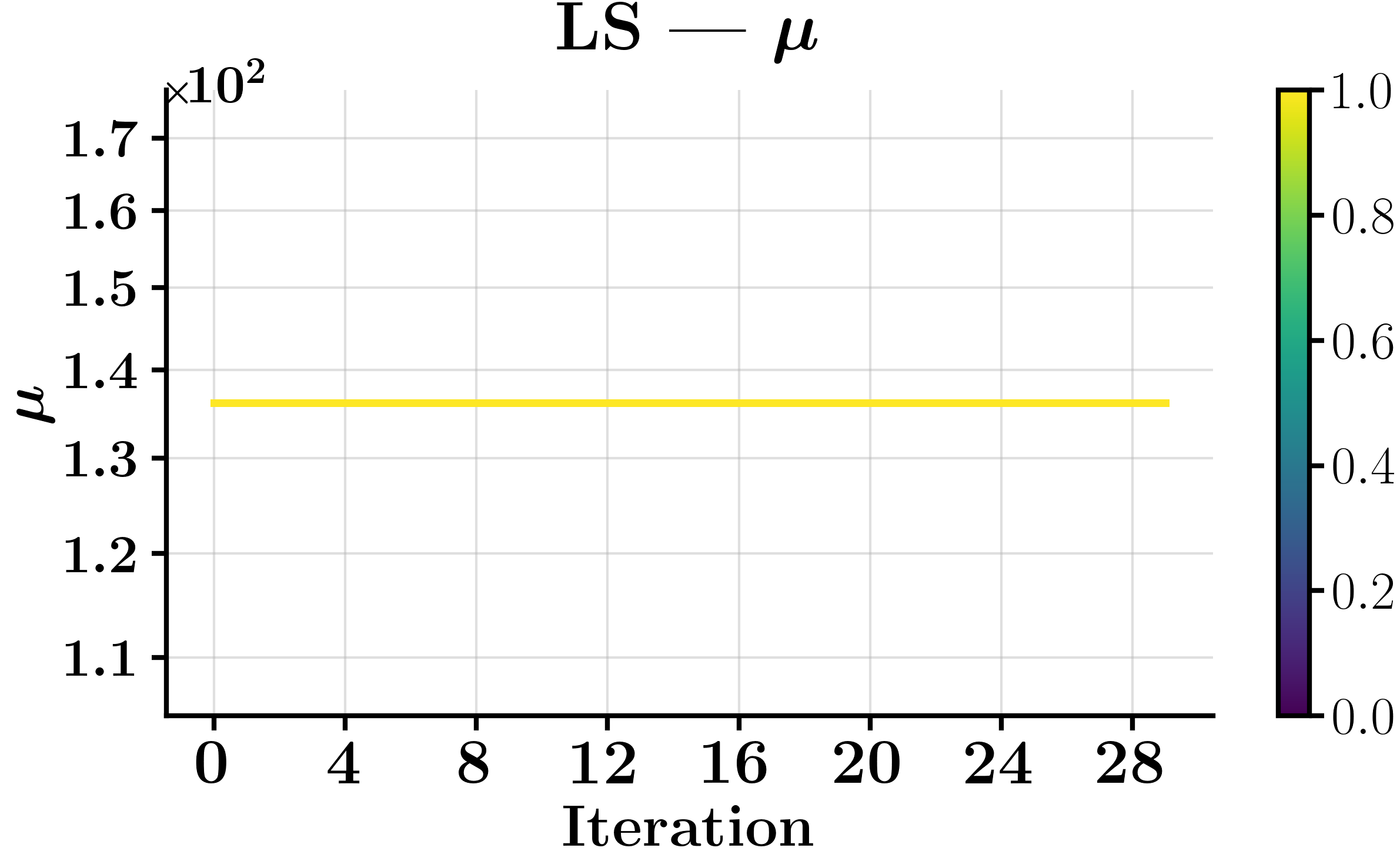}
    \end{tabular}
    \caption{Sample output from unsupervised LS policy at each iteration.}
    \label{tab:sample_rho_and_mu_ls}
\end{figure}

The LSPI policy is also problem agnostic but learns a separate scalar at each iteration, so its mean varies along the iterations while the variance remains zero.
\begin{figure}[H]
    \centering
    \begin{tabular}{cc}
        \includegraphics[width=0.45\linewidth]{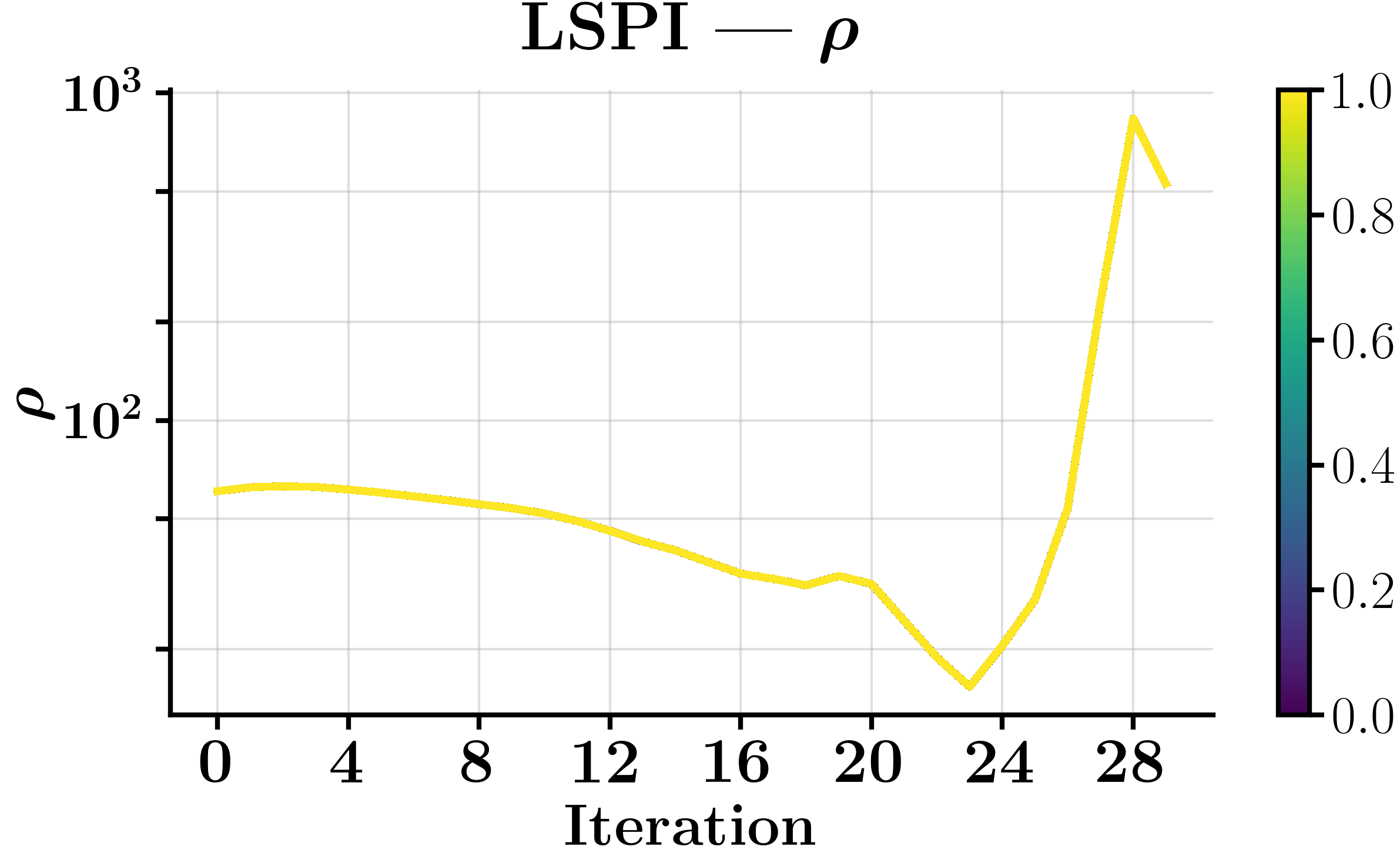} & \includegraphics[width=0.45\linewidth]{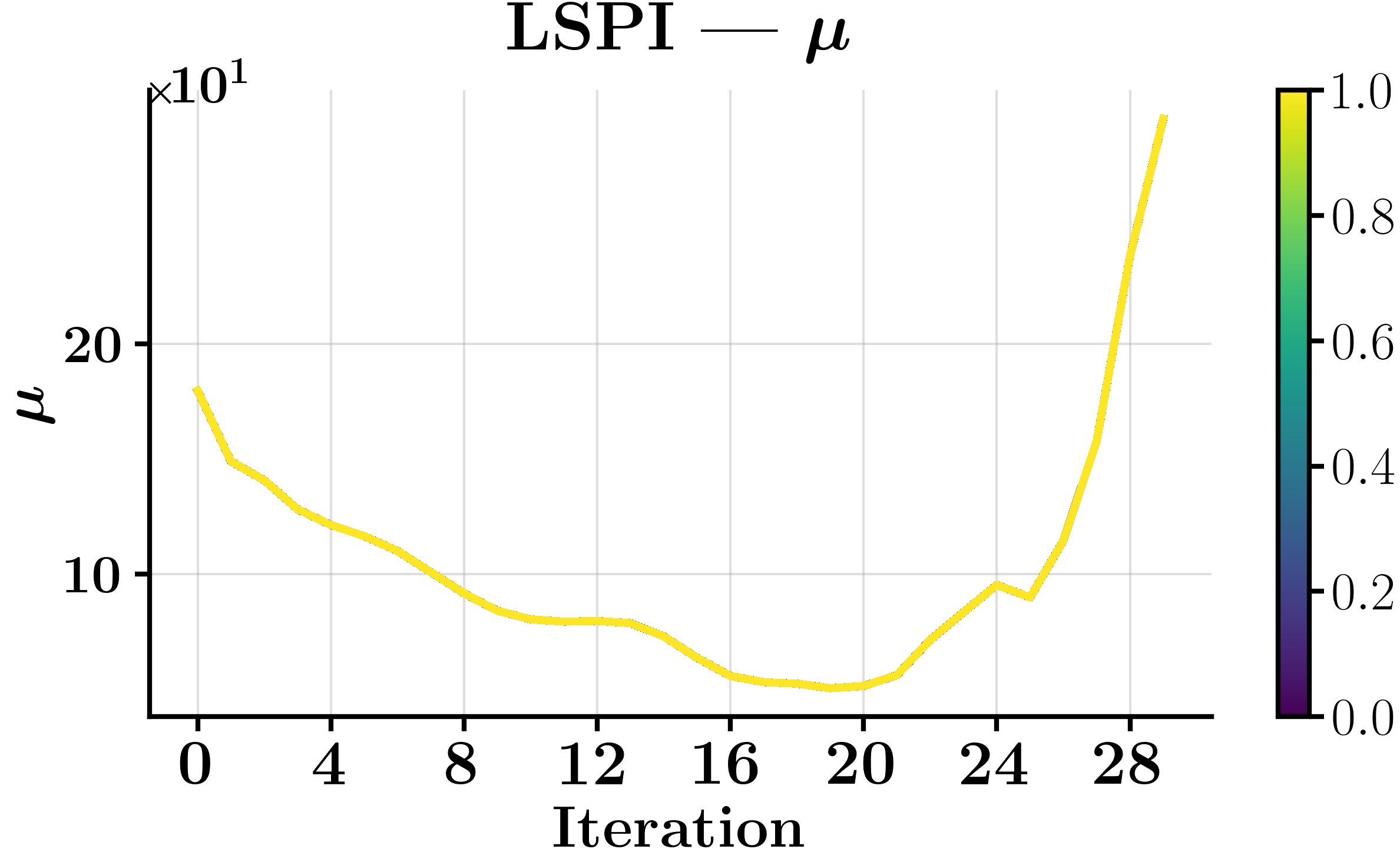}
    \end{tabular}
    \caption{Sample output from unsupervised LSPI at each iteration.}
    \label{tab:sample_rho_and_mu_lspi}
\end{figure}

The LSTM policy produces a distinct value for each entry, with variance even within the same agent, timestep, and state or control dimension.

\begin{figure}[H]
    \centering
    \begin{tabular}{cc}
        \includegraphics[width=0.45\linewidth]{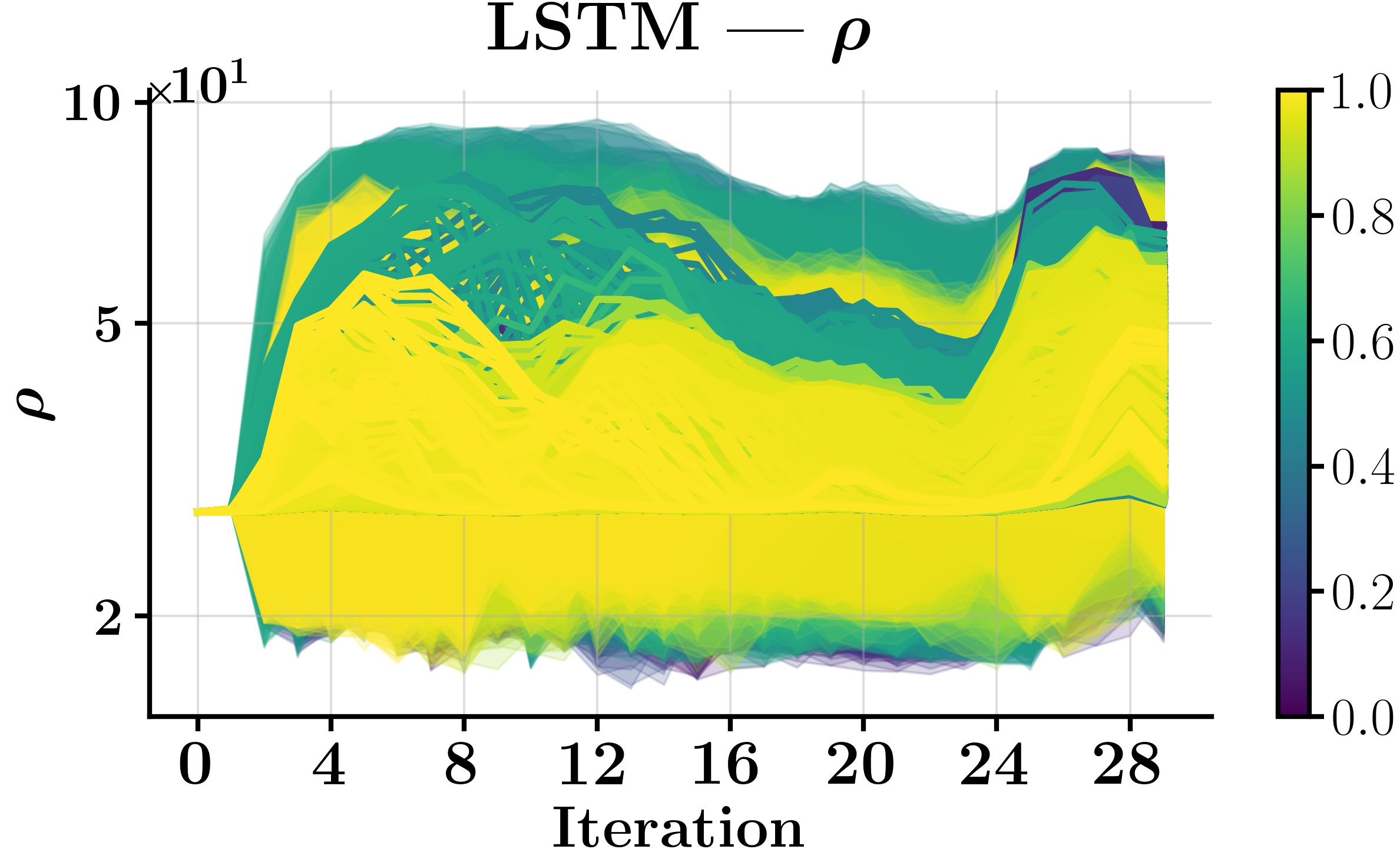} & \includegraphics[width=0.45\linewidth]{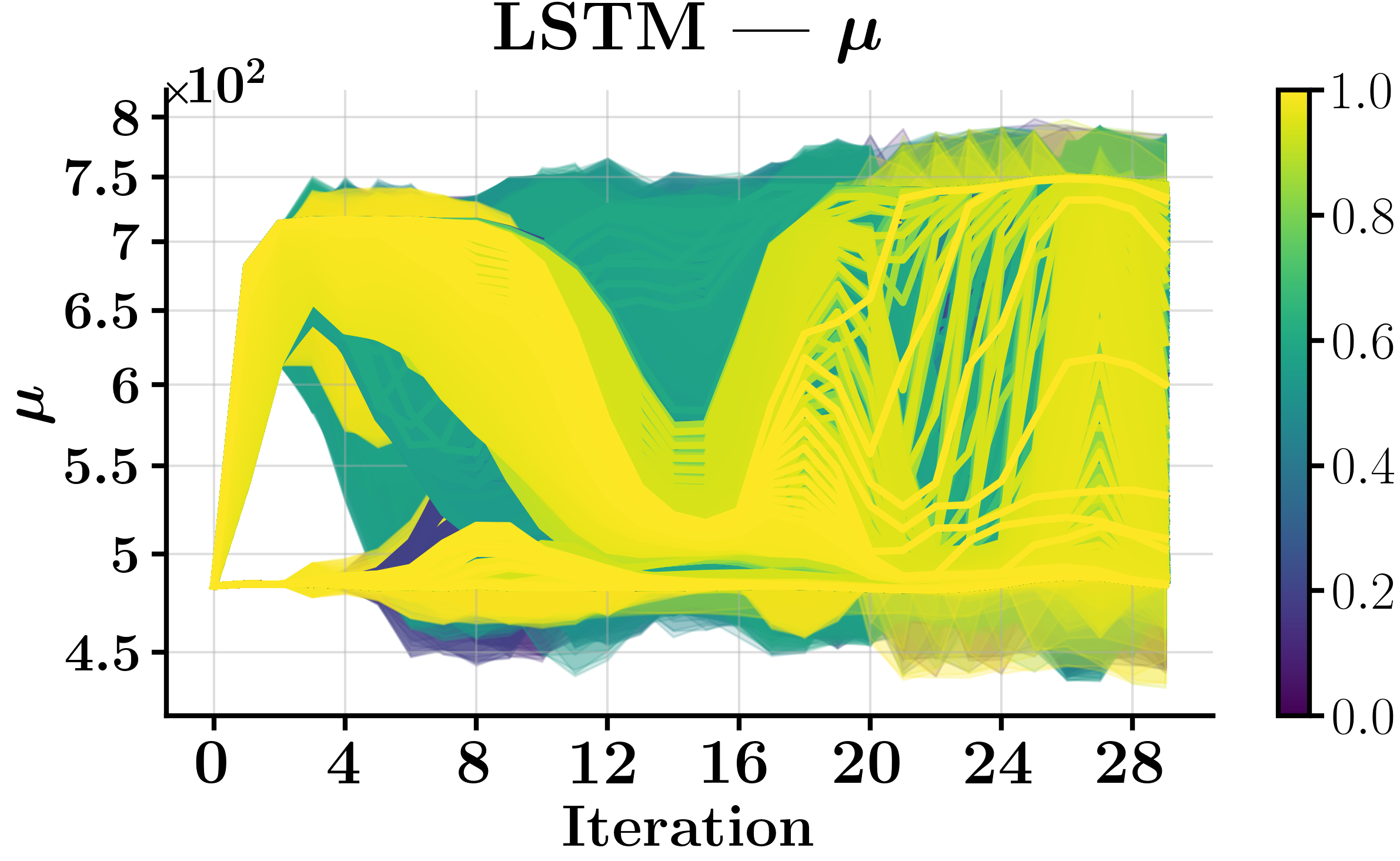}
    \end{tabular}
    \caption{Sample output from unsupervised LSTM policy at each iteration.}
    \label{tab:sample_rho_and_mu_lstm}
\end{figure}

\newpage
\section{Overview of the Implicit Differentiation Framework}\label{app:gradient_computation}
In this appendix, we provide an overview of our IFT-based differentiation framework. For a  detailed derivation and proofs of the theorems, consult Appendix~\ref{app:gradient_computation_detailed_derivation}. 

Our approach is based on differentiating through Subproblem 2 in its barrier-augmented form, which avoids the gradient discontinuities that arise when the active set changes and that compromise training stability.
We derive an analogous approach for computing gradients through Subproblem 1 based on~\citep{oshin2024differentiable}.
This approach is based on DDP and is favorable because it is more memory-efficient compared to the Pontryagin-based methods~\citep{jin2020pontryagin} and exploits the temporal sparsity of the control problem to compute gradients efficiently by reusing Hessian factorizations from the last iteration of DDP.

\subsection{Backpropagation Through Time}

The forward pass of the architecture is given by Algorithm \ref{alg:deep_coordinator}.
After $K$ iterations, the final iterate $\bv^K$ is used to compute the task loss $L(\bv^K)$.

To compute the hypergradient $\nabla_w L$ with respect to the policy weights $w$, we apply the chain rule at iteration $a$:
\begin{equation}
    % \pdv{L}{w} = \pdv{L}{\boldsymbol{\theta}^a} \pdv{\boldsymbol{\theta}^a}{w},
    \nabla_w L = \pdv{\boldsymbol{\theta}^a}{w}^\top \nabla_{\boldsymbol{\theta}^a} L,
\end{equation}
where the term $\nabla_{\boldsymbol{\theta}^a} L$ captures how the penalty parameters influence the loss through their effect on the solutions to subproblems 1 and 2.
Since computing $\nabla_{\boldsymbol{\theta}^a} L$ requires differentiating through both subproblems, we introduce adjoint variables $\delta \bz^a$ and $\delta \tz^a$ that satisfy adjoint equations derived from the subproblem KKT conditions.
These adjoint variables capture the sensitivity of the loss with respect to the solutions of their respective subproblems.

\subsection{Barrier-augmentation Approach}
Subproblem 2 solves an NLP problem with inequality constraints enforcing safety and inter-agent constraints.
At convergence, the solution $\tz^a = [\tx^a, \tu^a, \ty^a]$ satisfies the KKT conditions $\nabla_{\tz^a} \tilde{\calL}(\tz^a; \ts^a) = 0$, where $\ty^a$ are the Lagrange multipliers for the safety constraints, $\tilde{\calL}$ is the Lagrangian, and $\ts^a = [\bz^a, \blambda^{a - 1}, \bxi^{a - 1}, \boldsymbol{\theta}^a]$ collects the inputs to Subproblem 2.
Applying the IFT yields
\begin{equation}
    \pdv{\tz^a}{\ts^a} = - \nabla_{\tz\tz}^2 \tilde{\calL}^{-1} \nabla_{\tz\ts}^2 \tilde{\calL}.
\end{equation}
The main challenge is computing the adjoint variable $\delta \tz^a = - \nabla_{\tz\tz}^2 \tilde{\calL}^{-1} \nabla_{\tz^a} L$ efficiently.
We avoid instantiating the full Hessian matrix by exploiting the fact that the optimization is distributed across time, allowing us to solve $T$ linear systems in parallel.
To obtain smooth gradients, we differentiate through a barrier-augmented form of the subproblem~\citep{jin2021safe} rather than detecting the active set~\citep[Section 4]{stellato2020osqp}. By writing $\tilde{\ell}_t$ for the per-step objective and $g_{t,j}$ for the $j$-th inequality constraint at time $t$, the barrier-augmented objective is
\begin{equation} \label{eq:barrier_objective}
    \check{\calL}_t(\txu_t; \ts_t) = \tilde{\ell}_t(\txu_t; \ts_t) - \tau \sum_{j} \log\!\left(-g_{t,j}(\txu_t; \ts_t)\right),
\end{equation}
with barrier parameter $\tau > 0$. Because every constraint contributes a smooth term to $\check{\calL}_t$ regardless of its activeness, differentiating through~\eqref{eq:barrier_objective} avoids the discontinuities incurred when the active set changes.

\begin{proposition}[Differentiation through Subproblem 2]
\label{app:proposition:differentiation_through_subproblem_2a}
Let $\txu_t = (\tx_t, \tu_t)$ denote the primal variables at time $t$, and suppose the forward solve returns a strictly feasible stationary point of~\eqref{eq:barrier_objective}, i.e.\ $g_{t,j}(\txu_t; \ts_t) < 0$ for all $j$ and $\nabla_{\txu_t} \check{\calL}_t = 0$, with $\nabla_{\txu_t}^2 \check{\calL}_t$ nonsingular.
Then the adjoint variable $\delta \txu_t$ is the solution of the linear system
\begin{equation} \label{eq:subproblem_2_linear_system}
    \nabla_{\txu_t}^2 \check{\calL}_t \, \delta \txu_t = -\nabla_{\txu_t} L,
\end{equation}
where the Hessian on the left-hand side expands as
\begin{equation}
    \nabla_{\txu_t}^2 \check{\calL}_t = \nabla_{\txu_t}^2 \tilde{\ell}_t
    - \tau \sum_j \frac{\nabla_{\txu_t}^2 g_{t,j}}{g_{t,j}}
    + \tau \sum_j \frac{\nabla_{\txu_t} g_{t,j} \nabla_{\txu_t} g_{t,j}^\top}{g_{t,j}^2}.
\end{equation}
\end{proposition}

The linear system~\eqref{eq:subproblem_2_linear_system} is of dimension $n$, where $n$ is the number of variables, so solving the linear system requires $O(n^3)$ time complexity and $O(n^2)$ memory, independent of the number of constraints.

To fully recover the adjoint $\delta \tz_t$ we would also need the dual adjoint $\delta \ty_t$.
In our setting this is unnecessary for two reasons.
First, the task loss $L$ is not a function of the safety multipliers $\ty_t$, so $\nabla_{\ty_t} L = 0$.
Second, the multipliers do not feed back as inputs to either subproblem, so they carry no sensitivity downstream.
Consequently $\delta \ty_t$ contributes nothing to the hypergradient, and our implementation does not compute it.
For completeness, we give its closed form in Corollary~\ref{app:corollary:differentiation_through_subproblem_2_dual}.

\begin{corollary}[Recover adjoint variable]
\label{app:corollary:differentiation_through_subproblem_2_dual}
    Using the barrier multiplier estimate $\ty_{t,j} = -\tau / g_{t,j}$, the dual adjoint corresponding to constraint $g_{t,j}$ is
    \begin{equation*}
        \delta \ty_{t,j} = \frac{\tau}{g_{t,j}^2} \nabla_{\txu_t} g_{t,j}^\top \delta \txu_t,
    \end{equation*}
    and the full dual adjoint $\delta \ty_t$ is recovered by concatenating the $\delta \ty_{t,j}$ over $j$.
\end{corollary}

The barrier approach provides two advantages over active-set detection based methods. First, the gradient is continuous even when the active set changes, avoiding jumps incurred by active-set detection. Second, the barrier parameter $\tau$ acts as a tunable knob over the degree of smoothing, trading off the fidelity of gradient against the conditioning of the linear system~\eqref{eq:subproblem_2_linear_system}.

\subsection{Differentiation Through Subproblem 1}

Subproblem 1 solves an optimal control problem that is decoupled across agents.
At convergence, the solution $\bz^a = [\bx^a, \bu^a, \by^a]$ satisfies the KKT conditions $\nabla_{\bz^a} \bar{\calL}(\bz^a; \bs^a) = 0$, where $\by^a$ are the Lagrange multipliers for the dynamics constraints, $\bar{\calL}$ is the Lagrangian, and $\bs^a = [\tz^{a - 1}, \blambda^{a - 1},  \bxi^{a - 1}, \boldsymbol{\theta}^a]$ collects all the inputs to Subproblem 1.
Applying the IFT yields
\begin{equation}
    \pdv{\bz^a}{\bs^a} = -\nabla_{\bz\bz}^2 \bar{\calL}^{-1} \nabla_{\bz\bs}^2 \bar{\calL}.
\end{equation}
The adjoint variable $\delta \bz^a = -\nabla_{\bz\bz}^2 \bar{\calL}^{-1} \nabla_{\bz^a} L$ can be computed efficiently without forming the full Hessian by exploiting two key properties of the optimization: (1) the temporal sparsity of the problem allows an efficient backwards-in-time recursion, and (2) the problem is decoupled across agents, allowing for parallel computation.
Furthermore, the Hessian factorization from the final DDP iteration of the forward pass can be reused, avoiding expensive matrix refactorizations~\citep{oshin2024differentiable}.

\begin{theorem}[Differentiation through Subproblem 1]
\label{app:theorem:differentiation_through_subproblem_1}
% Let $\delta \bz^a = (\delta \bx^a, \delta \bu^a, \delta \by^a)$.
The gradients of the loss with respect to the Subproblem 1 inputs are given by
\begin{align*}
\begin{aligned}
\nabla_{\tilde{x}_{i,t}} L &= \delta x_{i,t} \odot (-\boldsymbol{\rho}_{i,t}), \\
\nabla_{\lambda_{i,t}} L &= \delta x_{i,t}, \\
\nabla_{\boldsymbol{\rho}_{i,t}} L &= \delta x_{i,t} \odot (x_{i,t} - \tilde{x}_{i,t}),
\end{aligned} \hfill \begin{aligned}
    \nabla_{\tilde{u}_{i,t}} L &= \delta u_{i,t} \odot (-\boldsymbol{\mu}_{i,t}), \\
    \nabla_{\xi_{i,t}} L &= \delta u_{i,t}, \\
    \nabla_{\boldsymbol{\mu}_{i,t}} L &= \delta u_{i,t} \odot (u_{i,t} - \tilde{u}_{i,t}).
\end{aligned}
\end{align*}
\end{theorem}
These expressions reveal an intuitive relationship: the gradient with respect to the penalty parameters is proportional to the primal residual $\bx^a - \tx^{a - 1}$ weighted by the loss sensitivity $\delta \bx^a$.
This provides direct signal for the policies to learn to adapt the penalty parameters for constraint satisfaction in an end-to-end fashion based on the task performance. 

\newpage
\section{Detailed Derivation of the Implicit Differentiation Framework}\label{app:gradient_computation_detailed_derivation}
In this appendix, we provide a detailed derivation of the gradient computation scheme for training Deep Coordinator.
The derivation is organized into three parts: (1) the overall backpropagation through one ADMM-DDP iteration, (2) implicit differentiation through Subproblem 2 with proofs of Proposition~\ref{app:proposition:differentiation_through_subproblem_2a} and Corollary~\ref{app:corollary:differentiation_through_subproblem_2_dual}, and (3) implicit differentiation through Subproblem 1 with a proof of Theorem~\ref{app:theorem:differentiation_through_subproblem_1}.

\subsection{Backpropagation Through One ADMM-DDP Iteration}
\noindent
To simplify the notation, let $\bzeta = (\blambda, \bxi)$ be the vector of combined dual variables.
At each iteration $a \in \{1, \ldots, K\}$, the Deep Coordinator forward pass consists of the following operations:
\begin{enumerate}
    \item Predict hyperparameters via
    \begin{equation} 
        \bTheta^a = \pi_{w^a}(\bTheta^{a -1 }, \bv^{a - 1}, \bchi)
    \end{equation}
    \item Solve Subproblem 1 for dynamically-feasible trajectories via
    \begin{equation} \label{eq:appendix:subproblem_1}
        \bz^a = (\bx^a, \bu^a, \by^a) = \text{DDP}(\bs^a),
    \end{equation}
    where $\by^a$ are the Lagrange multipliers for the dynamics constraints and $\bs^a = (\tz^{a - 1}, \bzeta^{a - 1}, \bTheta^a)$ collects all inputs to Subproblem 1 at iteration $a$.
    \item Solve Subproblem 2 for safe trajectories via
    \begin{equation} \label{eq:appendix:subproblem_2}
        \tz^a = (\tx^a, \tu^a, \ty^a) = \text{NLP}(\ts^a),
    \end{equation}
    where $\ty^a$ are the Lagrange multipliers for the safety constraints and $\ts^a = (\bz^a, \bzeta^{a - 1}, \bTheta^a)$ collects all inputs to Subproblem 2 at iteration $a$.
    \item Update dual variables via
    \begin{equation} \label{eq:appendix:dual_ascent}
        \bzeta^a = \bzeta^{a - 1} + \bTheta^a \odot ((\bx^a, \bu^a) - (\tx^a, \tu^a)).
    \end{equation}
\end{enumerate}
After $K$ iterations, the final iterate $\bv^K = (\bz^K, \tz^K, \bzeta^K)$ is used to compute the task loss $L(\bv^K)$.
Our goal is to compute the hypergradient $\nabla_w L$ to train the feedback policies via gradient descent.
Throughout the derivation, we use $\dv{f}{x}$ to denote the total derivative, which includes both direct and indirect relationships through other variables, while $\pdv{f}{x}$ denotes the partial derivative where all other variables are held constant.

Consider the final iteration $a = K$. We consider reverse-mode automatic differentiation (AD) of $L$ with respect to $w$. By the chain rule, the hypergradient is
\begin{equation} \label{eq:appendix:hypergradient_1}
    (\nabla_w L(\bv^K))^\top = \dv{L}{w} = \pdv{L}{\bz^K} \pdv{\bz^K}{w} + \pdv{L}{\tz^K} \pdv{\tz^K}{w} + \pdv{L}{\bzeta^K} \pdv{\bzeta^K}{w}.
\end{equation}
As we are performing reverse-mode AD, the order that we compute the derivatives matters for efficiency. First, we differentiate through the dual variable update \eqref{eq:appendix:dual_ascent}. Since $\bzeta^K$ depends on $\bz^K$, $\tz^K$, $\bzeta^{K - 1}$, and $\bTheta^K$ through \eqref{eq:appendix:dual_ascent}, we expand, yielding
\begin{equation*}
    \pdv{\bzeta^K}{w} = \pdv{\bzeta^K}{\bz^K} \pdv{\bz^K}{w} + \pdv{\bzeta^K}{\tz^K} \pdv{\tz^K}{w} + \pdv{\bzeta^K}{\bzeta^{K - 1}} \pdv{\bzeta^{K - 1}}{w} + \pdv{\bzeta^K}{\bTheta^K} \pdv{\bTheta^K}{w}.
\end{equation*}
Substituting into \eqref{eq:appendix:hypergradient_1} and combining like terms yields
\begin{align} \label{eq:appendix:hypergradient_2}
    \dv{L}{w} &= \left(\pdv{L}{\bz^K} + \pdv{L}{\bzeta^K} \pdv{\bzeta^K}{\bz^K}\right) \pdv{\bz^K}{w} \\
    &+ \left(\pdv{L}{\tz^K} + \pdv{L}{\bzeta^K} \pdv{\bzeta^K}{\tz^K}\right) \pdv{\tz^K}{w} + \pdv{L}{\bzeta^K} \pdv{\bzeta^K}{\bzeta^{K - 1}} \pdv{\bzeta^{K - 1}}{w} + \pdv{L}{\bzeta^K} \pdv{\bzeta^K}{\bTheta^K} \pdv{\bTheta^K}{w}.
\end{align}
Since $\tz^K$ depends on $\bz^K$ and $\bzeta^{K - 1}$, we differentiate through Subproblem 2 next.
Define the total derivative of $L$ with respect to $\tz^K$ as
\begin{equation*}
    \dv{L}{\tz^K} = \pdv{L}{\tz^K} + \pdv{L}{\bzeta^K} \pdv{\bzeta^K}{\tz^K},
\end{equation*}
so the second term of \eqref{eq:appendix:hypergradient_2} can be written more compactly as $\dv{L}{\tz^K} \pdv{\tz^K}{w} = \nabla_{\tz^K} L^\top \pdv{\tz^K}{w}$.
Computing $\pdv{\tz^K}{w} = \dv{\tz^K}{\ts^K} \pdv{\ts^K}{w}$ involves differentiating through Subproblem 2 \eqref{eq:appendix:subproblem_2} in order to compute $\dv{\tz^K}{\ts^K}$.
Let $\tilde{\calL}(\tz; \ts)$ denote the Lagrangian for Subproblem 2.
At convergence, the solution satisfies the KKT conditions $\nabla_{\tz^K} \tilde{\calL}(\tz^K; \ts^K) = 0$.
Applying the IFT yields
\begin{equation*}
    \dv{\tz^K}{\ts^K} = -\nabla_{\tz^K,\tz^K} \tilde{\calL}^{-1} \nabla_{\tz^K,\ts^K} \tilde{\calL}.
\end{equation*}

Now, we introduce the Subproblem 2 adjoint variable $\delta \tz^K$, which can be computed by solving the linear system 
\begin{equation} \label{eq:appendix:subproblem_2_adjoint}
    \delta \tz^K \defeq - \nabla_{\tz^K,\tz^K} \tilde{\calL}^{-1} \nabla_{\tz^K} L \iff \nabla_{\tz^K,\tz^K} \tilde{\calL} \, \delta \tz^K = - \nabla_{\tz^K} L.
\end{equation}

We utilize barrier-augmentation for formulating and solving this linear system efficiently by exploiting the structure of Subproblem 2. Namely, since Subproblem 2 decomposes across time, we solve $T$ independent linear systems in parallel (one per time step). Moreover, barrier-augmentation considers the effects from all constraints altogether, which provide smooth gradients. The derivation is provided in the following subsection.

Substituting the adjoint variable \eqref{eq:appendix:subproblem_2_adjoint} into the hypergradient \eqref{eq:appendix:hypergradient_2} and expanding $\ts^K$ using the chain rule yields
\begin{equation} \label{eq:appendix:hypergradient_3}
\begin{aligned}
    \dv{L}{w} = \left( \pdv{L}{\bz^K} + \pdv{L}{\bzeta^K} \pdv{\bzeta^K}{\bz^K} + (\delta \tz^K)^\top \nabla_{\tz^K,\bz^K} \tilde{\calL} \right) \pdv{\bz^K}{w} \\
    + \left( (\delta \tz^K)^\top \nabla_{\tz^K,\bzeta^{K - 1}} \tilde{\calL} + \pdv{L}{\bzeta^K} \pdv{\bzeta^K}{\bzeta^{K - 1}} \right) \pdv{\bzeta^{K - 1}}{w} \\
    + \left( (\delta \tz^K)^\top \nabla_{\tz^K,\theta^K} \tilde{\calL} + \pdv{L}{\bzeta^K} \pdv{\bzeta^K}{\bTheta^K} \right) \pdv{\bTheta^K}{w}.
\end{aligned}
\end{equation}
Define the total derivative of $L$ with respect to $\bz^K$ as
\begin{equation*}
    \dv{L}{\bz^K} \defeq \pdv{L}{\bz^K} + \pdv{L}{\bzeta^K} \pdv{\bzeta^K}{\bz^K} + (\delta \tz^K)^\top \nabla_{\tz^K,\bz^K} \tilde{\calL},
\end{equation*}
so the first term of \eqref{eq:appendix:hypergradient_3} can be written more compactly as $\dv{L}{\bz^K} \pdv{\bz^K}{w} = \nabla_{\bz^K} L^\top \pdv{\bz^K}{w}$.
Computing $\pdv{\bz^K}{w} = \dv{\bz^K}{\bs^K} \pdv{\bs^K}{w}$ involves differentiating through Subproblem 1 \eqref{eq:appendix:subproblem_1} in order to compute $\dv{\bz^K}{\bs^K}$.
Let $\bar{\calL}(\bz; \bs)$ denote the Lagrangian for Subproblem 1.
At convergence, the solution satisfies the KKT conditions $\nabla_{\bz^K} \bar{\calL}(\bz^K; \bs^K) = 0$. Applying the IFT yields
\begin{equation*}
    \dv{\bz^K}{\bs^K} = -\nabla_{\bz^K,\bz^K} \bar{\calL}^{-1} \nabla_{\bz^K,\bs^K} \bar{\calL}.
\end{equation*}

Next, we introduce the Subproblem 1 adjoint variable $\delta \bz^K$, which can be computed by solving the linear system.
\begin{equation} \label{eq:appendix:subproblem_1_adjoint}
    \delta \bz^K \defeq - \nabla_{\bz^K,\bz^K} \bar{\calL}^{-1} \nabla_{\bz^K} L \iff \nabla_{\bz^K,\bz^K} \bar{\calL} \, \delta \bz^K = - \nabla_{\bz^K} L.
\end{equation}
Our method exploits the agent-wise decomposition of Subproblem 1 and the temporal sparsity of the control problem to compute $\delta \bz^K$ by avoiding forming the entire Hessian $\nabla_{\bz^K,\bz^K} \bar{\calL}$.
Each individual agent's subproblem can be differentiated through by reusing the Hessian factorizations from the DDP forward pass~\citep[Theorem 5]{oshin2024differentiable}.
Moreover, Theorem 3 provides closed-form expressions for the gradients with respect to all of the Subproblem 1 inputs, with proof provided in the below subsection.
These expressions are equivalent to the matrix multiplication $(\delta \bz^K)^\top \nabla_{\bz^K, \bs^K} \bar{\calL}$ by expanding the Hessian with respect to each input. Substituting the adjoint variable \eqref{eq:appendix:subproblem_1_adjoint} into the hypergradient \eqref{eq:appendix:hypergradient_3} and expanding $\bs^K$ using the chain rule yields
\begin{equation*}
\begin{aligned}
    \dv{L}{w} &= (\delta \bz^K)^\top \nabla_{\bz^K, \tz^{K - 1}} \bar{\calL} \pdv{\tz^{K - 1}}{w} \\
    &+ \left( (\delta \tz^K)^\top \nabla_{\tz^K,\bzeta^{K - 1}} \tilde{\calL} + \pdv{L}{\bzeta^K} \pdv{\bzeta^K}{\bzeta^{K - 1}} + (\delta \bz^K)^\top \nabla_{\bz^K, \bzeta^{K - 1}} \bar{\calL} \right) \pdv{\bzeta^{K - 1}}{w} \\
    &+ \left( (\delta \tz^K)^\top \nabla_{\tz^K,\theta^K} \tilde{\calL} + \pdv{L}{\bzeta^K} \pdv{\bzeta^K}{\bTheta^K} + (\delta \bz^K)^\top \nabla_{\bz^K, \bTheta^K} \bar{\calL} \right) \pdv{\bTheta^K}{w}.
\end{aligned}
\end{equation*}

Finally, let
\begin{equation*}
    \pdv{L}{\bTheta^K} \defeq (\delta \tz^K)^\top \nabla_{\tz^K,\theta^K} \tilde{\calL} + \pdv{L}{\bzeta^K} \pdv{\bzeta^K}{\bTheta^K} + (\delta \bz^K)^\top \nabla_{\bz^K, \bTheta^K} \bar{\calL},
\end{equation*}
and note that $\bTheta^K = \pi_{w^K}(\bv^{K - 1}, \bchi)$, so
\begin{equation*}
    \pdv{\bTheta^K}{w} = \pdv{\pi_{w^K}}{w^K} \pdv{w^K}{w} + \pdv{\pi_{w^K}}{\bv^{K - 1}} \pdv{\bv^{K - 1}}{w},
\end{equation*}
where the first term gives the direct contribution of the weights $w_K$ for iteration $K$ and the second term propagates gradients backward to iteration $K - 1$.
Expanding $\bv^{K - 1}$ yields the final hypergradient:
\begin{equation*}
\begin{aligned}
    \dv{L}{w} = \pdv{L}{\bTheta^K} \pdv{\pi_{w^K}}{\bz^{K - 1}} \pdv{\bz^{K - 1}}{w} + \left( (\delta \bz^K)^\top \nabla_{\bz^K, \tz^{K - 1}} \bar{\calL} + \pdv{L}{\bTheta^K} \pdv{\pi_{w^K}}{\tz^{K - 1}} \right) \pdv{\tz^{K - 1}}{w} \\
    + \left( (\delta \tz^K)^\top \nabla_{\tz^K,\bzeta^{K - 1}} \tilde{\calL} + \pdv{L}{\bzeta^K} \pdv{\bzeta^K}{\bzeta^{K - 1}} + (\delta \bz^K)^\top \nabla_{\bz^K, \bzeta^{K - 1}} \bar{\calL} + \pdv{L}{\bTheta^K} \pdv{\pi_{w^K}}{\bzeta^{K - 1}} \right) \pdv{\bzeta^{K - 1}}{w} \\
    + \pdv{L}{\bTheta^K} \pdv{\pi_{w^K}}{w^K} \pdv{w^K}{w}.
\end{aligned}
\end{equation*}

To simplify the notation, define the accumulated partial derivatives through iteration $K$:
\begin{equation*}
\begin{aligned}
    \pdv{L}{\bz^{K - 1}} &\defeq \pdv{L}{\bTheta^K} \pdv{\pi_{w^K}}{\bz^{K - 1}}, \\
    \pdv{L}{\tz^{K - 1}} &\defeq (\delta \bz^K)^\top \nabla_{\bz^K, \tz^{K - 1}} \bar{\calL} + \pdv{L}{\bTheta^K} \pdv{\pi_{w^K}}{\tz^{K - 1}}, \\
    \pdv{L}{\bzeta^{K - 1}} &\defeq (\delta \tz^K)^\top \nabla_{\tz^K,\bzeta^{K - 1}} \tilde{\calL} + \pdv{L}{\bzeta^K} \pdv{\bzeta^K}{\bzeta^{K - 1}} + (\delta \bz^K)^\top \nabla_{\bz^K, \bzeta^{K - 1}} \bar{\calL} + \pdv{L}{\bTheta^K} \pdv{\pi_{w^K}}{\bzeta^{K - 1}}, \\
    \pdv{L}{w^K} &\defeq \pdv{L}{\bTheta^K} \pdv{\pi_{w^K}}{w^K}.
\end{aligned}
\end{equation*}
These derivatives capture how $\bz^{K - 1}$, $\tz^{K - 1}$, $\bzeta^{K - 1}$, and $w^K$ affect the loss through the final iteration (iteration $K$) of Deep Coordinator.
Using these definitions, the hypergradient simplifies to
\begin{equation*}
    \dv{L}{w} = \pdv{L}{\bz^{K - 1}} \pdv{\bz^{K - 1}}{w} + \pdv{L}{\tz^{K - 1}} \pdv{\tz^{K - 1}}{w} + \pdv{L}{\bzeta^{K - 1}} \pdv{\bzeta^{K - 1}}{w} + \pdv{L}{w^K} \pdv{w^K}{w}.
\end{equation*}
This is equivalent to the initial hypergradient presented in \eqref{eq:appendix:hypergradient_1}, except the backpropagation is now through iteration $a = K - 1$.
The process repeats recursively from $a = K, K - 1, \ldots, 1$, accumulating gradients via backpropagation through time.

\newpage
\subsection{Implicit Differentiation Through Subproblem 2}
\noindent

\subsubsection{Barrier Augmentation}
We now prove Proposition~\ref{app:proposition:differentiation_through_subproblem_2a}, which gives the per-timestep linear system for the primal adjoint $\delta \txu_t$.

Since the Subproblem 2 is separable across time, the adjoint system decouples into $T$ independent per-step system solved in parallel.
The forward solve returns a strictly feasible point $g_{t,j} < 0$, which we treat as a stationary point of~\eqref{eq:barrier_objective} and impose the stationarity condition for every $t$:
\begin{equation}
    \nabla_{\txu_t} \check{\calL}_t
    = \nabla_{\txu_t} \tilde{\ell}_t - \tau \sum_{j} \left( \frac{\nabla_{\txu_t} g_{t,j}}{g_{t,j}} \right)
     = 0.
\end{equation}

To compute $\delta \txu_t$, we solve the following linear system:
\begin{equation}
    \nabla_{\txu_t}^2 \check{\calL}_t \delta \txu_t = -\nabla_{\txu_t} L,
\end{equation}
where the Hessian on the left-hand side expands as
\begin{equation}
    \nabla_{\txu_t}^2 \check{\calL}_t = \nabla_{\txu_t}^2 \tilde{\ell}_t
    - \tau \sum_j \frac{\nabla_{\txu_t}^2 g_{t,j}}{g_{t,j}}
    + \tau \sum_j \frac{\nabla_{\txu_t} g_{t,j} \nabla_{\txu_t} g_{t,j}^\top}{g_{t,j}^2}.
\end{equation}

As a constraint becomes nearly active, $g_{t,j}$ is small in magnitude, so the last term scales as $\tau / g_{t,j}^2$ and dominates the Hessian, leaving it poorly conditioned. Moreover, for nonconvex constraints the term $\nabla_{\txu_t}^2 g_{t,j}$ may render $\nabla_{\txu_t}^2 \check{\calL}_t$ indefinite.

To prevent this ill-conditioning and mitigate possible indefiniteness, in practice, we regularize the Hessian and instead solve
\begin{equation}
    (\nabla_{\txu_t}^2 \check{\calL}_t + \delta I) \delta \txu_t = -\nabla_{\txu_t} L,
\end{equation}
where $\delta > 0$ is set to $10^{-5}$ in our implementation.
Note that the source of ill-conditioning here differs from the active-set method. Rather than the rank deficiency caused by linearly dependent active constraints, it stems from the large curvature of nearly-active barrier terms and the possible indefiniteness of nonconvex constraints.
In addition, $\tau$ controls how much the transition between active sets is smoothed; we set it to $10^{-6}$ for training.

\subsubsection{Recover Dual Adjoint}

We now derive Corollary~\ref{app:corollary:differentiation_through_subproblem_2_dual}, which recovers the dual adjoint $\delta \ty_t$ from the primal adjoint $\delta \txu_t$.

Under barrier augmentation the inequality multipliers are not independent variables. Actually, at convergence, they are tied to the primal solution through the perturbed complementarity relation
\begin{equation*}
    \ty_{t,j} = -\tau / g_{t,j},
\end{equation*}
where $\ty_{t,j}$ is the barrier multiplier estimate for the $j$-th constraint $g_{t,j}$. Since this expresses $\ty_{t,j}$ as a function of $\txu_t$ through $g_{t,j}(\txu_t)$, differentiating both sides with respect to $\txu_t$ and contracting with the primal adjoint $\delta \txu_t$ gives
\begin{equation*}
    \delta \ty_{t,j} = \nabla_{\txu_t}\!\left(-\frac{\tau}{g_{t,j}}\right)^{\!\top}\! \delta \txu_t = \frac{\tau}{g_{t,j}^2}\, \nabla_{\txu_t} g_{t,j}^\top\, \delta \txu_t,
\end{equation*}
and the full dual adjoint $\delta \ty_t$ is recovered by concatenating $\delta \ty_{t,j}$ over $j$.

Because $\delta \ty_t$ follows in closed form once $\delta \txu_t$ is known, no separate dual linear system is required. As previously noted, $\delta \ty_t$ does not enter the hypergradient and is reported only for completeness.

\newpage

\subsection{Implicit Differentiation Through Subproblem 1}
\noindent
Finally, we prove Theorem~\ref{app:theorem:differentiation_through_subproblem_1}, which gives closed-form expressions for the gradients of $L$ with respect to the Subproblem 1 inputs.
\begin{proof}
Recall that for Subproblem 1, each agent solves a control problem with the form
\begin{equation*}
\begin{aligned}
    \min_{x, u} & \sum_{t = 0}^{T - 1} \Bigl(\ell(x_t, u_t) + \frac{1}{2} x_t^\top (\rho_t I) x_t + (\lambda_t - \rho_t \odot \tilde{x}_t)^\top x_t \\
    & \quad + \frac{1}{2} u_t^\top (\mu_t I) u_t + (\xi_t - \mu_t \odot \tilde{u}_t)^\top u_t \Bigl) \\
    & \quad + \ell_T(x_T) + \frac{1}{2} x_T^\top (\rho_T I) x_T + (\lambda_T - \rho_T \odot \tilde{x}_T)^\top x_T, \\
    \text{subject to} \quad & x_{t + 1} = f(x_t, u_t), \quad x_0 = \bar{x}_0,
\end{aligned}
\end{equation*}
where we have simplified the quadratic penalty terms and dropped the agent index $i$ for clarity.
Let $\bar{\calL}(x, u, y; s)$ denote the Lagrangian for this optimization, where recall $s = (\tilde{x}, \tilde{u}, \lambda, \xi, \rho, \mu)$.
The optimal solution $z^* = (x^*, u^*, y^*)$ satisfies the KKT conditions $\nabla_z \bar{\calL}(z^*; s) = 0$, namely
\begin{equation*}
\begin{aligned}
    \nabla_{y_0} \bar{\calL} &= \bar{x}_0 - x_0^* = 0, \\
    & \hquad \vdots \\
    \nabla_{x_t} \bar{\calL} &= \nabla_x \ell_t(x_t^*, u_t^*) + \rho_t \odot (x_t^* - \tilde{x}_t) + \lambda_t + \pdv{f_t}{x_t}^\top y_{t + 1}^* - y_t^* = 0, \\
    \nabla_{u_t} \bar{\calL} &= \nabla_u \ell_t(x_t^*, u_t^*) + \mu_t \odot (u_t^* - \tilde{u}_t) + \xi_t + \pdv{f_t}{u_t}^\top y_{t + 1}^* = 0, \\
    \nabla_{y_{t + 1}} \bar{\calL} &= f(x_t^*, u_t^*) - x_{t + 1}^* = 0, \\
    & \hquad \vdots \\
    \nabla_{x_T} \bar{\calL} &= \nabla \ell_T(x_T^*) + \rho_T \odot (x_T^* - \tilde{x}_T) + \lambda_T - y_T^* = 0.
\end{aligned}
\end{equation*}
Given $\nabla_z L$, reverse-mode AD through this optimization involves computing the adjoint vector $\delta z = -\nabla_{zz}\bar{\calL}^{-1} \nabla_z L$, as discussed above.
This linear system can be solved efficiently through a single DDP iteration that reuses the Hessian factorizations from the last iteration of DDP from the forward pass~\citep[Theorem 5]{oshin2024differentiable}.
This vector is then multiplied by the Hessian $\nabla_{zs} \bar{\calL}$ yielding $\nabla_s L = \nabla_{zs} \bar{\calL}^\top \delta z$.
The individual blocks of the Hessian $\nabla_{zs} \bar{\calL}$ are given by
\begin{equation}
\label{eq:appendix:subproblem_1_gradients}
\begin{aligned}
    \nabla_{x_t,\tilde{x}_t} \bar{\calL} &= \diag(-\rho_t), \\
    \nabla_{x_t,\lambda_t} \bar{\calL} &= I, \\
    \nabla_{x_t,\rho_t} \bar{\calL} &= \diag(x_t - \tilde{x}_t), \\
\end{aligned} \quad
\begin{aligned}
    \nabla_{u_t,\tilde{u}_t} \bar{\calL} &= \diag(-\mu_t), \\
    \nabla_{u_t,\xi_t} \bar{\calL} &= I, \\
    \nabla_{u_t,\mu_t} \bar{\calL} &= \diag(u_t - \tilde{u}_t),
\end{aligned}
\end{equation}
which can be derived straightforwardly by noting that each input ($\tilde{x}_t$, $\lambda_t$, $\rho_t$, etc.) at time $t$ only depends on its respective state/control at time $t$.
Multiplying each of these terms by the respective elements of $\delta z$ yields~\eqref{eq:appendix:subproblem_1_gradients}.
\end{proof}

\newpage
\section{Details on Experimental Setup}\label{app:details_on_experimental_setup}
\subsection{Details on Problem Types}\label{app:details_on_problem_types}
In this appendix, we include details on problem generation and agent dynamics. Note that we warm-start Subproblem 2 with the solution from Subproblem 1. Also, for fast computation of Subproblem 2, we linearize all constraints at current nominal trajectory making the problem a linearly-constrained quadratic programming problem.

\subsubsection{Car Obstacle Field}
We first consider a multi-agent avoidance task with a randomized obstacle field. Each agent has Dubins vehicle dynamics, with the state given by $(x_{i,t}, y_{i,t}, \theta_{i,t})\in\R^3$ and controls given by $(v_{i,t}, \omega_{i,t})\in\R^2$, where $(x_{i,t}, y_{i,t})$ is the agent's position, $\theta_{i,t}$ is its orientation, $v_{i,t}$ is its linear forward velocity, and $\omega_{i,t}$ is its angular velocity. The dynamics are Euler-discretized with timestep $dt = 0.1$. Each agent has quadratic cost 
\begin{equation} \label{eq:quadratic_cost}
\begin{aligned}[c]
    \paramell{t}(\bx_{i,t}, \bu_{i,t}) &= \frac{1}{2}\bx_{i,t}^\top \bQ_i \bx_{i,t} + \frac{1}{2}\bu_{i,t}^\top \bR_i \bu_{i,t} \\
    \paramell{T}(\bx_{i,T}) &= \frac{1}{2}\bx_{i,T}^\top \bQ_i^f \bx_{i,T}
\end{aligned}
\end{equation}
with $\bQ_i = \bI_3$, $\bR_i = \diag(0.3, 0.3)$ and $\bQ_i^f = 100 \cdot \bQ_i$, where $\bI_3 \in \R^{3 \times 3}$ is the identity matrix. The controls are subject to $-v_{\max} \leq v_{i,t} \leq v_{\max}$ and $\quad -\omega_{\max}\leq \omega_{i,t} \leq \omega_{\max}$, where $v_{\max} = 5$ m/s and $\omega_{\max} = 5$ rad/s.
Each vehicle must also satisfy the circular obstacle constraints defined by
\begin{equation}\label{eq:car_obstacle_field_circular_obstacle_constraints}
    \norm{(x_{i,t}, y_{i,t}) - \bc}_2 \geq r_{\text{circ}} + d_{\text{safe}},
\end{equation}
where $\bc = (c_1, c_2)$ is the center of the obstacle, $r_{\text{circ}} > 0$ is its radius, and $d_{\text{safe}} > 0$ is the desired safety distance.
The agents must also abide by the inter-agent distance constraints
\begin{equation}\label{eq:car_obstacle_field_interagent_constraints}
    \norm{(x_{i,t}, y_{i,t}) - (x_{j,t}, y_{j,t})}_2 \geq d_{\text{safe}} \ \forall (i, j), t.
\end{equation}
We consider $N = 15$ agents initialized in a stationary grid and a target formation across the obstacle field. The agents must navigate to the target formation across a field of 3 circular obstacles with randomized radii and centers.

\subsubsection{Car Intersection}
\label{app:sec:car_intersection_generation}
Next, we consider a more complex task where both the initial and target positions of the agents are randomized, requiring significantly different solutions for each problem instance.
We consider $N = 8$ agents with Dubins vehicle dynamics discretized in $dt = 0.05$ and random initial conditions chosen from one of 16 fixed locations at a four-way road intersection.
The remaining 8 positions are then randomly assigned as targets.
The agents must navigate the intersection without running off the road or colliding to reach their target.
The quadratic cost~\eqref{eq:quadratic_cost} is chosen with parameters $\bQ_i = \diag(1, 1, 0)$, $\bR_i =\diag(0.1, 0.1)$ and $\bQ_i^f = 100 \cdot \bI_3$ imposed for each agent.
The control box constraints are set to $v_{\max} = 5$ m/s and $\omega_{\max} = 1$ rad/s.

We center the intersection at $\bc \in \R^2$, where each roadway $j\in\{1,2\}$ has a heading $\boldsymbol{\theta}_j$ and width $w_j > 0$. Road $j$ is the set of points whose lateral distance to the road centerline is at most $w_j /2$. Thus, for an agent $i$, its road membership constraint with respect to road $j$ is defined by
\begin{equation}
    s_{i,t}^j = |\boldsymbol{n}(\boldsymbol{\theta}_j)^\top(\bp_{i,t}- \bc)| - \left(\frac{w_j}{2} - d_{\text{safe}}\right) \leq 0,
\end{equation}
where $\boldsymbol{n}(\boldsymbol{\theta}_j) = (-\sin\boldsymbol{\theta}_j, \cos\boldsymbol{\theta}_j)\in\R^2$. The intersection constraint is then given by
\begin{equation}
    \softmin_\beta(s_{i,t}^1, s_{i,t}^2) \leq 0,
\end{equation}
where $\beta$ is the reciprocal of the temperature of the softmin function. For our problem setup, we set $\bc = (0,0)$, $\boldsymbol{\theta}_1 = 0$, $\boldsymbol{\theta}_2 = \pi/2$, $w_1=w_2=10$, and $\beta = 100$.

\subsubsection{Quadrotor Obstacle Field}
Finally, we deploy \deepcoordinator on multi-agent quadrotor maneuvering problems. We adopt the dynamics model from~\citep{Sabatino2015QuadrotorCM} where the state is given by the positions $\bp_{i, t}$, body velocities $\bv_{i,t}$, Euler angles $\bTheta_{i,t}$, and body rates $\bOmega_{i,t}$ of the quadrotor, and the controls are the collective thrust $F_{i, t}$ and torques $\tau_{i,t}^x, \tau_{i,t}^y, \tau_{i,t}^z$.
We discretize the dynamics using the Euler method with timestep $dt = 0.1$.
The problem consists of $N = 10$ agents and a time horizon of $T = 100$ timesteps.
Each agent has a quadratic cost~\eqref{eq:quadratic_cost} with $\bQ_i = \bdiag(\bI_3, 0.1\bI_3, 0.7\bI_3,0.1\bI_3)$, $\bR_i = \bI_4$, and $\bQ_i^f=100 \cdot \bQ_i$.
The control limits for each quadrotor are
\begin{equation*}
    F_{\min} \leq F_{i,t} \leq F_{\max}, \quad -\tau_{\max} \leq\tau_{i,t}^x, \tau_{i,t}^y, \tau_{i,t}^z\leq \tau_{\max},
\end{equation*}
where $F_{\min} = 5$, $F_{\max} = 15$ N, and $\tau_{\max} = 0.2$ N$\cdot$m. We use constraints identical to \eqref{eq:car_obstacle_field_circular_obstacle_constraints} to enforce cylindrical obstacle constraints and impose inter-agent distance constraints by extending \eqref{eq:car_obstacle_field_interagent_constraints} to $\R^3$. 

The drones are initialized with random altitude displacements $p_{i,0}^z$ and aim to reach their designated target state while navigating across a field with 7 cylindrical obstacles with randomized radii and centers. To maximize interactions between agents, we set target states in a grid such that the agents initialized farthest from the obstacle are assigned the most distant targets.

\newpage
\subsection{Details on Scaled Problem Types}\label{app:details_on_scaled_problem_types}
In this appendix, we discuss how the problems for scaling experiments are generated. We let $M$ denote the scaling factor relative to the training problems. 

\subsubsection{Scaled Car Obstacle Field}
For each problem, we randomly sample an obstacle field containing $2M + 1$ circular obstacles. The obstacle centers $c_{o}$ have $y$-positions uniformly spaced along $[-12M - 9.5, \, 12M + 2.5]$; the offset of 2.5 is added to ensure that more agents are initialized at the same y-coordinate as the obstacle centers, as was the case in the original task, which forces agents to break the symmetry when passing the obstacle. The $x$-positions and radii are sampled independently via $c_{o, x} \sim \mathcal{U}([-7, \,7])$ and $r_o \sim \mathcal{U}([2.75, \,4.75])$, similarly to the original task. Again, a safety radius of 1.0 is enforced. Problems contain $15M$ Dubins vehicle agents, initially arranged in a $5M \times 3$ grid where the agent in row $n$ and column $m$ has position
\begin{equation*}
\begin{aligned}[c]
    p_{\textrm{init}, mn} = (-25 + 5m, -2.5 - 10M + 5n)
\end{aligned}
\end{equation*}
for $n = 0, ..., 5M - 1$ and $m = 0, ..., 2$. The agents begin with an initial heading of 0 rad. Each agent's target state is again $x_{\textrm{target}, i} = x_{\textrm{init}, i} + (40, 0, 0)$. Obstacle constraints are enforced between all obstacles and agents. To minimize the number of inter-agent constraints, each agent only imposes these constraints on the agents initialized in its own row, the four rows above it, and the four rows below it.

\begin{figure}[h]
    \centering
    \includegraphics[width=\linewidth]{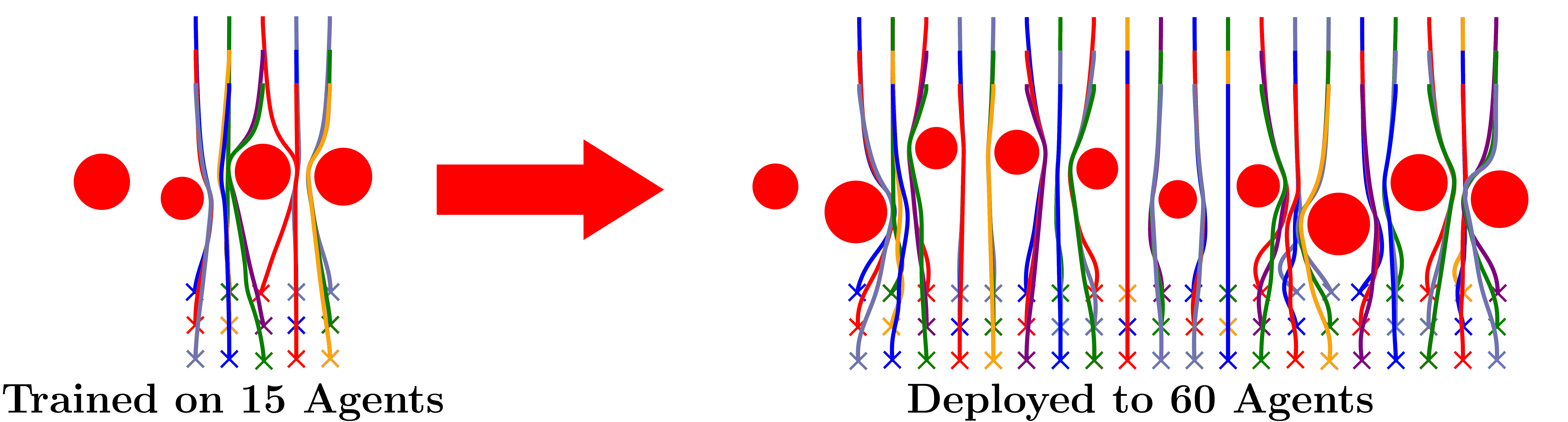}
    \caption{Car obstacle field scaling experiments. \deepcoordinator trained on $N = 15$ agents and deployed to $N = 60$ agents.}
    \label{fig:obs_scale}
\end{figure}

In these scaled car obstacle field problems, scaling is nearly isomorphic. In other words, each agent observes a structurally similar local problem at each scaling factor even as the total number of agents increases. Thus, this setting enables us to analyze whether \deepcoordinator produces transferable coordination rules rather than memorizing the training problem size. 

\subsubsection{Scaled Car Intersection}
For each problem, we employ the roadway constraint described in Appendix~\ref{app:sec:car_intersection_generation} with $c = (0, 0), \, \theta_1 = 0, \, \theta_2 = \pi/2, \, w_1 = w_2 = 10$ and $\beta = 100$. We similarly generate the initial and target state sets as
\begin{subequations}
    \begin{align}
        \mathcal{X}_{\textrm{inits}} &= \bigcup_{d\in\mathcal{D}}\{(-d, \pm h, 0), (\pm h, -d, -\pi/2), (d, \pm h, \pi), (\pm h, d, \pi/2)\}, \\
        \mathcal{X}_{\textrm{targets}} &= \bigcup_{d\in\mathcal{D}}\{(-d, \pm h, \pi), (\pm h, -d, \pi/2), (d, \pm h, 0), (\pm h, d, -\pi/2)\},
    \end{align}
\end{subequations}
where $\mathcal{D} = \{8.0, 10.0, 12.0\}$ when $N=12$, and $\mathcal{D} = \{8.0, 9.0, 10.0, 11.0\}$ when $N = 16$. We then sample a subset of initial and target states uniformly at random without replacement, with target positions distinct from all sampled initial positions. Then, we perturb the sampled initial and target positions of each agent to get the final states
\begin{subequations}
    \begin{align}
        x_{\textrm{init}, i} &= \overline{x}_{\textrm{init}, i} + v_i, \ v_i\sim\mathcal{N}(\bzero_3, \sigma\bdiag(\bI_2, \bzero_1)), \\
        x_{\textrm{target}, i} &= \overline{x}_{\textrm{target}, i} + w_i, \ w_i\sim\mathcal{N}(\bzero_3, \sigma\bdiag(\bI_2, \bzero_1)),
    \end{align}
\end{subequations}
where $\sigma = 0.1$, $\bI_2\in\R^{2\times 2}$ is the identity matrix, and $\bzero_3\in\R^3$, $\bzero_1\in\R$ are zero vectors.

\begin{figure}[h]
    \centering
    \includegraphics[width=\linewidth]{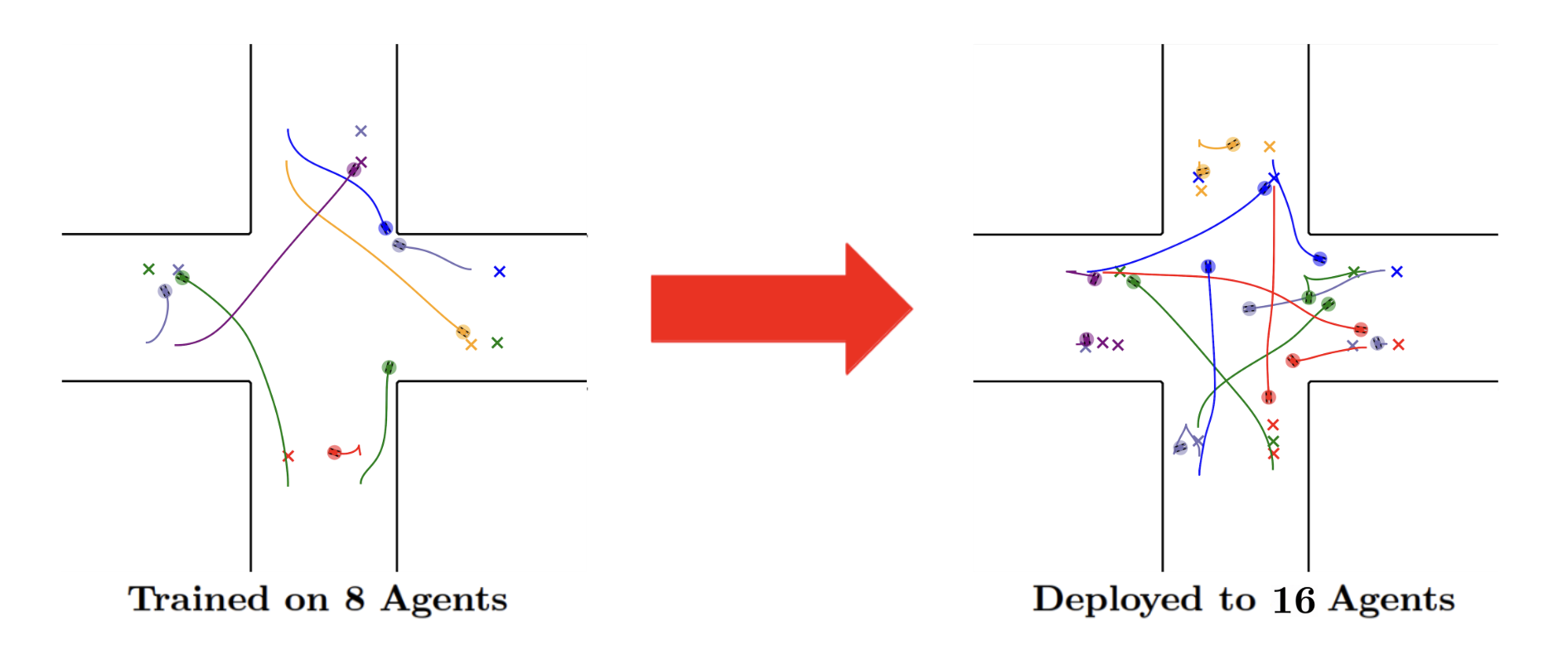}
    \caption{Car intersection scaling experiments. \deepcoordinator trained on $N = 8$ agents and deployed to $N = 16$ agents.}
    \label{fig:inter_scale}
\end{figure}

Unlike the obstacle field problems, the larger intersection problems are not locally isomorphic to smaller problems. We increase the density of the agents, introducing more competing right-of-way decisions and tighter coupling between agents. 

\subsubsection{Scaled Quadrotor Obstacle Field}
For each problem, we randomly sample an obstacle field containing $7M$ cylindrical obstacles. The obstacle centers $c_o$ have $y$-positions uniformly spaced along $[-18M, \,18M]$ and $x$-positions and radii are sampled independently via $c_{o, x} \sim \mathcal{U}([-7, \,7])$ and $r_o \sim \mathcal{U}([1.5, \,2])$, similarly to the original task. Again, a safety radius of 1.0 is enforced. Problems contain $10M$ quadrotor agents, initially arranged in a $5M \times 2$ grid where the agent in row $n$ and column $m$ has $x,y$ position 
\begin{equation*}
\begin{aligned}[c]
    (p^x, p^y)_{\textrm{init}, \,mn} = (-20 + 5m, -2.5 - 10M + 5n)
\end{aligned}
\end{equation*}
for $n = 0, ..., 5M - 1$ and $m = 0, 1$. The initial $z$ position of each quadrotor is independently sampled via $p_{\textrm{init}, \,i}^z \sim \mathcal{U}([-12,\,12])$ and all other state variables are again initialized to the hovering condition. Each agent's target state is given by $p_{\textrm{target}, \,i} = p_{\textrm{init}, \,i} \odot (-1, 1, ..., 1)$. The target $z$ position for all agents is set to zero, and the hovering condition is enforced at the target state. Obstacle constraints are enforced between all obstacles and agents. To minimize the number of inter-agent constraints, each agent only imposes these constraints on the agents initialized in its own row, the four rows above it, and the four rows below it.

Similar to the scaled car obstacle field problems, scaling is nearly isomorphic in our scaled quadrotor obstacle field problems. 

\begin{figure}[h]
    \centering
    \includegraphics[width=\linewidth]{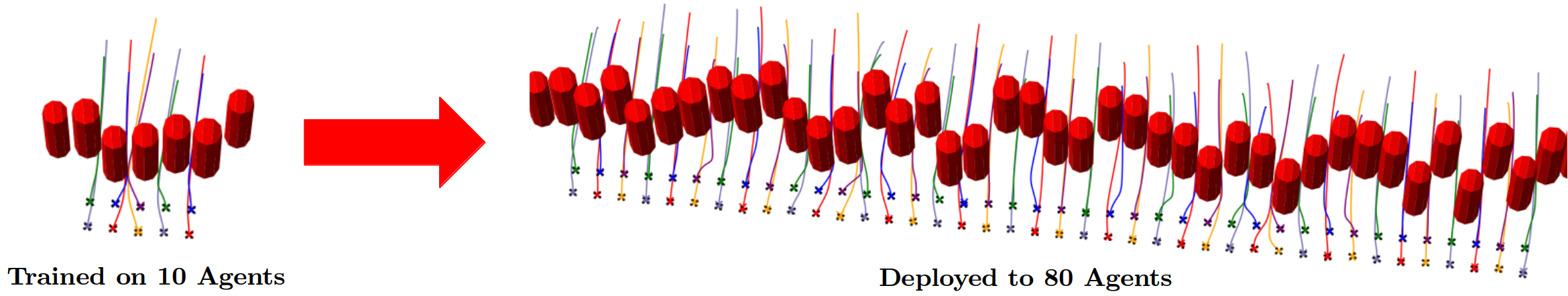}
    \caption{Quadrotor obstacle field scaling experiments. \deepcoordinator trained on $N = 10$ agents and deployed to $N = 80$ agents.}
    \label{fig:quad_scale}
\end{figure}

\newpage 

\newpage
\subsection{LSTM Architecture}\label{app:lstm}
In this appendix, we describe the LSTM feedback architecture used in the experiments. Recall that we learn policies $\pi_{w^a}$ which adapt hyperparameters $\boldsymbol{\theta}^a = [\boldsymbol{\rho}^a, \boldsymbol{\mu}^a]$ at every network layer via
\begin{equation}
\begin{split}
    \boldsymbol{\theta}^a = \pi_{w^a}(\boldsymbol{\theta}^{a-1}, \bv^{a-1}, \boldsymbol{\chi}).
\end{split}
\end{equation}
In the version of \deepcoordinator parameterized by an LSTM, we employ the shared per-agent, per-timestep feedback scheme. We predict $\rho$ and $\mu$ using two different policies with identical architectures but different weights $w_\rho$ and $w_\mu$. This structure ensures that \deepcoordinator learns policies that can generalize to larger problems but are also aware of the different roles played by $\rho$ and $\mu$. To minimize the number of weights, $w_\rho$ and $w_\mu$ are shared across all ADMM-DDP iterations. 

The policies $\pi_{w_\rho}$ and $\pi_{w_\mu}$ are executed for each agent $i$, timestep $t$, and state or control dimension $d$ to generate the relevant component of the penalty vector. Each instantiation of the policy is given 14 inputs derived from the previous hyperparameters $\theta^{a -1}$, the previous iterate $\bv^{a-1}$ and the problem context $\boldsymbol{\chi}$. The first three inputs are the relevant components of the dual variable and ADMM residuals. Formally, these inputs may be expressed as
\begin{align}
    \blambda^{a-1}_{i, t, d}&, \\
    \boldsymbol{r}_{\textrm{prim}, x, i, t, d}^{a-1}&= \bx^{a-1}_{i, t, d} - \tx^{a-1}_{i, t, d}, \\
    \boldsymbol{r}_{\textrm{dual},x, i, t, d}^{a-1}&= \boldsymbol{\rho}^{a-1}_{i, t, d}(\tx^{a-1}_{i, t, d} - \tx^{a-2}_{i, t, d}),
\end{align}
for $\pi_{w_\rho^a}$, and 
\begin{align}
    \bxi^{a-1}_{i, t, d} &, \\
    \boldsymbol{r}_{\textrm{prim},u, i, t, d}^{a-1}&= \bu^{a-1}_{i, t, d} - \tu^{a-1}_{i, t, d}, \\
    \boldsymbol{r}_{\textrm{dual}, u, i, t, d}^{a-1}&= \boldsymbol{\mu}^{a-1}_{i, t, d}(\tu^{a-1}_{i, t, d} - \tu^{a-2}_{i, t, d}),
\end{align}
for $\pi_{w_\mu^a}$. Both policies then receive an input representing the global average state constraint violation
\begin{equation}
\begin{split}
    G(\bx^{a-1}, \bu^{a-1}) &= \frac{1}{M}\sum_{g_{i,t}\in \boldsymbol{\chi}} \max\{0, g_{i,t}(\bx_{i,t}^{a-1}, \bu_{i,t}^{a-1})\}^2 \\
    &\qquad \qquad \quad+ \frac{1}{M}\sum_{h_{ij,t}\in\boldsymbol{\chi}} \max\{0, h_{ij,t}(\bx_{i,t}^{a-1}, \bu_{i,t}^{a-1}, \bx_{j,t}^{a-1}, \bu_{j,t}^{a-1})\}^2
\end{split}
\end{equation}
where $M$ is the total number of constraints, $g_{i,t}$ are the single-agent constraints for agent $i$ at time step $t$, and $h_{ij,t}$ are the inter-agent constraints for agents $i$ and $j$ at timestep $t$. By providing the constraint violation as an input to $\pi_{w^a}$, we enable our policy to adapt the penalty parameters to both indirect ($\blambda, \bxi$) and direct ($g_{i,t}, h_{ij,t}$) metrics regarding trajectory safety. 

Finally, both policies also receive a positional-encoding vector $PE^{a}\in\R^{10}$, enabling the feedback functions to adapt their behavior at each of the $K$ unfolded ADMM-DDP iterations while keeping shared weights across all layers. We use the transformer-style positional encoding where entry $l$ is given by
\begin{equation}
    (PE^a)_l = \begin{cases}
        \displaystyle \sin\left(\frac{a}{2^{l/10}}\right) & \text{if $l$ is even}, \\
        \displaystyle \cos\left(\frac{a}{2^{l/10}}\right) & \text{if $l$ is odd}.
    \end{cases}
\end{equation}
Each of these inputs is then concatenated into the vectors 
\begin{subequations}
    \begin{align}
        [\blambda^{a-1}_{i, t, d}, \, \boldsymbol{r}_{\textrm{prim}, x, i, t, d}^{a-1},  \, \boldsymbol{r}_{\textrm{dual}, x, i, t, d}^{a-1},  \, G(\bx^{a-1}, \,  \bu^{a-1}),  PE^a], \\
        [\bxi^{a-1}_{i, t, d},  \, \boldsymbol{r}_{\textrm{prim}, u, i, t, d}^{a-1},  \, \boldsymbol{r}_{\textrm{dual}, u, i, t, d}^{a-1},  \, G(\bx^{a-1}, \bu^{a-1}),  \,  PE^a].
    \end{align}
\end{subequations}
We parameterize $\pi_{w_\rho}$, $\pi_{w_\mu}$ as recurrent neural networks consisting of a single LSTM cell at each layer. At each unrolled iteration of ADMM-DDP, the LSTM receives the input vector of dimension $d_{\text{in}} = 14$ and maintains a hidden state and cell state of dimension $d_h = 32$. To compute the final penalty parameter value, the output of the LSTM is transformed by applying a multilayer perceptron (MLP) to the hidden state. The MLP consists of three fully-connected hidden layers of size 32 with sigmoid activation functions,
\begin{equation}
    \sigma(z)  = \frac{1}{1+e^{-z}},
\end{equation}
applied between layers and an $\exp(\cdot)$ activation applied to the final layer. This design decouples temporal representation learning, handled by the LSTM, from output regression, handled by the MLP. The final output dimension from the MLP is $d_\textrm{out}=1$, enabling our policy to remain agnostic of problem dimensions by batching over the agent, time, and state or control dimension.

Lastly, we scale the output from the LSTM architecture by learned scaling parameters $\gamma_{\rho}^a$ and $\gamma_{\mu}^a$. We initialize these parameters to match the tuning of Vanilla ADMM-DDP to ensure \deepcoordinator does not converge to the degenerate behavior of producing near-zero $\boldsymbol{\rho}^a$ and $\boldsymbol{\mu}^a$ at every iteration. 

\newpage
\subsection{Baseline Solvers}\label{app:baseline_solvers}
\subsubsection{Vanilla ADMM-DDP} \label{subsec:van_admm_ddp}
Vanilla ADMM-DDP refers to ADMM-DDP using constant hyperparameters across all iterations. To tune Vanilla ADMM-DDP, we select the lowest values of $\rho, \mu \in \{..., 10, 20, 30, ..., 90, 100, 200, 300, ...\}$ that yield trajectories which are constraint-satisfying and reach the target after $D$ iterations, where $D$ is the number of iterations used to generate the dataset. This yields $\rho = 300, \mu = 100$ for car obstacle field, $\rho = 500, \mu = 400$ for car intersection, and $\rho = 200, \mu = 400$ for quadrotor obstacle field.

\subsubsection{Adaptive ADMM-DDP} Adaptive ADMM-DDP refers to ADMM-DDP which adapts penalty parameters $\boldsymbol{\rho}^a$ and $\boldsymbol{\mu}^a$ using  the ADMM adaptation rule,
\begin{equation}
    \boldsymbol{\rho}_{i,t}^a= \begin{cases}
    \boldsymbol{\rho}_{i,t}^{a-1}\chi_{\textrm{adapt}} & \text{if } \lVert\boldsymbol{r}_{\textrm{prim}, x,i,t}^{a-1}\rVert_2 \geq \sigma_{\textrm{adapt}} \lVert\boldsymbol{r}_{\textrm{dual}, x, i,t}^{a-1}\rVert_2, \\
    \boldsymbol{\rho}_{i,t}^{a-1}/\chi_{\textrm{adapt}} & \text{if } \lVert \boldsymbol{r}_{\textrm{dual}, x, i,t}^{a-1}\rVert_2 \geq \sigma _{\textrm{adapt}}\lVert\boldsymbol{r}_{\textrm{prim}, x,i,t}^{a-1}\rVert_2,
    \end{cases}
\end{equation}
where $\chi_{\textrm{adapt}}> 0$ and $\sigma_{\textrm{adapt}} > 0$, and $\boldsymbol{\mu}^a$ is adapted similarly. We initialize the penalty parameters to be identical to those used in Vanilla ADMM-DDP and tune the adaptive scaling terms $\chi_{\textrm{adapt}}$ and $_{\sigma_\textrm{adapt}}$ until the adaptive solver is numerically stable with minimal constraint violation over 30 iterations. This yields $\chi_{\textrm{adapt}} = 1.2$ and $\sigma_{\textrm{adapt}} = 10$.

\newpage
\subsection{Training Setup}\label{app:training_setup_details}
For both the supervised and unsupervised training schemes, we train policies on 80 problems split into 4 batches, and select the model which shows minimum loss value within 50 epochs. For training, we utilize the AdamW~\citep{loshchilov2017decoupled} optimizer with the weight decay set to $10^{-3}$. The learning rate $\gamma_\text{rate}$ is tuned separately for each (training scheme, task, policy) tuple. 

\subsubsection{Supervised Training}
\label{app:sec:supervised_training}
For supervised training, the ground truth solution $\bv^*(\boldsymbol{\chi}_m)$ is obtained by allowing Vanilla ADMM-DDP as described in ~\ref{subsec:van_admm_ddp} to run until convergence.

For the car intersection and quadrotor obstacle field task, \deepcoordinator is trained with the following supervised loss on all iterates
\begin{equation}
\label{app:eq:loss_all_iterates_scaled}
    L_j(\{\bv^a\}_{a = 1}^K) = \sum_{a = 1}^K \gamma^a ||\bv^a - \bv^*(\boldsymbol{\chi}_j)||_2^2,
\end{equation}
where $\bv^a$ is the iterate produced by \deepcoordinator at iteration $a$ and $\gamma^a$ is a scaling parameter. 
The scaling parameter is included to equalize the contributions from all iterates, as later iterates exhibit smaller distances. 
Under the control-theoretic interpretation of deep-unfolding, these intermediate costs can be interpreted as stage costs that encourage the optimizer to find a smooth path to the optimal solution.
In our experiments, we set $\gamma^a = \exp(\tfrac{a+1-K}{K/4})$.

However, the benefit of considering the intermediate iterates does not hold for all training tasks. In particular, \deepcoordinator models trained on the car obstacle field task exhibited improved performance when trained with the following supervised loss which considers only the final iterate
\begin{equation}
    \label{app:eq:loss_final_iterate}
    L(\bv^K) =  ||\bv^K - \bv^*(\boldsymbol{\chi}_j)||_2^2,
\end{equation}
where $\bv^K$ is the iterate at the final iteration.

Importantly, for both choices of supervised loss, we evaluate the norm over the dual variables $\blambda, \bxi$ in addition to the primal variables.
The dual variables capture how difficult it is to satisfy the constraints, and thus play an important role in shaping the optimizer's dynamics.
Including $\blambda, \bxi$ in the loss allows the learned optimizer to take this coupling information into account, improving generalizability and convergence towards the optimal solution $\bv^*$.

We tuned learning rate $\gamma_\text{rate}$ for each policy by sweeping values in $\{..., 1\times10^{-3}, 2\times10^{-3}, 5\times10^{-3}, 1\times10^{-2}, 2\times10^{-2}, 5\times10^{-2}, ...\}$.

\begin{table}[h]
    \centering
    \begin{tabular}{ccccc}
         Task & Parameter & LS & LSPI & LSTM  \\
         \toprule
         Car Obstacle Field & $\gamma_\text{rate}$ & $0.1$ & $0.05$ & $0.005$ \\
         \midrule
         Car Intersection & $\gamma_\text{rate}$ & $0.1$ & $0.04$ & $0.002$ \\
         \midrule
         Quadrotor Obstacle Field & $\gamma_\text{rate}$ & $0.1$ & $0.1$ & $0.002$ \\
         \bottomrule \\
    \end{tabular}
    \caption{Tuned parameters of supervised scheme for each policy.}
    \label{tab:supervised_parameters}
\end{table}

\subsubsection{Unsupervised Training}
For unsupervised training, we utilize the same 80 environments as in supervised training, though the ground-truth solution $\bv^*(\boldsymbol{\chi}_m)$ is not necessary in this case. By adjusting $\gamma_\text{cost}$ and $\gamma_\text{const}$, we can control the trade-off between constraint satisfaction and cost minimization. However, it is not necessary to tune all three of $\gamma_\text{cost},$  $\gamma_\text{const},$ and $\gamma_\text{rate}$, since the latter acts as a scaling factor for the two other terms. 

Instead, we fix $\gamma_\text{const} = 2 \times 10^{7}$ and coarsely sweep $\gamma_\text{cost}$ over values in $\{..., \, 0.015, \, 0.05, \, 0.15, \\ \, 0.5, \,...\}$ and the learning rate $\gamma_\text{rate}$ over values in $\{..., 1\times10^{-3}, 2\times10^{-3}, 5\times10^{-3}, 1\times10^{-2}, 2\times10^{-2}, 5\times10^{-2}, 1\times10^{-1}, 2\times10^{-1}, 5\times10^{-1}, ...\}$.
Note that the resulting performance is fairly robust to the precise choice of $\gamma_\text{cost}$ and $\gamma_\text{const}$, as indicated by the coarseness of the sweep.

\begin{table}[h]
    \centering
    \begin{tabular}{ccccc}
         Task & Parameter & LS & LSPI & LSTM  \\
         \toprule
         \multirow{3}{*}{Car Obstacle Field} & $\gamma_\text{rate}$ & $0.2$ & $0.05$ & $0.01$ \\
         & $\gamma_\text{cost}$ & $0.05$ & $0.015$ & $0.015$ \\
         & $\gamma_\text{const}$  & $2 \!\times\! 10^{7}$ & $2 \!\times\! 10^{7}$ & $2 \!\times\! 10^{7}$ \\
         \midrule
         \multirow{3}{*}{Car Intersection} & $\gamma_\text{rate}$ & $0.1$ & $0.05$ & $0.002$ \\
         & $\gamma_\text{cost}$ & $0.15$ & $0.015$ & $0.05$ \\
         & $\gamma_\text{const}$  & $2 \!\times\! 10^{7}$ & $2 \!\times\! 10^{7}$ & $2 \!\times\! 10^{7}$ \\
         \midrule
         \multirow{3}{*}{Quadrotor Obstacle Field} & $\gamma_\text{rate}$ & $0.05$ & $0.05$ & $0.002$ \\
         & $\gamma_\text{cost}$ & $0.05$ & $0.05$ & $0.015$ \\
         & $\gamma_\text{const}$  & $2 \!\times\! 10^{7}$ & $2 \!\times\! 10^{7}$ & $2 \!\times\! 10^{7}$ \\
         \bottomrule \\
    \end{tabular}
    \caption{Tuned parameters of unsupervised scheme for each policy.}
    \label{tab:supervised_parameters}
\end{table}

\newpage
\section{Detailed Experimental Results}\label{app:detailed_experimental_results}

\subsection{Cost and Constraint Violation Over an Extended Wall-Clock Time}\label{app:extended_wall_clock_time_plots}
Fig. \ref{fig:car_plot_type_2}-\ref{fig:quad_plot_type_2} show the cost and maximum constraint violation of \deepcoordinator and the baseline models on 20 unseen, training-scale problems over an extended wall-clock time. Since the wall-clock time required to complete an iteration of ADMM-DDP may vary between problem instances, the values are sub-sampled at discrete time intervals. The horizontal dashed lines represent the value found by the model at $K = 30$ iterations. Adaptive ADMM-DDP is truncated at 30 iterations as it frequently produces numerically unstable subproblems when run for longer, causing the optimization process to fail. All graphs end when Vanilla ADMM-DDP converges on all problems.

\begin{figure}[H]
\centering
% \begin{tabular}{ccc}
%     \hspace{-12pt}
%     \includegraphics[width=0.33\linewidth]{Figures/per_policy_output/carobs/plots/mult_1/bar_cost.pdf} &
%     \hspace{-12pt}
%     \includegraphics[width=0.33\linewidth]{Figures/per_policy_output/carobs/plots/mult_1/bar_constraint.pdf} &
%     \hspace{-12pt}
%     \includegraphics[width=0.33\linewidth]{Figures/per_policy_output/carobs/plots/mult_1/bar_time_to_target.pdf}
%     \hspace{-12pt}
% \end{tabular}
\begin{tabular}{cc}
    \hspace{-12pt}
    \includegraphics[width=0.49\linewidth]{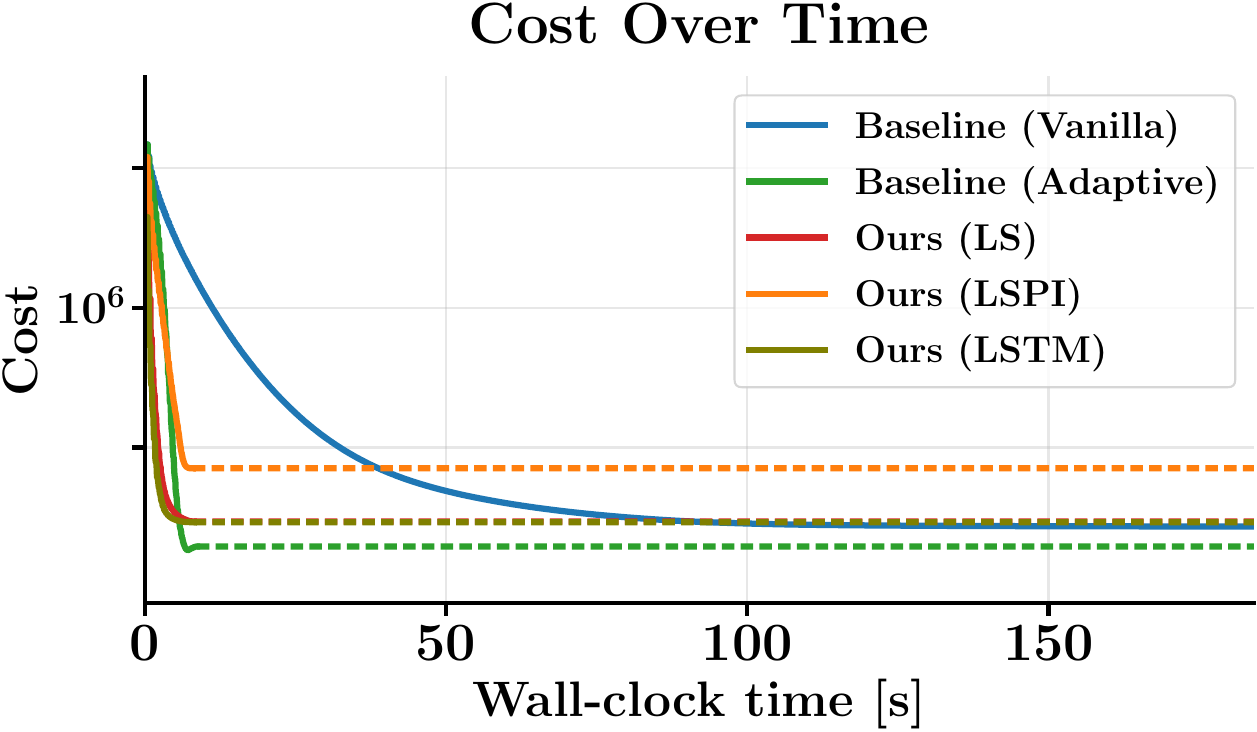} &
    \hspace{-12pt}
    \includegraphics[width=0.49\linewidth]{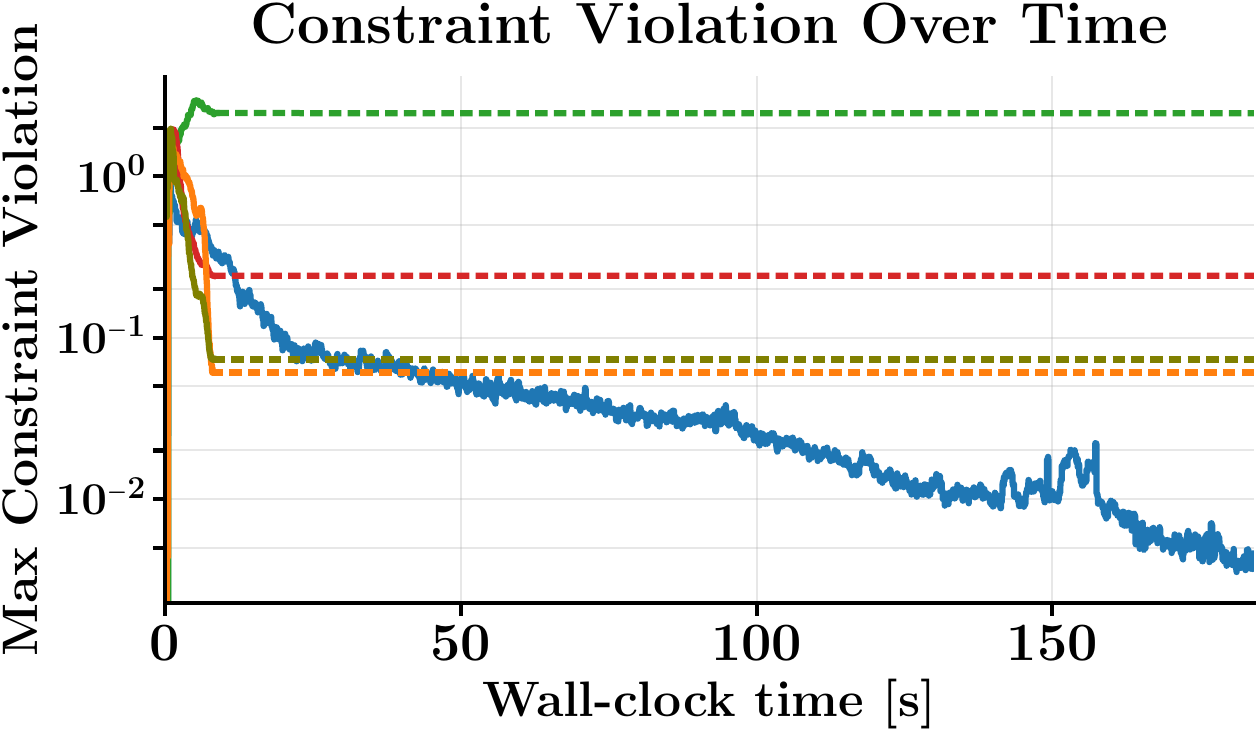}
    \hspace{-12pt}
\end{tabular}
\caption{Results of the car obstacle field task over time on 20 unseen test instances.}
\label{fig:car_plot_type_2}
\end{figure}

\begin{figure}[H]\label{app:intersection_extended_results}
\centering
% \begin{tabular}{ccc}
%     \hspace{-12pt}
%     \includegraphics[width=0.33\linewidth]{Figures/per_policy_output/intersection/plots/mult_1/bar_cost.pdf} &
%     \hspace{-12pt}
%     \includegraphics[width=0.33\linewidth]{Figures/per_policy_output/intersection/plots/mult_1/bar_constraint.pdf} &
%     \hspace{-12pt}
%     \includegraphics[width=0.33\linewidth]{Figures/per_policy_output/intersection/plots/mult_1/bar_time_to_target.pdf}
%     \hspace{-12pt}
% \end{tabular}

\begin{tabular}{cc}
    \hspace{-12pt}
    \includegraphics[width=0.49\linewidth]{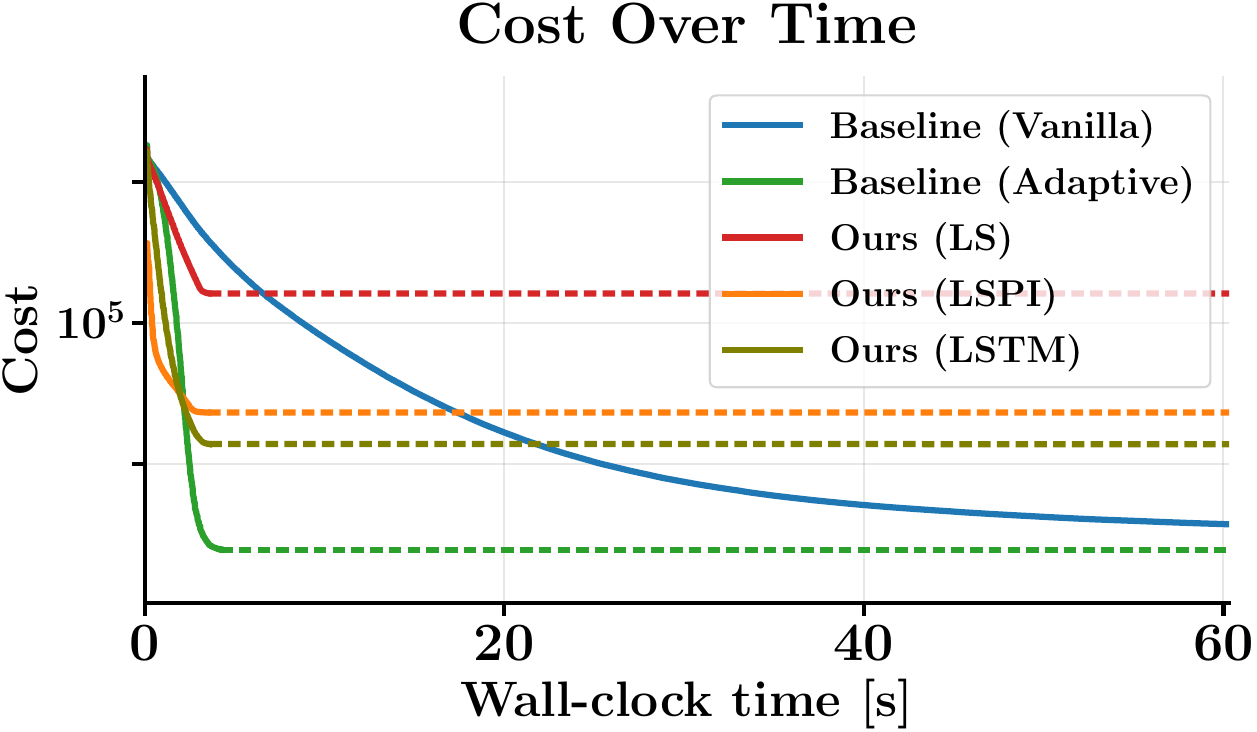} &
    \hspace{-12pt}
    \includegraphics[width=0.49\linewidth]{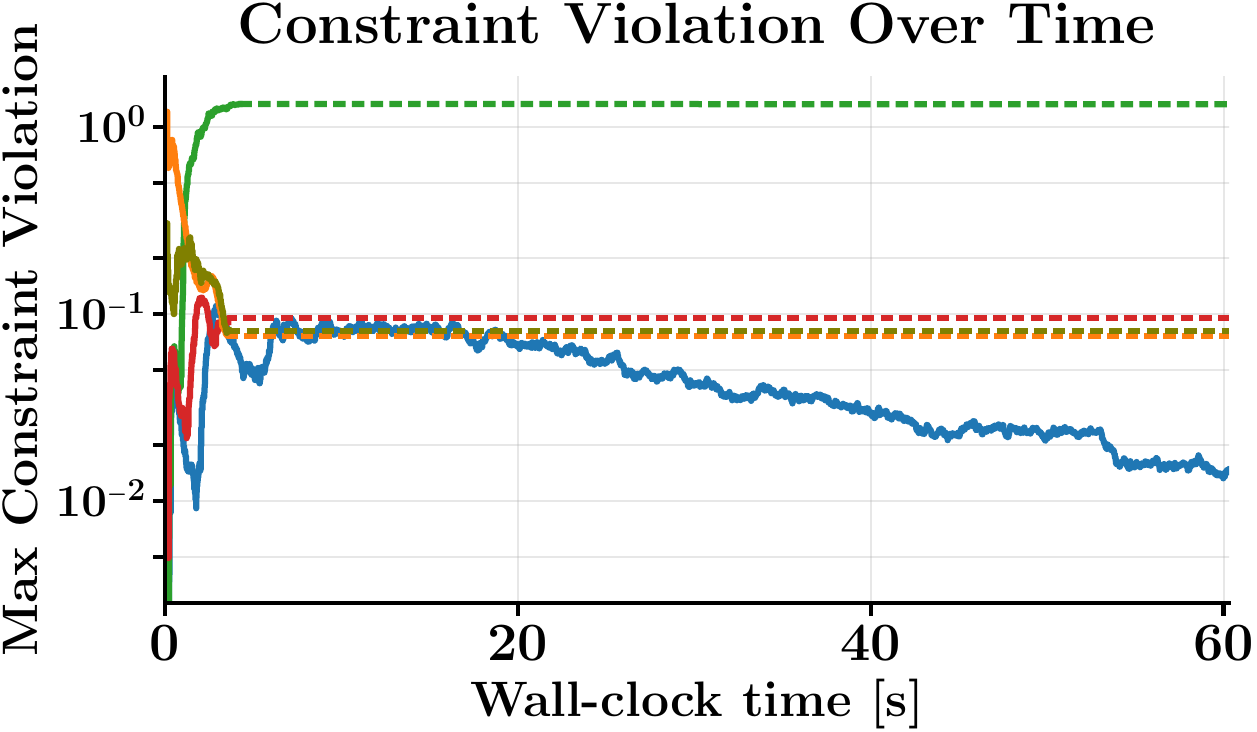}
    \hspace{-12pt}
\end{tabular}
\caption{Results of the car intersection task over time on 20 unseen test instances.}
\label{fig:inter_plot_type_2}
\end{figure}

\begin{figure}[H]
\centering
% \begin{tabular}{ccc}
%     \hspace{-12pt}
%     \includegraphics[width=0.33\linewidth]{Figures/per_policy_output/quad/plots/mult_1/bar_cost.pdf} &
%     \hspace{-12pt}
%     \includegraphics[width=0.33\linewidth]{Figures/per_policy_output/quad/plots/mult_1/bar_constraint.pdf} &
%     \hspace{-12pt}
%     \includegraphics[width=0.33\linewidth]{Figures/per_policy_output/quad/plots/mult_1/bar_time_to_target.pdf}
%     \hspace{-12pt}
% \end{tabular}
\begin{tabular}{cc}
    \hspace{-12pt}
    \includegraphics[width=0.49\linewidth]{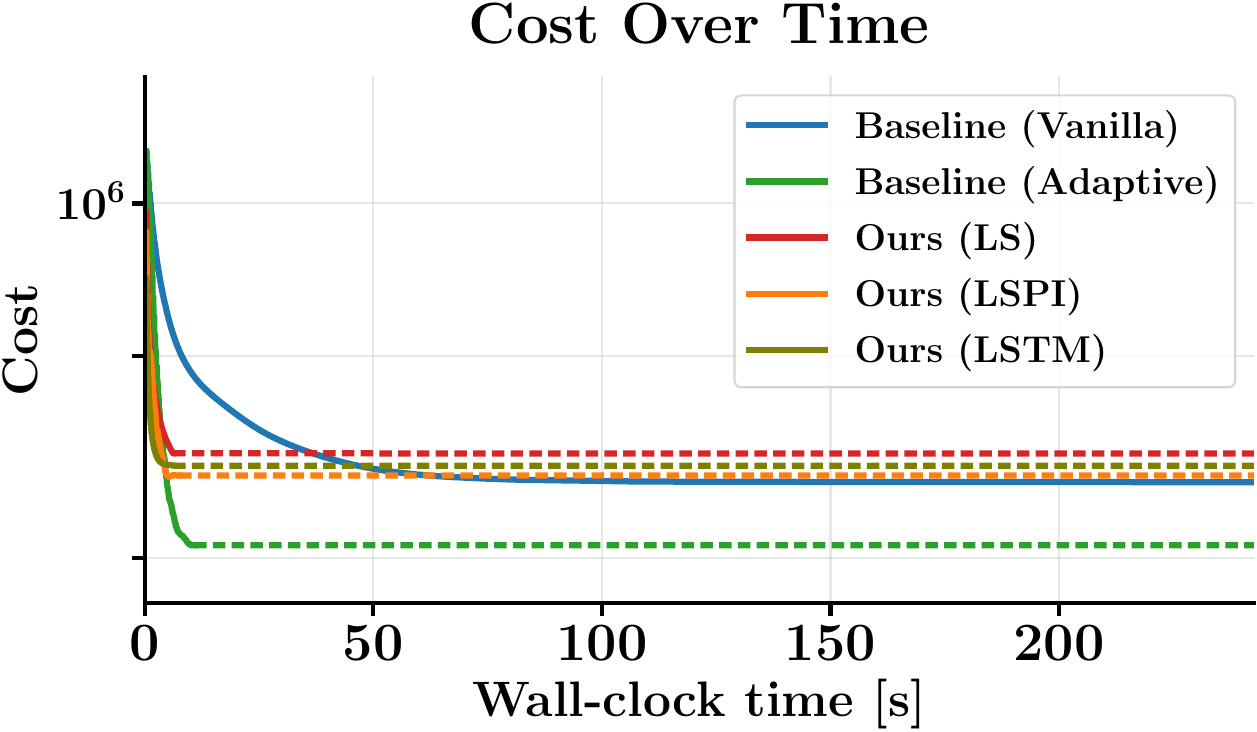} &
    \hspace{-12pt}
    \includegraphics[width=0.49\linewidth]{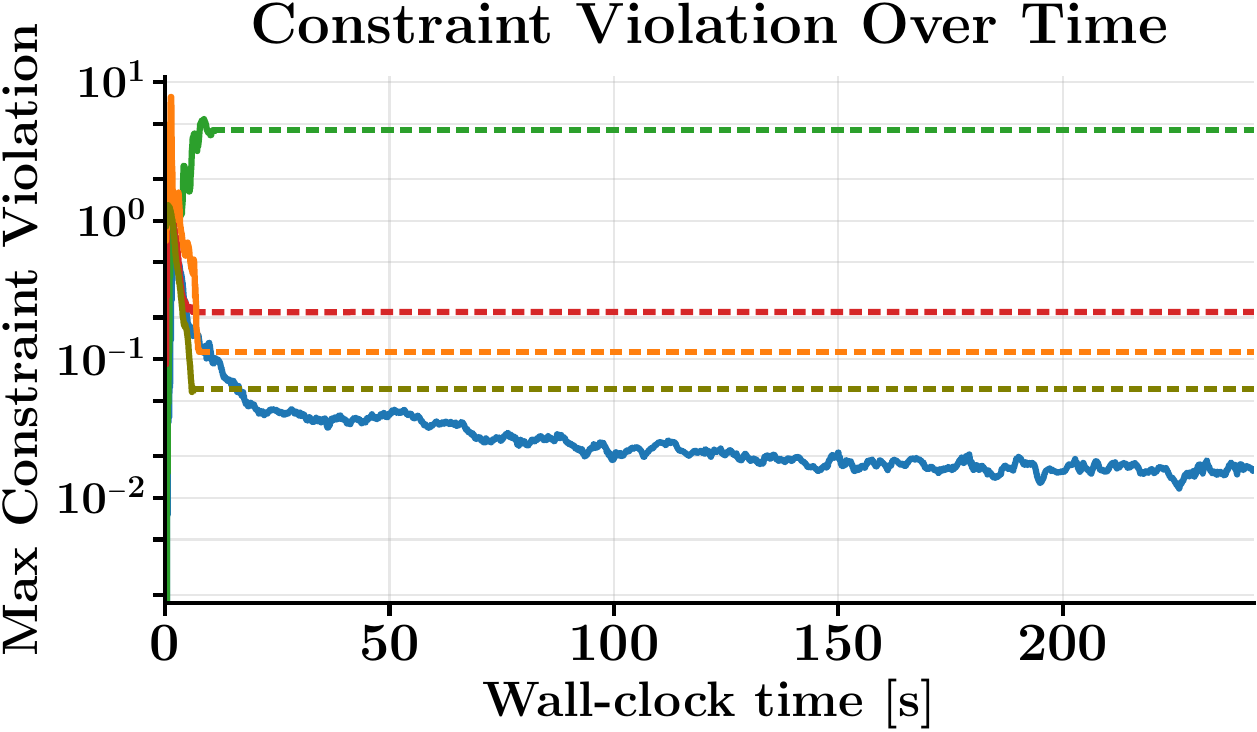}
    \hspace{-12pt}
\end{tabular}
\caption{Results of the quadrotor obstacle field task over time on 20 unseen test instances.}
\label{fig:quad_plot_type_2}
\end{figure}

\subsection{Detailed Numerical Results of Scaling Experiments}\label{app:scaling_table}

\begin{table}[H]
\centering
\begin{tabular}{ccccc|ccc}
\multirow{2}{*}{Task} & \multirow{2}{*}{$N$} & \multirow{2}{*}{Metric} & \multicolumn{2}{c|}{ADMM-DDP} & \multicolumn{3}{c}{Deep Coordinator} \\
 & & & Vanilla & Adaptive & LS & LSPI & LSTM \\ \toprule
\multirow{9}{*}{A} & \multirow{3}{*}{15} & Final Cost & $1235076.8$ & $\boldsymbol{305770.9}$ & $346303.4$ & $450818.9$ & $\underline{344902.0}$ \\
 &  & Max Const. & $0.395$ & $2.5$ & $0.242$ & $\boldsymbol{0.061}$ & $\underline{0.073}$ \\
 &  & Speed-up & - & - & $\underline{9.35}$ & $4.83$ & $\boldsymbol{9.44}$ \\
\cmidrule(lr){2-8}
 & \multirow{3}{*}{30} & Final Cost & $2453478.6$ & $\boldsymbol{610535.4}$ & $691183.5$ & $892261.2$ & $\underline{687550.9}$ \\
 &  & Max Const. & $0.417$ & $2.4$ & $0.297$ & $\boldsymbol{0.057}$ & $\underline{0.100}$ \\
 &  & Speed-up & - & - & $\underline{9.62}$ & $4.99$ & $\boldsymbol{9.78}$ \\
\cmidrule(lr){2-8}
 & \multirow{3}{*}{60} & Final Cost & $4909243.0$ & $\boldsymbol{1311819.8}$ & $1383651.4$ & $1786253.3$ & $\underline{1378203.1}$ \\
 &  & Max Const. & $0.515$ & $2.8$ & $0.360$ & $\boldsymbol{0.067}$ & $\underline{0.125}$ \\
 &  & Speed-up & - & - & $\underline{9.67}$ & $5.09$ & $\boldsymbol{9.84}$ \\
\midrule
\multirow{9}{*}{B} & \multirow{3}{*}{8} & Final Cost & $156365.7$ & $\boldsymbol{32893.1}$ & $115555.8$ & $64493.6$ & $\underline{55195.8}$ \\
 &  & Max Const. & $0.100$ & $1.3$ & $0.095$ & $\boldsymbol{0.076}$ & $\underline{0.081}$ \\
 &  & Speed-up & - & - & $1.87$ & $\underline{5.07}$ & $\boldsymbol{6.18}$ \\
\cmidrule(lr){2-8}
 & \multirow{3}{*}{12} & Final Cost & $191742.3$ & $\boldsymbol{39049.7}$ & $144439.9$ & $76969.7$ & $\underline{72047.7}$ \\
 &  & Max Const. & $0.152$ & $1.3$ & $0.162$ & $\underline{0.099}$ & $\boldsymbol{0.092}$ \\
 &  & Speed-up & - & - & $1.93$ & $\underline{5.47}$ & $\boldsymbol{5.92}$ \\
\cmidrule(lr){2-8}
 & \multirow{3}{*}{16} & Final Cost & $232346.3$ & $\boldsymbol{45185.4}$ & $170993.5$ & $\underline{89625.4}$ & $93419.0$ \\
 &  & Max Const. & $\underline{0.314}$ & $1.4$ & $0.325$ & $0.411$ & $\boldsymbol{0.293}$ \\
 &  & Speed-up & - & - & $1.86$ & $\boldsymbol{23.77}$ & $\underline{23.31}$ \\
\midrule
\multirow{12}{*}{C} & \multirow{3}{*}{10} & Final Cost & $590101.0$ & $\boldsymbol{211639.1}$ & $321112.6$ & $\underline{290707.0}$ & $303822.8$ \\
 &  & Max Const. & $0.158$ & $4.5$ & $0.219$ & $\underline{0.112}$ & $\boldsymbol{0.061}$ \\
 &  & Speed-up & - & - & $5.03$ & $\boldsymbol{6.34}$ & $\underline{6.24}$ \\
\cmidrule(lr){2-8}
 & \multirow{3}{*}{20} & Final Cost & $1141216.9$ & $\boldsymbol{434224.8}$ & $641816.3$ & $\underline{580886.1}$ & $607757.9$ \\
 &  & Max Const. & $0.185$ & $5.6$ & $0.254$ & $\underline{0.153}$ & $\boldsymbol{0.083}$ \\
 &  & Speed-up & - & - & $5.41$ & $\boldsymbol{7.57}$ & $\underline{6.74}$ \\
\cmidrule(lr){2-8}
 & \multirow{3}{*}{40} & Final Cost & $2297989.1$ & $\boldsymbol{875745.0}$ & $1287241.8$ & $\underline{1163967.6}$ & $1218391.9$ \\
 &  & Max Const. & $0.213$ & $6.3$ & $0.292$ & $\underline{0.203}$ & $\boldsymbol{0.144}$ \\
 &  & Speed-up & - & - & $6.30$ & $\boldsymbol{19.11}$ & $\underline{11.52}$ \\
\cmidrule(lr){2-8}
 & \multirow{3}{*}{80} & Final Cost & $4641287.9$ & $\boldsymbol{1815901.7}$ & $2585792.7$ & $\underline{2339338.8}$ & $2448767.6$ \\
 &  & Max Const. & $0.389$ & $6.0$ & $0.468$ & $\underline{0.288}$ & $\boldsymbol{0.279}$ \\
 &  & Speed-up & - & - & $10.61$ & $\boldsymbol{48.61}$ & $\underline{30.92}$ \\
\bottomrule \\
\end{tabular}
\caption{Results for scaling experiments. \textbf{Task A:} Car obstacle field; \textbf{Task B:} Car intersection; \textbf{Task C:} Quadrotor obstacle field. The policies are trained using the smallest number of agents within deployment settings. The final cost and maximum constraint violation are obtained after 30 iterations. The speed-up ratio compares \deepcoordinator against the Baseline (Vanilla), where the vanilla optimizer runs until reaching a cost within 5\% of Deep Coordinator. Results are averaged over 20 problem instances for all scales. \textbf{Bold} indicates the best-performing model, while $\underline{\text{underline}}$ indicates the second-best.}
\label{tab:scaling_details}
\end{table}

\end{document}